\documentclass[
  showpacs,
  twocolumn,
  pra,
  superscriptaddress,
  notitlepage,
  nofootinbib,
  longbibliography
]{revtex4-1}

\usepackage[left]{lineno}

\usepackage{amsthm}  
\usepackage{etoolbox}  
\newtheorem{assumption}{Assumption}  
\usepackage[dvips]{graphicx}
\usepackage{amsmath,amssymb,amsthm,mathrsfs,amsfonts,dsfont,amscd,keyval}
\usepackage{mathtools,mathrsfs}

\usepackage{yfonts} 
\usepackage{ytableau}
\usepackage{subfigure, epsfig}
\usepackage{braket}
\usepackage{bm} 
\usepackage{array} 
\usepackage{fancyhdr}
\usepackage{enumerate}
\usepackage{color}
\usepackage{hyperref}
\usepackage[normalem]{ulem}
\usepackage{tabularx}
\usepackage{times}
\hypersetup{colorlinks=true, linkcolor=blue, citecolor=blue, urlcolor=black}
\usepackage{multirow}
\usepackage{float}
\usepackage{tensor}
\usepackage{multirow}
\usepackage{physics}
\usepackage{soul}

\newcommand{\ie}{\textit{i.e.}}
\newcommand{\eg}{\textit{e.g.}}

\theoremstyle{definition}
\newtheorem{definition}{Definition}[section]

\theoremstyle{plain}
\newtheorem{theorem}[definition]{Theorem}
\newtheorem{lemma}[definition]{Lemma}

\newtheorem{proposition}[definition]{Proposition}

\theoremstyle{remark}
\newtheorem*{remark}{Remark}
\makeatletter
\renewenvironment{remark}[1][]{%
  \ifx\relax#1\relax
    \begin{itshape}%
    \textbf{Remark.}\ 
  \else
    \begin{itshape}%
    \textbf{Remark} (#1).\ 
  \fi
}{%
  \end{itshape}%
}
\makeatother
\newcommand{\comments}[1]{}

\makeatletter
\def\l@subsubsection#1#2{} 
\makeatother
\begin{document}

\let\oldaddcontentsline\addcontentsline

\renewcommand{\addcontentsline}[3]{}

\title{Artificial Entanglement in the Fine-Tuning of Large Language Models}

\author{Min Chen}
\affiliation{Department of Computer Science, University of Pittsburgh, Pittsburgh, PA 15260, USA}

\author{Zihan Wang}
\affiliation{Department of Computer Science, Northwestern University, Evanston, IL 60208, USA}

\author{Canyu Chen}
\affiliation{Department of Computer Science, Northwestern University, Evanston, IL 60208, USA}

\author{Zeguan Wu}
\affiliation{Department of Computer Science, University of Pittsburgh, Pittsburgh, PA 15260, USA}

\author{Manling Li}
\affiliation{Department of Computer Science, Northwestern University, Evanston, IL 60208, USA}

\author{Junyu Liu}
\affiliation{Department of Computer Science, University of Pittsburgh, Pittsburgh, PA 15260, USA}

\date{Dated: November 11, 2024}

\date{\today}

\begin{figure*}
    \centering
\includegraphics[width=\textwidth]{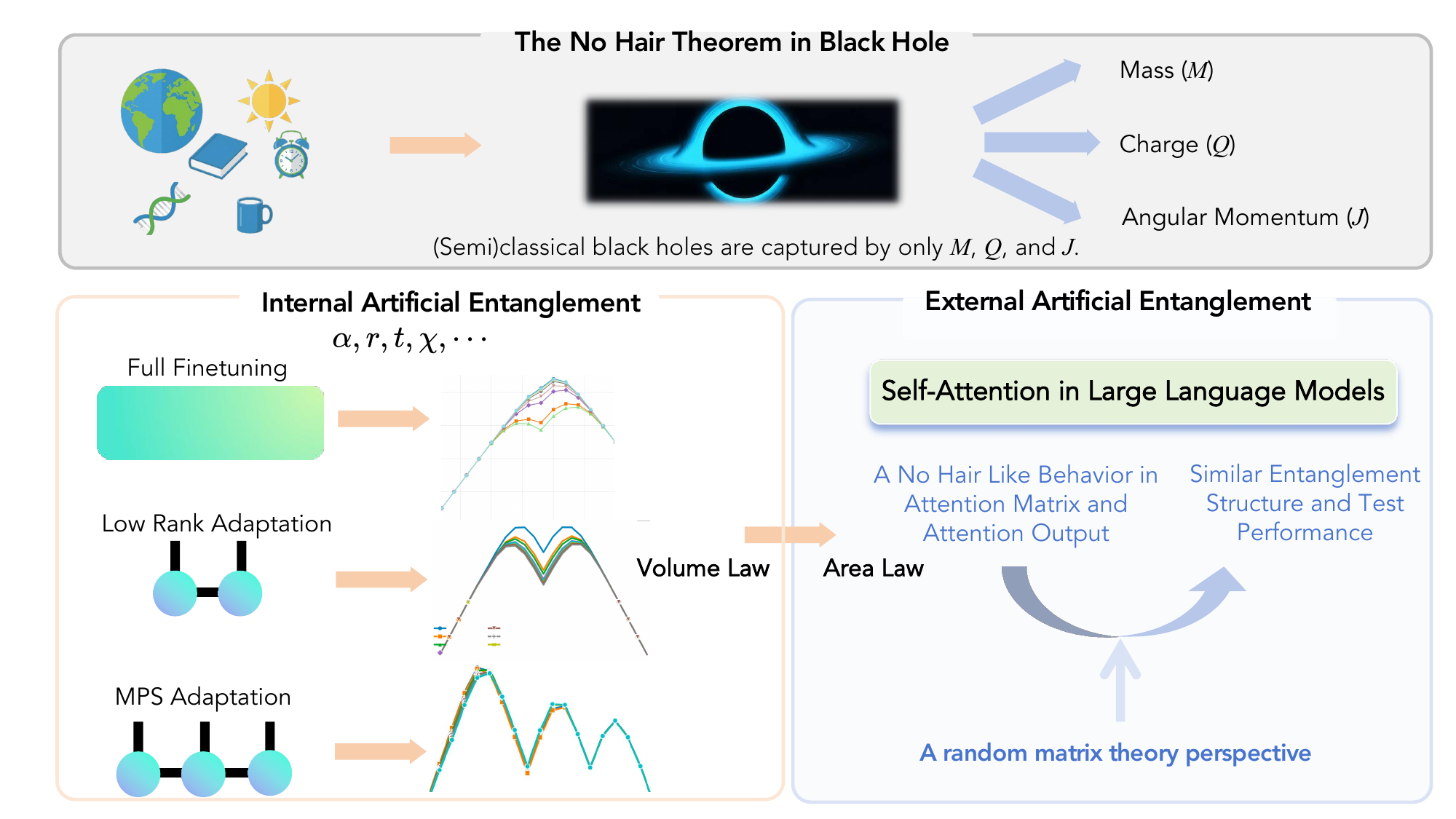}   
    \vspace{-10pt}
    \caption{Key Findings.
\underline{\textbf{(i)}} Projection matrices (\(\Delta W_Q\) and \(\Delta W_V\)) exhibit volume-law internal entanglement profiles with distinctive entanglement valleys that differ between FFT, LoRA and MPS adaptation.
\underline{\textbf{(ii)}} Attention matrices show area-law scaling with logarithmic corrections. Random matrix theory explains this through the Attention Cardy Formula.
\underline{\textbf{(iii)}} Despite internal differences, external attention outputs remain invariant, revealing a no-hair–like effect where the attention mechanism acts as a coarse-graining operator.}
\label{fig:overall_framework}
\end{figure*}

\begin{abstract}
Large language models (LLMs) can be adapted to new tasks using parameter-efficient fine-tuning (PEFT) methods that modify only a small number of trainable parameters, often through low-rank updates. In this work, we adopt a quantum-information-inspired perspective to understand their effectiveness. From this perspective, low-rank parameterizations naturally correspond to low-dimensional Matrix Product States (MPS) representations, which enable entanglement-based characterizations of parameter structure. 
Thereby, we term and measure \textit{``Artificial Entanglement''}, defined as the entanglement entropy of the parameters in artificial neural networks (in particular the LLMs). We first study the representative low-rank adaptation (LoRA) PEFT method, alongside full fine-tuning (FFT), using LLaMA models at the 1B and 8B scales trained on the Tulu3 and OpenThoughts3 datasets, and uncover: \underline{(i)} \textit{Internal artificial entanglement} in the updates of query and value projection matrices ($\Delta W_Q, \Delta W_V$) in LoRA follows a \textit{volume law} with a central suppression (termed as the \textit{``Entanglement Valley''}), which is sensitive to hyper-parameters and is distinct from that in FFT;  
\underline{(ii)} \textit{External artificial entanglement} in attention matrices, corresponding to token--token correlations in representation space, follows an \textit{area law} with logarithmic corrections and remains robust to LoRA hyper-parameters and training steps.
Drawing a parallel to the No-Hair Theorem in black hole physics, we propose that although LoRA and FFT induce distinct internal entanglement signatures, such differences do not manifest in the attention outputs, suggesting a ``no-hair'' property that results in the effectiveness of low rank updates. 
We further provide theoretical support based on random matrix theory, and extend our analysis to an \textit{MPS Adaptation} PEFT method, which exhibits qualitatively similar behaviors.
\end{abstract} 

\maketitle

\section{Introduction}
\label{sec:introduction}

Recently, large language models (LLMs) have demonstrated desirable abilities by scaling model size and pretraining on large corpora of data~\cite{radford2018improving,radford2019language,brown2020language,devlin2019bert,raffel2020exploring,vaswani2017attention,kaplan2020scaling,hoffmann2022training,henighan2020scaling,chowdhery2023palm,achiam2023gpt4}. 
Post-training, usually via fine-tuning, is adopted to adapt pretrained LLMs for downstream tasks. However, it is often considered unnecessary and computationally expensive to update the full set of parameters. Therefore, many parameter-efficient fine-tuning (PEFT) methods are proposed, including low-rank adaptation (LoRA)~\cite{hu2022lora} and its variants~\cite{dettmers2023qlora, zhang2023adalora,liu2024dora, kopiczko2024vera,hayou2024loraplus, valipour2023dylora,zhang2023lora,zhang2024loraprune,zi2023delta,lialin2023relora}. 
Previous studies~\cite{hu2022lora,li2018intrinsic,shuttleworth2024lora,hayou2024loraplus,kalajdjieski2024rank,pan2024lora} also contribute significantly to explaining their effectiveness. However, a fundamental question remains unsolved: \textit{how complex are the internal parameter structures induced by distinct fine-tuning methods, and to what extent are such complexities reflected in attention-level representations?}

We investigate this question by adopting a quantum information perspective, and draw a parallel with a well-known phenomenon in theoretical physics: the tension between internal quantum correlations and external simplicity.
This tension has been most prominently explored within black hole physics~\cite{bombelli1986quantum,srednicki1993entropy,eisert2008area}. A fundamental conflict~\cite{israel1967event,carter1971axisymmetric,bhattacharya2007no,mavromatos1996eluding} emerges between the quantum mechanical description of black holes (via entanglement entropy) and their classical geometric description. 
While the scaling behavior of entanglement entropy implies that black holes are complex systems characterized by a vast number of degrees of freedom, classical gravitational theories predict that black holes are surprisingly simple objects with no distinct features. 
This classical perspective is formalized by uniqueness theorems~\cite{israel1968event,carter1971axisymmetric,robinson1975uniqueness}. The physical essence of these theorems is encapsulated in the No-Hair Theorem~\cite{israel1967event,israel1968event,carter1971axisymmetric,robinson1975uniqueness}, which states that all stationary black hole solutions of the Einstein–Maxwell equations can be completely characterized by only three externally observable classical parameters: mass, electric charge, and angular momentum.
All other semi-classical information about the matter that formed the black hole (the ``hair'') disappears behind the event horizon and is inaccessible to an external observer \cite{harlow2016jerusalem}.

Motivated by this dilemma, we define and probe the \textit{``artificial entanglement''}, where ``artificial'' refers to the application of entanglement entropy measures to artificial neural networks (particularly the LLMs). Notably, ``artificial'' emphasizes that large language models are entirely classical systems and do not possess any physical quantum degrees of freedom or genuine quantum entanglement. Therefore, this adopted perspective serves purely as a mathematical tool.  \textbf{Our contributions are as follows:} 

\textbf{\underline{(i)}} We first represent the updates of query and value projection matrices ($\Delta W_Q, \Delta W_V$) as Matrix Product State (MPS) \cite{eisert2008area,vidal2002computable,vidal2003efficient,vidal2003entanglement,vidal2004efficient,vidal2007entanglement,orus2019tensor,berezutskii2025tensor,milsted2022collisions} and compute the von Neumann entanglement entropy across bond partitions, yielding an \textit{artificial entanglement profile} that captures internal parameter structure. We thereby refer the corresponding entanglement as the \textit{``internal artificial entanglement''}.
We show that the resulting internal artificial entanglement under LoRA follows a \textit{volume law} (\ie, entanglement entropy scaling linearly with system size, indicating a high degree of correlation that cannot be faithfully captured by simple representations)~\cite{eisert2008area,amico2008entanglement} with a pronounced central suppression (termed as the \textit{ ``Entanglement Valley''}), is highly sensitive to hyperparameters (e.g., rank $r$ and scaling factor $\alpha$), and is qualitatively distinct from full fine-tuning (FFT), revealing substantial differences in internal parameter correlation. Meanwhile, we also measure the \textit{external artificial entanglement} via similar routine, where ``external'' means here we choose the entanglement to be the one on attention matrices and outputs that is  produced during the forward computation rather than on model parameters, corresponding to token--token correlations in representation space.
We demonstrate that it follows an approximate \textit{area law with logarithmic corrections} 
(\ie, entanglement entropy scaling logarithmically with sequence length). Despite the markedly different internal entanglement structures induced by LoRA and FFT, this external artificial entanglement remains empirically robust across a certain range of hyper-parameters and training stages, suggesting that variations in internal parameter correlation do not significantly propagate to attention-level representations within the regimes explored. Drawing a parallel to the No-Hair Theorem, we show that the attention mechanism exhibits a ``no-hair'' property, whereby high-correlation internal entanglement signatures are coarse-grained and rendered indistinguishable at the level of attention outputs. Combining the above results, we propose that LoRA learns intrinsically different internal parameter structures with FFT, but remains effective because the attention mechanism coarse-grains these differences through a no-hair property. \textbf{\underline{(ii)}} We provide theoretical support for the observations based on random matrix theory by establishing an \textit{``Attention Cardy Formula''}, which shows that the entanglement entropy of attention matrices exhibits a logarithmic scaling with sequence length under certain initialization conditions, analogous to the Cardy formula~\cite{calabrese2004entanglement,calabrese2009entanglement} for critical quantum systems. Moreover, we relate the emergence of low external artificial entanglement to \textit{stable rank collapse} (\ie, spectral mass concentrates onto a few dominant modes)~\cite{saadamind}. \textbf{\underline{(iii)}} We extend our analysis beyond LoRA to an \textit{MPS Adaptation} PEFT strategy, in which weight updates are parameterized directly via MPS, and demonstrate that it exhibits qualitatively similar internal and external artificial entanglement behaviors, validating the generality of our findings.

Our experiments are conducted on the Llama 3 series models~\cite{dubey2024llama} at the 1B and 8B scales, and employ the Tulu3~\cite{lambert2024tulu} and OpenThoughts3~\cite{guha2025openthoughts} datasets, which target instruction following and reasoning tasks respectively with distinct scope and structures. Our key findings are summarized in FIG.~\ref{fig:overall_framework}.
We begin by outlining the fine-tuning methods under consideration and introducing our overall framework (Sec.~\ref{sec:sketch_MPS}).
We then present empirical results on a LLaMA model that reveal a no-hair property, together with several additional phenomena (Sec.~\ref{sec:Entanglement Structure in FFT and LoRA}).
Next, we provide a theoretical interpretation based on random matrix theory (Sec.~\ref{sec:random_matrix_theory}).
Building on the above results, we extend our study to an MPS adaptation method, which exhibits qualitatively similar behaviors (Sec.~\ref{sec:MPS_adaptation}).
Finally, we describe our methods in detail and conclude with a discussion on potential applications (Secs.~\ref{sec:method} and~\ref{sec:discussion}).

\section{Results}

\subsection{FFT, LoRA and the Overall Framework for Artificial Entanglement Analysis}
\label{sec:sketch_MPS}

In this section, we first review FFT and LoRA~\cite{hu2022lora} methods.
FFT adapts a pre-trained model by updating all model parameters.
Given a pre-trained weight projection matrix $W_0$, the model learns a full matrix update
$\Delta W_{\mathrm{Full}}$, resulting in $W = W_0 + \Delta W_{\mathrm{Full}}$.
In contrast, LoRA injects trainable low-rank matrices to approximate the weight
updates.
Specifically, LoRA decomposes the update into two smaller matrices, denoted as
$A \in \mathbb{C}^{r \times d_{\mathrm{in}}}$ and
$B \in \mathbb{C}^{d_{\mathrm{out}} \times r}$,\footnote{
For notational generality, we allow the matrices to be complex-valued.
All experiments in this work are conducted with real-valued parameters 
as in standard LoRA implementations.
}
where the rank $r \ll \min(d_{\mathrm{in}}, d_{\mathrm{out}})$.
In addition, LoRA uses a scaling hyperparameter $\alpha$ to control the magnitude
of the low-rank update. Formally, LoRA update is formulated as
\begin{equation}
    \Delta W_{\mathrm{LoRA}} := \frac{\alpha}{r} \, BA
    \in \mathbb{C}^{d_{\mathrm{out}} \times d_{\mathrm{in}}},
\end{equation}
which is then added to the frozen pre-trained weight $W_0$.

To analyze artificial entanglement in fine-tuning, we employ a framework that includes MPS decomposition with entanglement entropy computation. Our approach consists of four key steps: (i) reshaping matrices (including the updates of weight projection matrices and attention-related matrices) into higher-order tensors, (ii) factorizing them into MPS representations, (iii) interpreting the MPS as a mathematical formalism of many-body states (formally analogous to quantum many-body states, though the underlying systems are entirely classical), and (iv) computing the von Neumann entanglement entropy across different bonds\footnote{With a slight abuse of terminology, we sometimes use the terms
\emph{bond}, \emph{cut}, \emph{cut position}, and \emph{bi-partition}
interchangeably, as they are in one-to-one correspondence in the
one-dimensional settings considered here. Specifically, a cut at a given position induces a bi-partition of the
Hilbert space, while in an MPS representation such a cut corresponds to cutting a virtual bond.
} by performing SVD and obtaining singular values (Schmidt indices), yielding an \textit{artificial entanglement profile}. This profile quantifies how correlations vary across different positions, revealing the fruitful internal structure. Below we take the update of a projection matrix \(\Delta W\) (\(\Delta W_Q\) and \(\Delta W_V\)) as an example and detail the essential steps (See preliminaries on tensor basis in \textit{Appendix~\ref{sec:Preliminaries and Related Works}}):

\underline{Step 1}: Although $\Delta W \in \mathbb{C}^{d_{\mathrm{out}} \times d_{\mathrm{in}}}$ is originally an Order-2 tensor, we first reshape it into a higher-order tensor by factorizing its input and output dimensions into their prime components. This produces the finest-grained decomposition
\begin{equation}
    d_{\mathrm{out}} = \prod_{k=1}^{n} f_k,\qquad 
d_{\mathrm{in}} = \prod_{\ell=1}^{m} g_\ell,
\end{equation}
where $n$ and $m$ denote the number of prime factors in $d_{\mathrm{out}}$ and $d_{\mathrm{in}}$, respectively. This allows us to view $\Delta W$ as a tensor living on $(n+m)$ sites. This provides the ``lattice'' on which we build the MPS. 

\underline{Step 2:} We then perform a sequence of singular value decompositions (SVDs) along the reshaped tensor dimensions to factorize $\Delta W$ into a chain of local Order-3 tensors. The resulting representation
\begin{equation}
    \Delta W \ \longrightarrow\ \{\mathcal{T}^{[1]},\mathcal{T}^{[2]},\dots,\mathcal{T}^{[n+m]}\}
\end{equation}
is an MPS whose virtual bond indices capture how different tensor sites are correlated. For LoRA, the factorization naturally reflects its low-rank structure: the update $BA$ induces a restricted bond dimension. For FFT, the MPS is obtained directly from $\Delta W_{\mathrm{Full}}$ without such constraints.

\underline{Step 3:} Once factorized, the MPS can be interpreted as a mathematical formalism of many-body states. Specifically, the MPS factorization of $\Delta W$ is formally analogous to the MPS representation of the amplitude tensor of a quantum many-body state $|\Psi\rangle$~\cite{liu2023simulating,chen2023towards} (though the underlying system is entirely classical). Each ``site'' corresponds to one of the factorized input/output degrees of freedom, and the virtual bonds represent the latent coupling strength between subsystems.

\underline{Step 4:} When performing an SVD at each bond during the MPS decomposition, we obtain the singular values $\{\sigma_{\alpha_k}\}$, where $k$ denotes the bond position (the $k$-th bond in the MPS chain) and $\alpha_k$ indexes the Schmidt modes at bond $k$. 
We normalize the singular values according to
\begin{equation}
    \lambda_{\alpha_k} = \frac{\sigma_{\alpha_k}}{\sqrt{\sum_{\alpha_k} \sigma_{\alpha_k}^2}},
\end{equation}
such that $\sum_{\alpha_k} \lambda_{\alpha_k}^2 = 1$.
These normalized Schmidt coefficients define the bipartite entanglement entropy,
\begin{equation}
    S_k = -\sum_{\alpha_k} \lambda_{\alpha_k}^2 \log \lambda_{\alpha_k}^2,
\end{equation}
where $S_k$ quantifies the entanglement entropy across bond $k$. 
Scanning across all bonds yields an \textit{``artificial entanglement profile''}, allowing us to probe correlation patterns. 
The full details are provided in Sec.~\ref{sec:method}.

\begin{remark}[On the use of the term ``artificial entanglement"]
The entanglement entropy considered in this work is purely a mathematical
quantity derived from the Schmidt decomposition of high-dimensional tensors, and should not be interpreted as physical quantum entanglement. Although large language models are entirely classical systems, MPS representations and their associated entanglement measures provide a principled, quantum-inspired, basis-independent way to characterize how correlations are distributed under different bipartitions. In this sense, the resulting ``artificial entanglement profiles'' are used as descriptive diagnostics of structural correlations, rather than as real physical entanglement in an underlying physical Hilbert space.
\end{remark}

\subsection{Artificial Entanglement Profile in FFT and LoRA Fine-tuning}
\label{sec:Entanglement Structure in FFT and LoRA}

\begin{figure}[!t]
    \centering
\includegraphics[width=\linewidth]{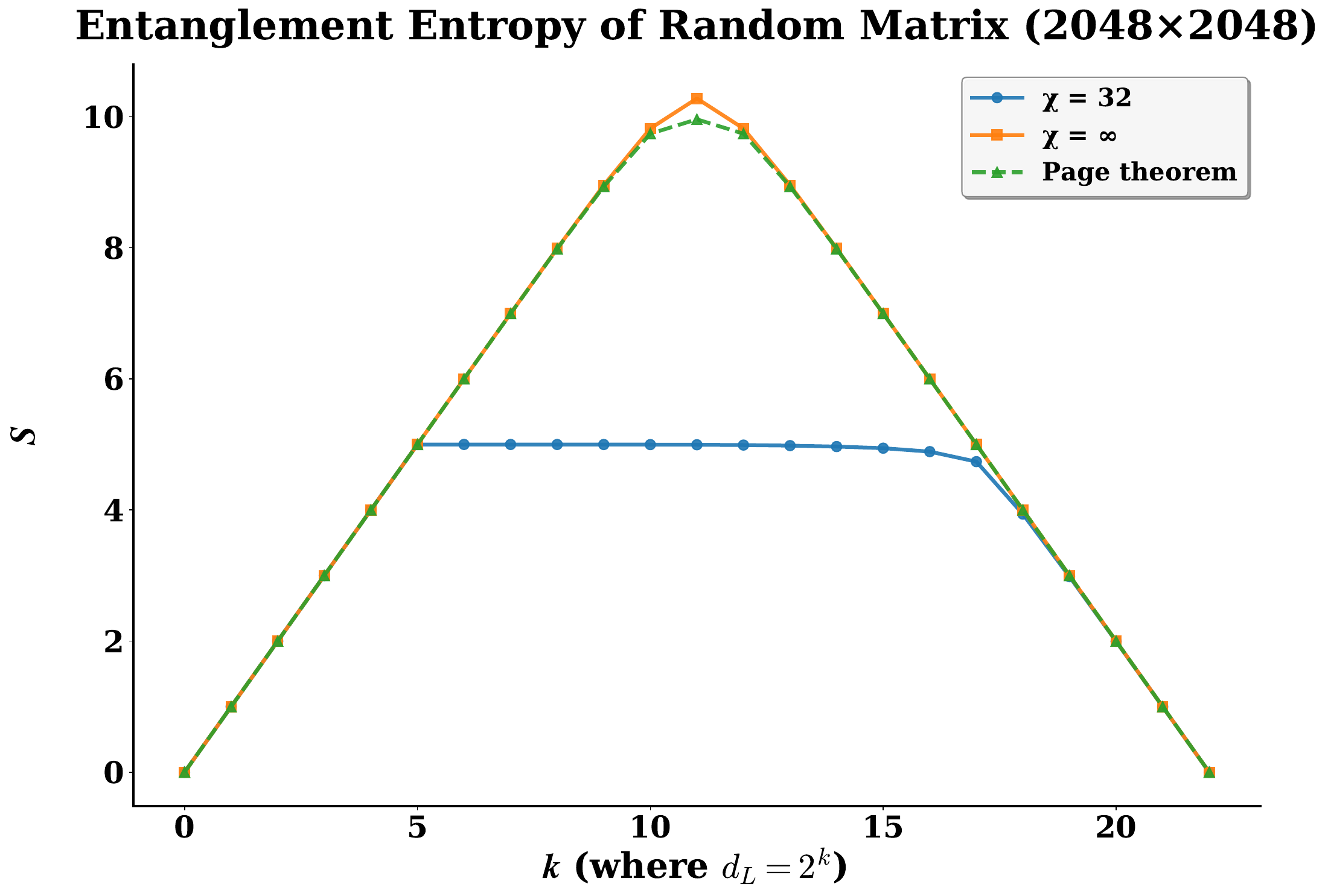}
    \caption{Artificial entanglement profiling of a random Gaussian matrix. The orange curve (\(\chi\) = \(\infty\)) corresponds to the full SVD at each bond and closely matches the Page-curve prediction for Haar-random states (green dashed line). The blue curve ($\chi=32$) demonstrates the effect of truncating the MPS bond dimension the entropy saturates once the Schmidt rank exceeds \(\chi\), forming a plateau, while the true entropy (the orange) continues to increase toward the Page limit.}
    \label{fig:random_matrix_entropy}
\end{figure}

\textbf{Volume law in a random Gaussian matrix.} 
A tensor (including a matrix) with i.i.d.\ Gaussian entries, when reshaped into a vector in a tensor-product space and properly normalized, is statistically equivalent to
sampling a Haar-random pure state~\cite{mezzadri2006generate}.
Here, viewing the tensor as a vector is natural, since an MPS is precisely a structured decomposition of such a high-dimensional vector via successive
Schmidt decompositions.
In the quantum domain, Haar-random pure states serve as a canonical model of
maximally unstructured states, whose bipartite entanglement properties are
well captured by Page theory~\cite{page1993average}.
Accordingly, the von Neumann entanglement entropy $S$ obeys the Page law, which
to leading order reads
$S_{\mathrm{Page}}(d_L) \simeq \log d_L - \frac{d_L}{2 d_R \ln 2}$
for $d_L \le d_R$, where $d_L$ and $d_R$ denote the dimensions of the left and right subsystems respectively after a bi-partition. This implies a \textit{volume-law} scaling of entanglement, where $S$ increases linearly with the size of the smaller subsystem. For Haar-random states, this linear growth further leads to a saturation near the mid cut. We show this profiling curve by sampling a real-valued random matrix \(R \in \mathbb{R}^{m \times n}\) with i.i.d. Gaussian entries, with the total dimension \(D = mn\) where we set \(m, n = 2048\). We then factorize \(D\) into a list of local dimensions using \emph{prime factorization}, where in this case the factor is 2. Then a similar construction of the MPS is conducted (See Sec.~\ref{sec:sketch_MPS}). At each cut, we retain at most \(\chi_{\text{max}}\) singular values, where \(\chi_{\text{max}}=32 \text{ or } \infty\). Therefore the exact \(\chi = \min\{\chi_{\text{max}}, d_L, d_R\}\). FIG.~\ref{fig:random_matrix_entropy} shows this artificial entanglement profiling curve, where we also draw \(S_{\mathrm{Page}}\) as an analytic benchmark. FIG.~\ref{fig:random_matrix_entropy} here also provides a benchmark profiling about how an artificial entanglement profiling is visualized in our study. 

\textbf{Default Experimental Setups.} In the main text, we present the results in fine-tuning the Llama-3.2-1B-Instruct, an instruction-tuned model with 1 billion parameters from  Llama-3.2 family~\cite{touvron2023llama,dubey2024llama}, on a subset of the tulu-3-sft-mixture instruction-tuning (Tulu3) corpus~\cite{lambert2024tulu}. For more experiments in other model and dataset setups, see \textit{Appendix \ref{app:more_experimental_results}}. Unless otherwise stated, we adopt and keep the following settings: we uniformly subsample $5{,}000$ examples from the full dataset and reserve $5\%$ as a test split. For LoRA we use rank $r=256$ and a scaling hyperparameter $\alpha = 16$, applied only to the self-attention query and value projections (\(W_Q\) and \(W_V\))~\cite{hu2022lora}. This is to follow the default settings~\cite{schulman2025lora}, serving as representative cases to study the resulting behaviors. The learning rate is set to $3\times 10^{-4}$ for LoRA and \(3 \times 10^{-5}\) for FFT, since under this setup the testing loss is approximately equal~\cite{schulman2025lora}. The test loss is defined as the average token-level cross-entropy, \ie, negative log-likelihood (NLL). We quantify artificial entanglement using the von Neumann entropy~\cite{eisert2008area,orus2019tensor,berezutskii2025tensor} and denote it as \(S\), the standard measure of entanglement in quantum information theory.

\begin{figure}[!t]
    \centering
\includegraphics[width=\linewidth]{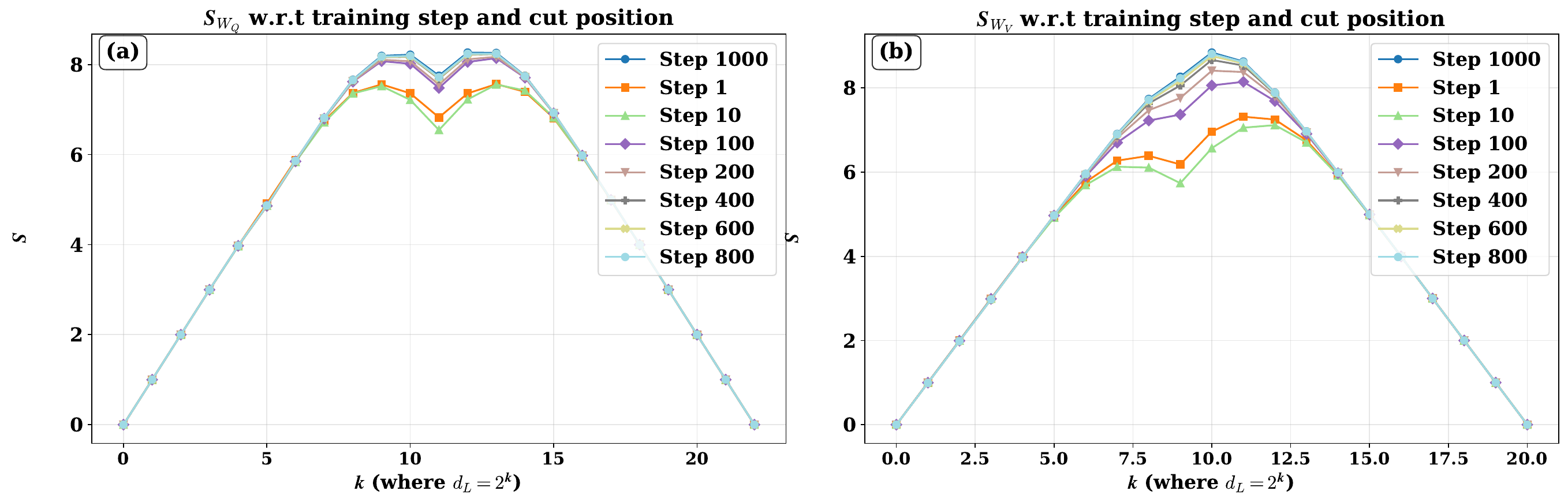}
    \caption{
Artificial entanglement profiling of $\Delta W_Q$ (a) and $\Delta W_V$ (b) across different bi-partition positions $k$ during FFT. Each curve corresponds to a training step, which shows how \(S\) evolves and gradually converges as fine-tuning progresses.
}
    \label{fig:full_wq_wv}
\end{figure}

\begin{figure}[!t]
    \centering
\includegraphics[width=\linewidth]{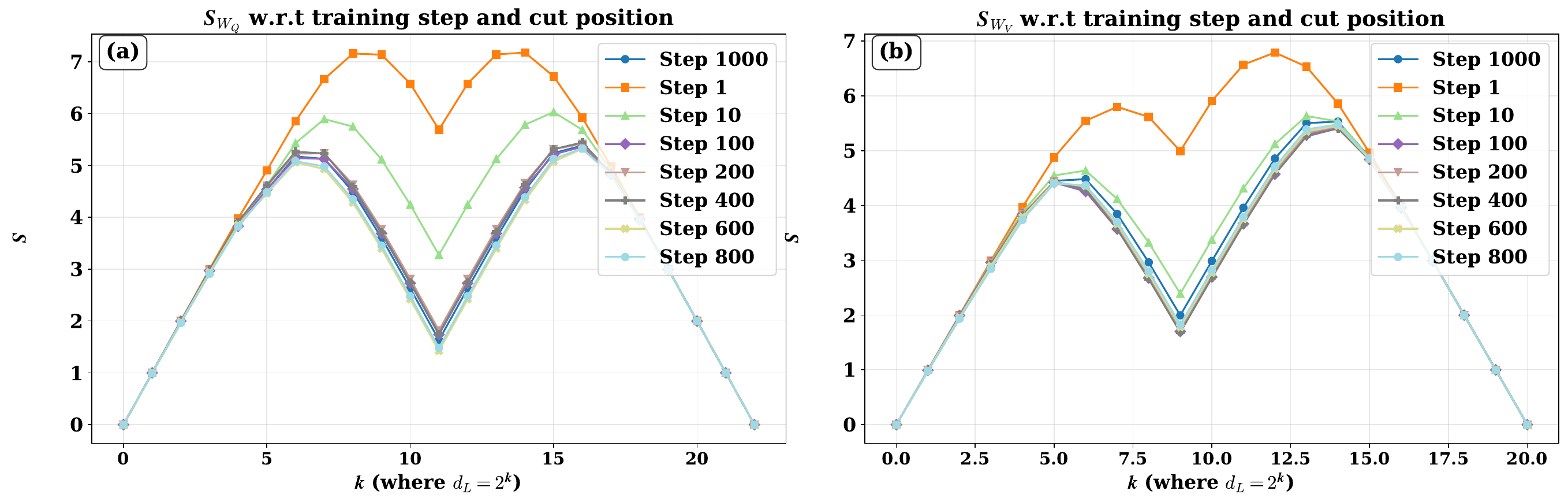}
    \caption{Artificial entanglement profiling of $\Delta W_Q$ (a) and $\Delta W_V$ (b) during LoRA fine-tuning, similar to FIG.~\ref{fig:full_wq_wv}.}
    \label{fig:lora_wq_wv}
\end{figure}

\textbf{Artificial Entanglement Profiling \(S_{\Delta W_Q}\) and \(S_{\Delta W_V}\) \footnote{For notational convenience, we interchangeably use $S_{W_Q}$ and $S_{W_V}$ to denote $S_{\Delta W_Q}$ and $S_{\Delta W_V}$ respectively.} when \(\alpha\) is small\footnote{Here ``small'' refers to relatively small values of the scaling parameter $\alpha$ (\eg, $\alpha = 16$) compared to larger values used in experiments (e.g., $\alpha = 256$ or higher).}.}
We compare the artificial entanglement profiling of FFT and LoRA across bonds, see FIG.~\ref{fig:full_wq_wv} and FIG.~\ref{fig:lora_wq_wv}. 
Both methods exhibit overall \textbf{volume-law} scaling of entanglement with a \emph{central dip}. We term it as the \emph{entanglement valley} (A detailed analysis is provided in \textit{Appendix \ref{app:SVD_dist_cut}}). However, their entanglement evolution w.r.t time differs substantially: In LoRA with a small scaling parameter $\alpha = 16$, the valley deepens as training progresses, making the valley more pronounced, whereas FFT gradually lifts the valley. This contrast is also type-dependent: in $W_V$, the valley almost disappears in later FFT steps, while in $W_Q$ it remains clearly visible even at late time. 

\begin{remark}
\textit{From the perspective of quantum information theory, a volume-law entanglement profile signals intrinsically high correlation that requires a large number of effective degrees of freedom to represent, exceeding what a low-rank parametrization such as LoRA can faithfully capture. Therefore, we propose that \textbf{LoRA is fundamentally unable to represent the full entanglement}, resulting in distinct  signatures.}
\end{remark}

\begin{figure}[t!]
    \centering
\includegraphics[width=\linewidth]{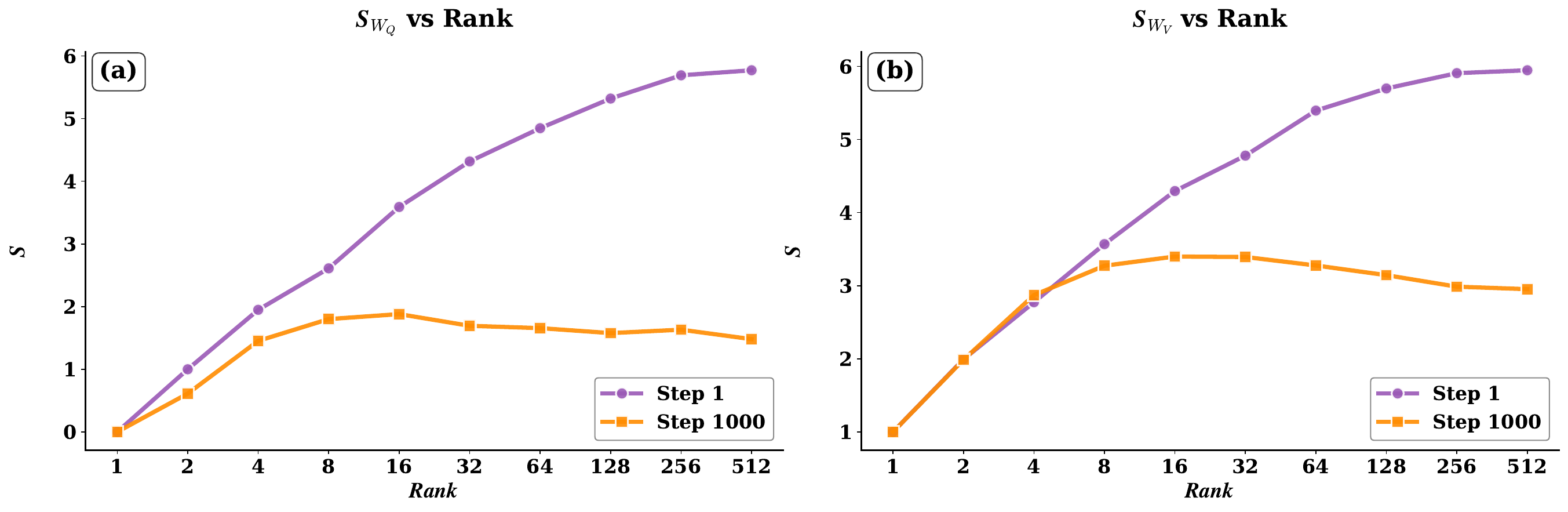}
    \caption{
Artificial entanglement profiling of \(\Delta W_Q\) and \(\Delta W_V\) as a function of LoRA rank \(r\) in different steps. 
(a) $S_{\Delta W_Q}$ at the early time of fine-tuning (Step~1) and the late time (Step~1000). 
(b) $S_{\Delta W_V}$ at the early time of fine-tuning (Step~1) and the late time (Step~1000). These reveal how differing \(r\) and training time jointly shape the entanglement structure of the learned updates.
}
\label{fig:S_WQ_WV_rank_Step_1_1000}
\end{figure}

\begin{figure}[t!]
    \centering
\includegraphics[width=\linewidth]{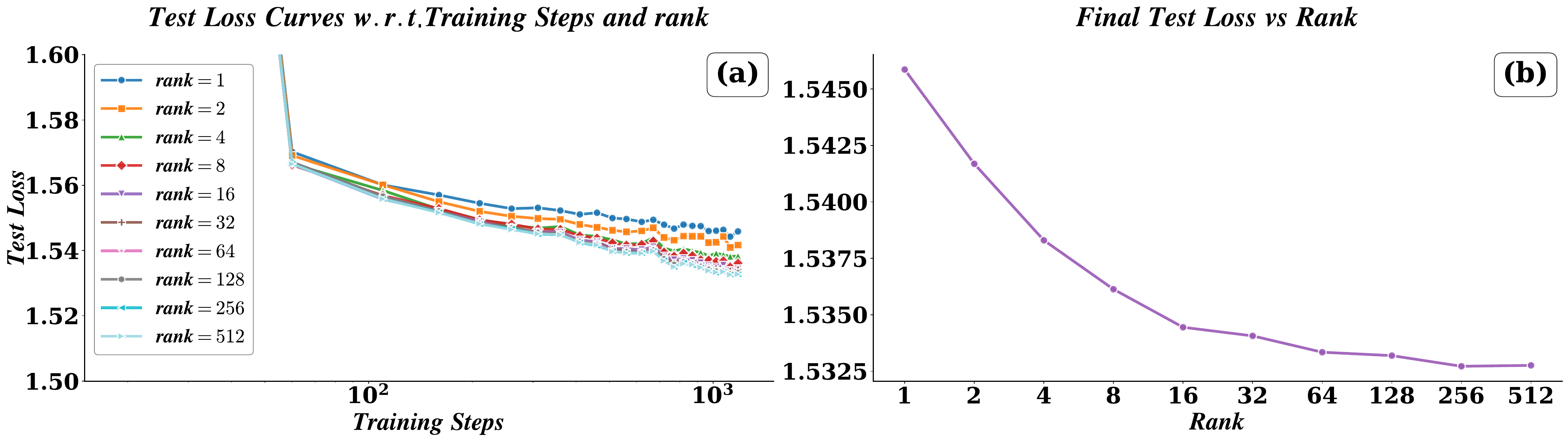}
    \caption{Test loss across training steps for different LoRA ranks \(r\). When \(r\) is small, increasing \(r\) generally improves performance. While increasing to a specific \(r\) the test loss  saturates and is not improved much. We keep other hyper-parameters fixed, such as the learning rate and \(\alpha\).}
\label{fig:test_loss_rank_time}
\end{figure}

\begin{figure}[t!]
    \centering
\includegraphics[width=\linewidth]{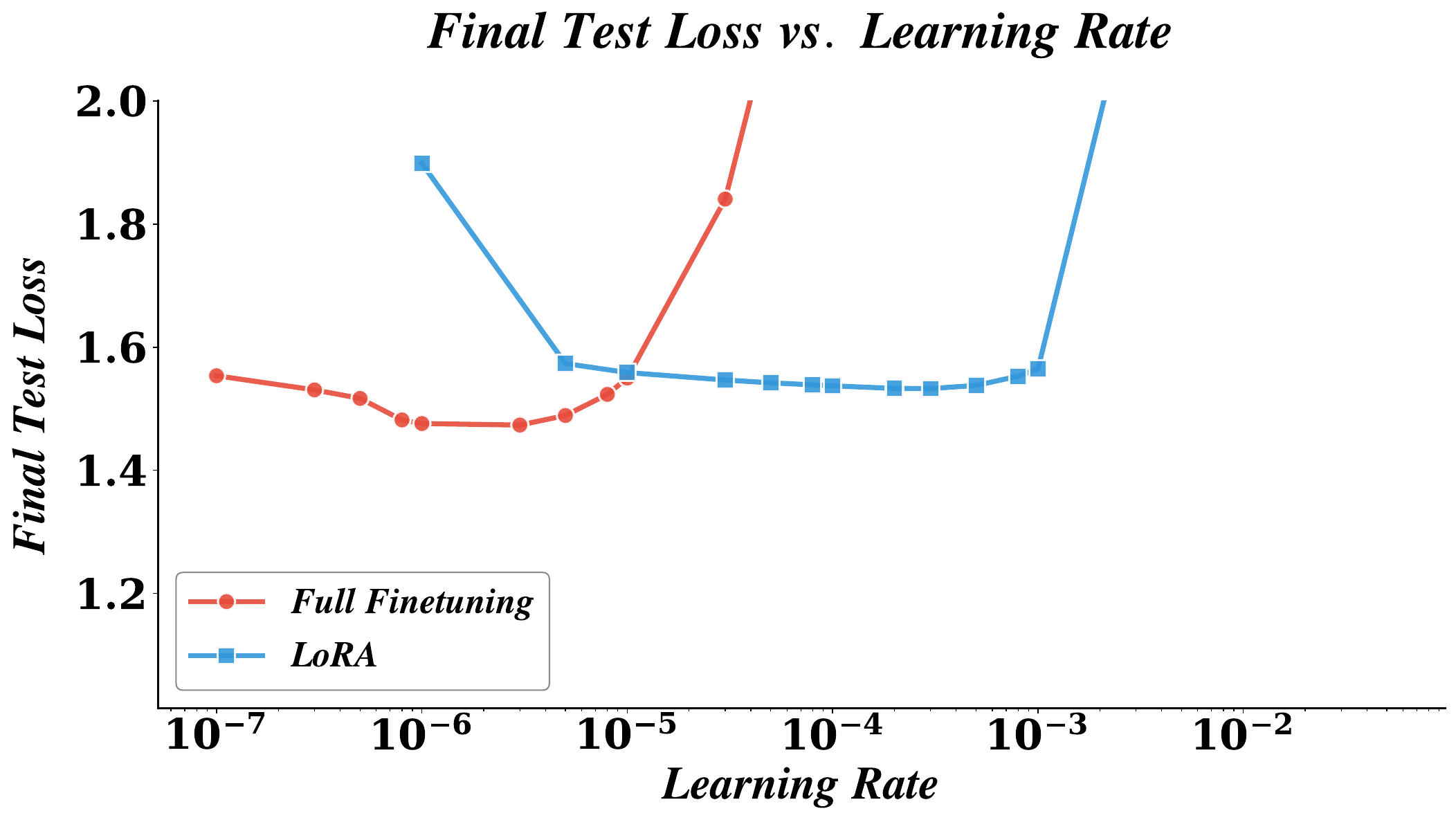}
    \caption{Final test loss as a function of learning rate for FFT and LoRA when the scaling parameter $\alpha$ is small (here $\alpha=16$). LoRA requires a larger learning rate than FFT to achieve the best test performance, consistent with previous observations~\cite{biderman2024lora,schulman2025lora}.}
\label{fig:final_test_loss_vs_lr}
\end{figure}

To study the relationship between rank settings in LoRA and the deepening of the valley, we examine how the LoRA rank \(r\) influences the artificial entanglement profiling of \(\Delta W_Q\) and \(\Delta W_V\). 
FIG.~\ref{fig:S_WQ_WV_rank_Step_1_1000} reports $S_{\Delta W_Q}$ and $S_{\Delta W_V}$ as a function of \(r\), evaluated both at the early time in step 1 and the late time in step 1000. FIG.~\ref{fig:S_WQ_WV_rank_Step_1_1000} shows that: at early time, \(S\) increases monotonically with $r$. This is consistent with viewing the LoRA rank \(r\) as an effective bond dimension $\chi$ of the MPS factorization: a larger $\chi$ enlarges the admissible subspace of tensor network states and raises the maximal entanglement entropy that can be represented across each bond. Consequently, as $r$ (and hence the effective \(\chi\) ) grows, the learned updates can exploit higher entanglement, and the measured \(S\) rises accordingly. In contrast, at late time \(S\) no longer grows with $r$. It instead quickly approaches a plateau and remains nearly unchanged beyond a certain rank. 
Heuristically, this suggests that in LoRA, optimization converges to a relatively low-entanglement solution whose intrinsic correlation is already captured at moderate bond dimension, so that additional entanglement capacity from larger $r$ remains largely unused and does not translate into further growth of entanglement entropy. In contrast, FFT does not exhibit such saturation, as it is not constrained by rank limitations and can represent higher entanglement. This contrast indicates that \textbf{LoRA is consistently unable to represent the full entanglement across different LoRA ranks}. 

We can further observe the empirical relation between how \(r\)  impacts the test loss (See FIG.~\ref{fig:test_loss_rank_time}) and the perspective from \(S\) (FIG.~\ref{fig:S_WQ_WV_rank_Step_1_1000}). In FIG.~\ref{fig:test_loss_rank_time} we visualize the test loss with respect to different \(r\) and training steps, and observe that increasing \(r\) does not necessarily indicate a large improvement in the test performance, along with the saturation of \(S\) (FIG.~\ref{fig:S_WQ_WV_rank_Step_1_1000}). Furthermore, investigating the best learning rate for FFT and LoRA (See FIG.~\ref{fig:final_test_loss_vs_lr}), we observe a similar phenomenon as~\cite{biderman2024lora,schulman2025lora}. From FIG.~\ref{fig:final_test_loss_vs_lr}, LoRA requires a larger learning rate to achieve the best test loss than FFT. These observations motivate us to further investigate how hyper-parameters, where we specifically investigate the scaling parameter $\alpha$, affects the entanglement structure and optimization behavior in LoRA. 

\begin{figure}[t!]
    \centering
\includegraphics[width=\linewidth]{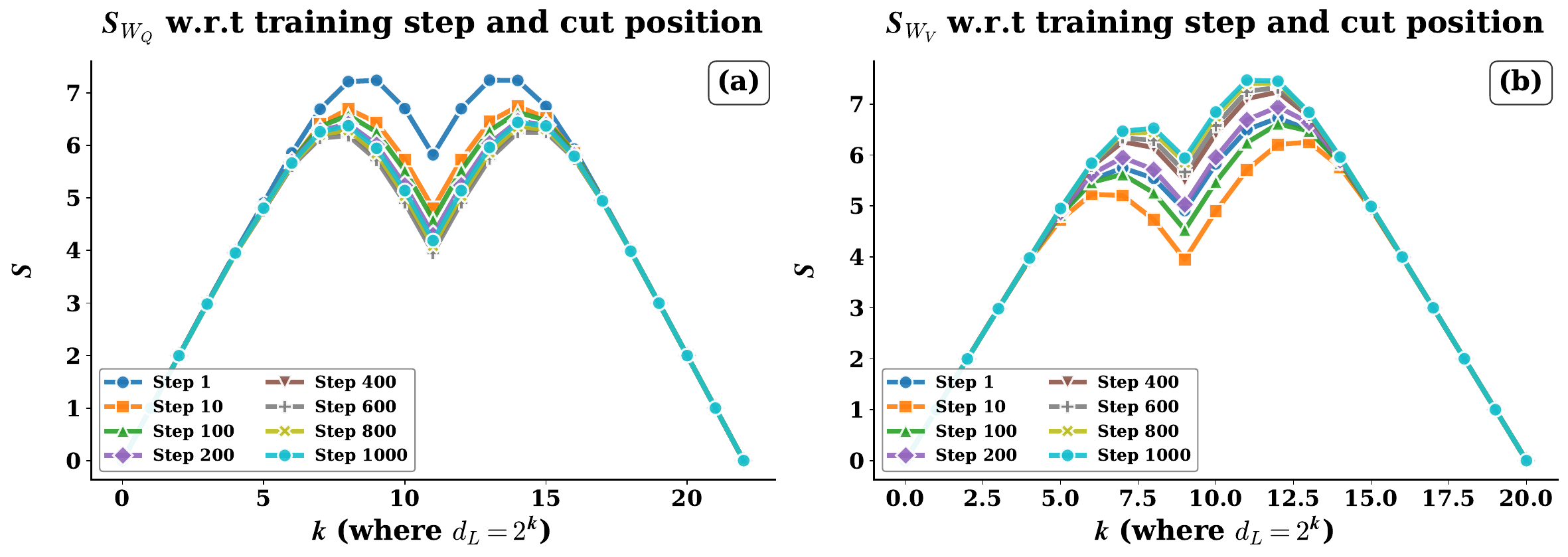}
    \caption{$S_{\Delta W_Q}$ and $S_{\Delta W_V}$ w.r.t the cut position $k$ (with $d_L = 2^k$) across training steps when increasing \(\alpha = 256\). Panels (a) and (b) show that at early time, the curves exhibit a pronounced ``entanglement valley'' structure centered near the middle cuts. As training progresses, the curves of \(\Delta W_Q\) are less deepened as the cases in smaller \(\alpha=16\), while the curves of \(\Delta W_V\) gradually lift and approach a saturated shape.}
\label{fig:large_alpha_WQ_WV}
\end{figure}

\begin{figure}[t!]
    \centering
\includegraphics[width=\linewidth]{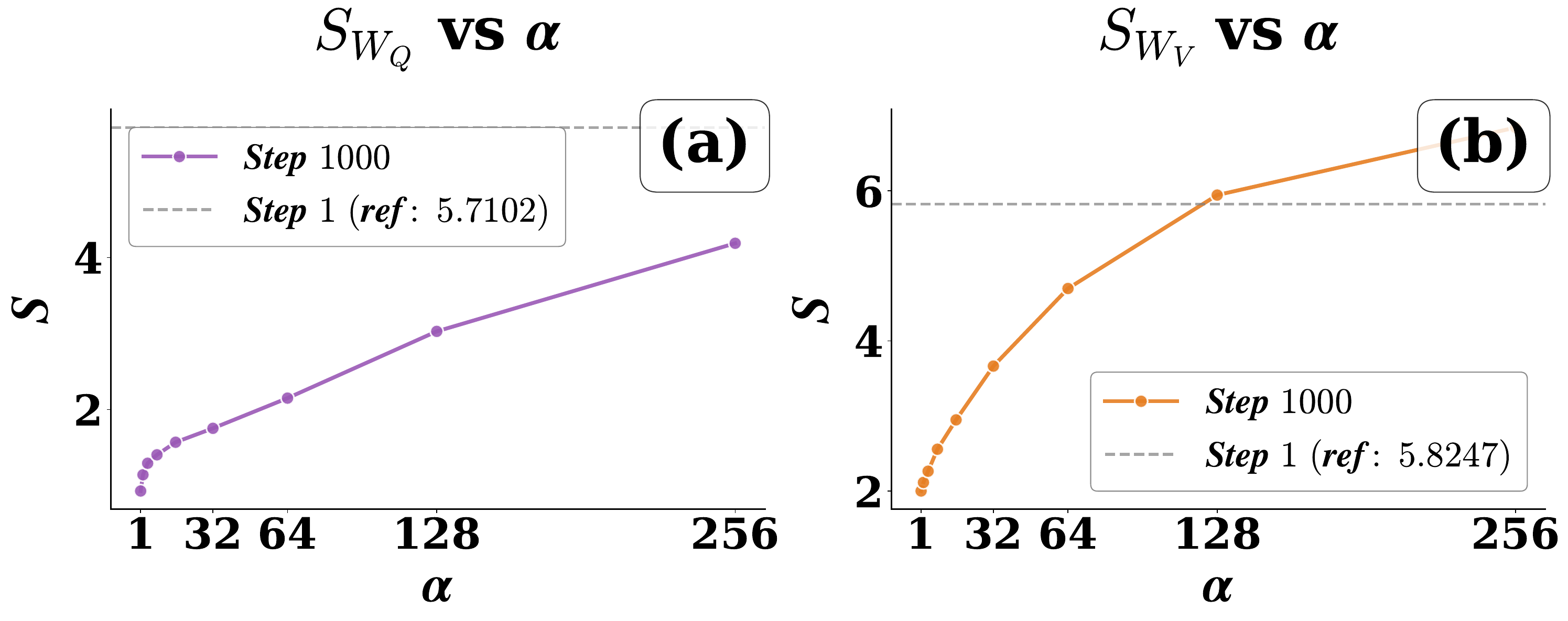}
    \caption{Entanglement entropy $S$ in the middle bi-partition for different $\alpha$. \(S\) in early time (Step = 1) serves as a baseline. By showing the converged solution at late time (Step = 1000), this demonstrates how \(\alpha \in \{1,2,4,8,16,32,64,128,256\}\) impacts the solution learned by LoRA.
}
\label{fig:entropy_evolution_by_alpha}
\end{figure}

\begin{figure}[t!]
    \centering
\includegraphics[width=\linewidth]{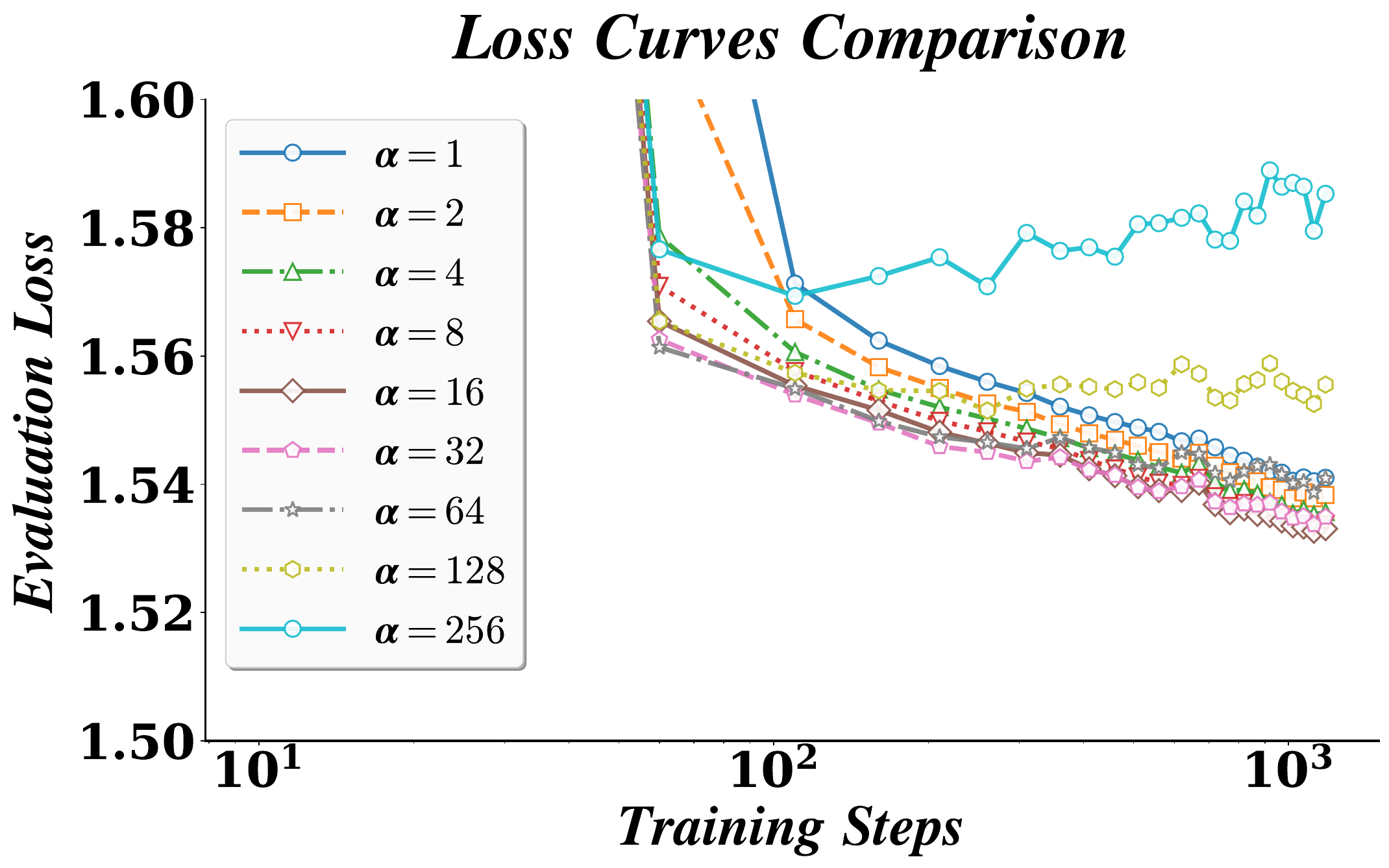}
    \caption{Test loss comparison across the scaling factor $\alpha$ in LoRA. $X$-axis is shown in log scale.}
\label{fig:loss_comparison_log_scale_all_alphas}
\end{figure}

\begin{figure}[t!]
    \centering
\includegraphics[width=\linewidth]{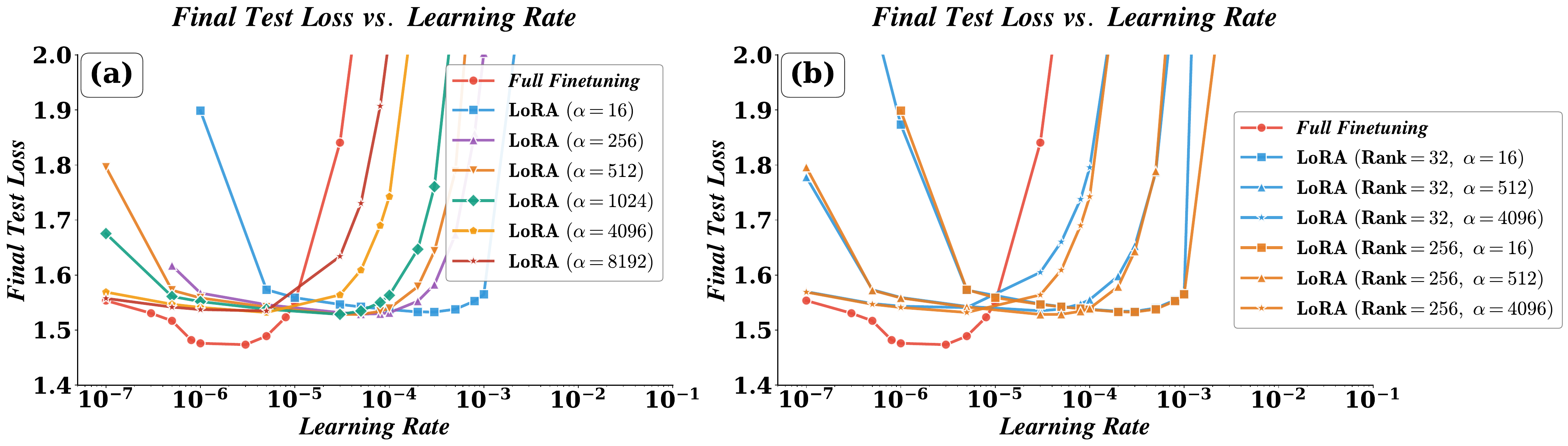}
    \caption{Final test loss with respect to learning rate under different LoRA configurations. (a) Comparison of full fine-tuning with LoRA at multiple $\alpha$ when rank \(r=256\) . Larger $\alpha$ values shift the optimal learning rate to be smaller, approaching the solutions from full-fine-tuning. (b) Joint variation of \(r\) and \(\alpha\) for LoRA. The settings \(r=32\) and \(r=256\) do not shift the optimal learning rate in different \(\alpha\).}
\label{fig:final_test_loss_vs_lr_with_varying_alpha}
\end{figure}

\textbf{Artificial Entanglement Profiling of \(S_{\Delta W_Q}\) and \(S_{\Delta W_V}\) when \(\alpha\) is large.} We now investigate the case when \(\alpha\) is increased. We first tune \(\alpha\) from 16 to 256. From FIG.~\ref{fig:large_alpha_WQ_WV} we see that in this setting LoRA will learn a solution that mitigates the deepening of the entanglement valley in \(\Delta W_Q\) and even lifts the valley in \(\Delta W_V\) during late time, suggesting a closer yet still distinct solution to the one learned by FFT. To better characterize this phenomenon in terms of both \(S\), learning rates and final test performance, we conduct below experiments: 

\underline{\textbf{(i)}} We visualize the entanglement entropy $S$ in the middle bi-partition position with respect to the training steps (Specifically evaluating at the early time in step 1 and the late time in step 1000 respectively) and \(\alpha\), see FIG.~\ref{fig:entropy_evolution_by_alpha};

\underline{\textbf{(ii)}} We visualize how \(\alpha\) changes the final test loss while keeping all other hyper-parameters fixed, see FIG.~\ref{fig:loss_comparison_log_scale_all_alphas}. 

\underline{\textbf{(iii)}} We sweep the learning rates by setting \(\alpha= \{16, 256, 512, 1024, 4096, 8192\}\), where the \(\alpha\) is set higher or equal to the one in FIG.~\ref{fig:final_test_loss_vs_lr}, see FIG.~\ref{fig:final_test_loss_vs_lr_with_varying_alpha} (a). 

We first draw some observations: As \(\alpha\) increases from 1, test performance is initially improved. Then, the test performance reaches the best when \(\alpha\) is around \(\alpha = 16\). However, the test performance is worse as \(\alpha\) increases. Meanwhile, from \(\alpha=1\) to \(\alpha=256\) we can observe an increasing fluctuation along the path. This indicates that the optimal learning rate in LoRA has been shifted, see FIG.~\ref{fig:final_test_loss_vs_lr_with_varying_alpha} (a): the learning rates corresponding to the best performance in LoRA and FFT are within less gap, suggesting more aligned fine-tuning solutions. And the optimal learning rates are not impacted by a relatively small rank, see FIG.~\ref{fig:final_test_loss_vs_lr_with_varying_alpha} (b). Combining these results we can see that the artificial entanglement signatures are actually sensitive to hyper-parameter settings, including \(\alpha\) (FIG.~\ref{fig:entropy_evolution_by_alpha}), time (FIG.~\ref{fig:lora_wq_wv}) and rank (FIG.~\ref{fig:S_WQ_WV_rank_Step_1_1000}). Besides, \textbf{entanglement signatures might serve as an indicator of the combined effect of multiple hyper-parameter choices} (See detailed discussion in Sec.~\ref{sec:discussion}). 

While the above analysis indicates that LoRA and FFT exhibit different artificial entanglement profiles, exhibiting different fine-tuning solutions (\eg, the updates of projection matrices and the optimal learning rate), a puzzle still remains: \textit{ Why are the resulting test losses still approximate even when the artificial entanglement profiles are different?} When only focusing on LoRA,  we can formulate this problem as: \textit{Why different settings of hyper-parameters that result in different artificial entanglement profiles in LoRA can still result in similar test loss (\eg, see FIG.~\ref{fig:loss_comparison_log_scale_all_alphas}, the final test loss is approximately the same when \(\alpha=1\) and \(\alpha = 64\), ignoring the fluctuations)?} 

To explore this problem, we extend our artificial entanglement analysis framework from only the updates of \(W_Q\) and \(W_V\) (\ie, \(\Delta W_Q\) and \(\Delta W_V\)) to attention matrix \(A(X_0)\) and attention outputs \(X\), specifically the \(XX^T\). This choice is natural because $XX^\top$ captures the pairwise correlations between token representations after attention mixing, analogous to how correlation is probed in quantum many-body systems via reduced density matrices (\ie, $XX^\top$ serves as the analogue of an output-state density operator).

\begin{figure}[t!]
    \centering
\includegraphics[width=\linewidth]{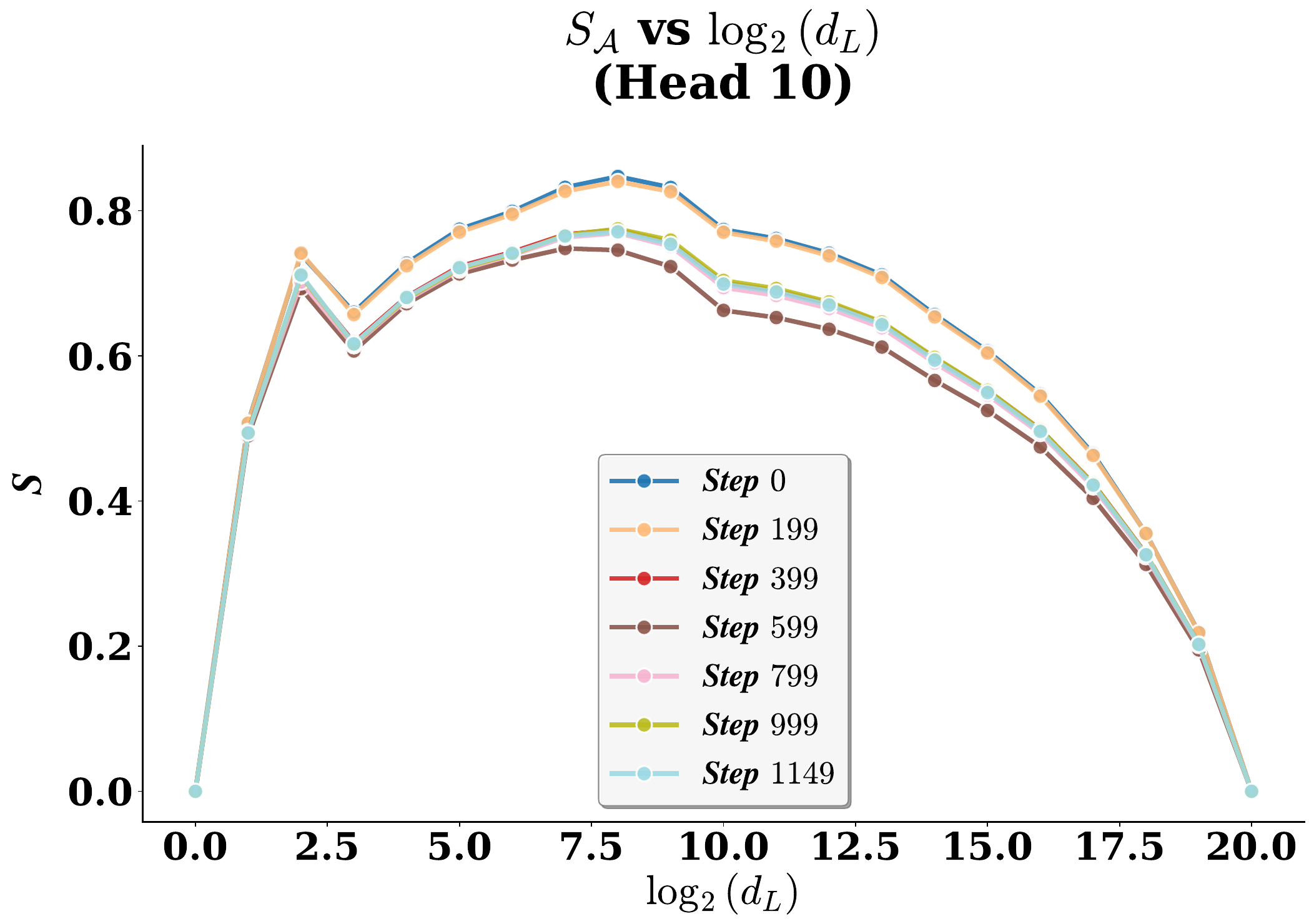}
    \caption{Artificial Entanglement Profiling of the attention matrix (Head 10) $S_A$ across training steps. 
$S_A$ remains uniformly low throughout training and exhibits a characteristic of an approximate area-law scaling, where the entropy grows only mildly that roughly following a shallow logarithmic rise. This indicates that the attention matrix operates in a strongly low-entanglement regime.}
\label{fig:attention_entropy_evolution_head10}
\end{figure}

\begin{figure}[t!]
    \centering
\includegraphics[width=\linewidth]{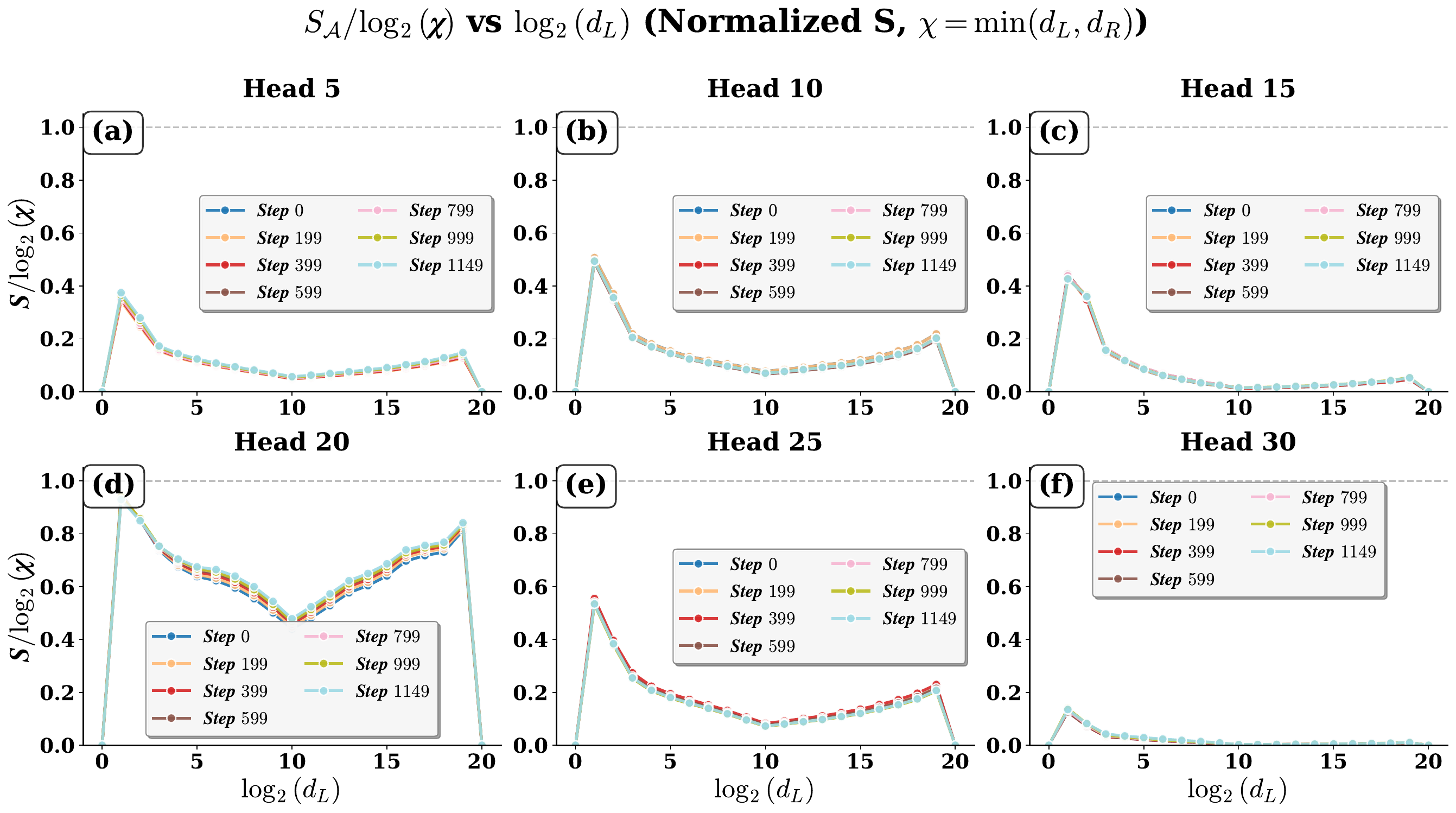}
    \caption{Normalized entanglement entropy $S_A/\log(\chi)$ of six representative attention heads across training steps, where $\chi=\min(d_L,d_R)$ denotes the bond dimension, and $\log(\chi)$ represents the theoretical maximum entanglement entropy. Most heads exhibit consistently mild scaling throughout training, far from the volume law, indicating that the effective correlations encoded by the attention mechanism remain far below the theoretical maximum.}
\label{fig:attention_entropy_normalized_evolution_all_heads}
\end{figure}

\begin{figure}[t!]
    \centering
\includegraphics[width=\linewidth]{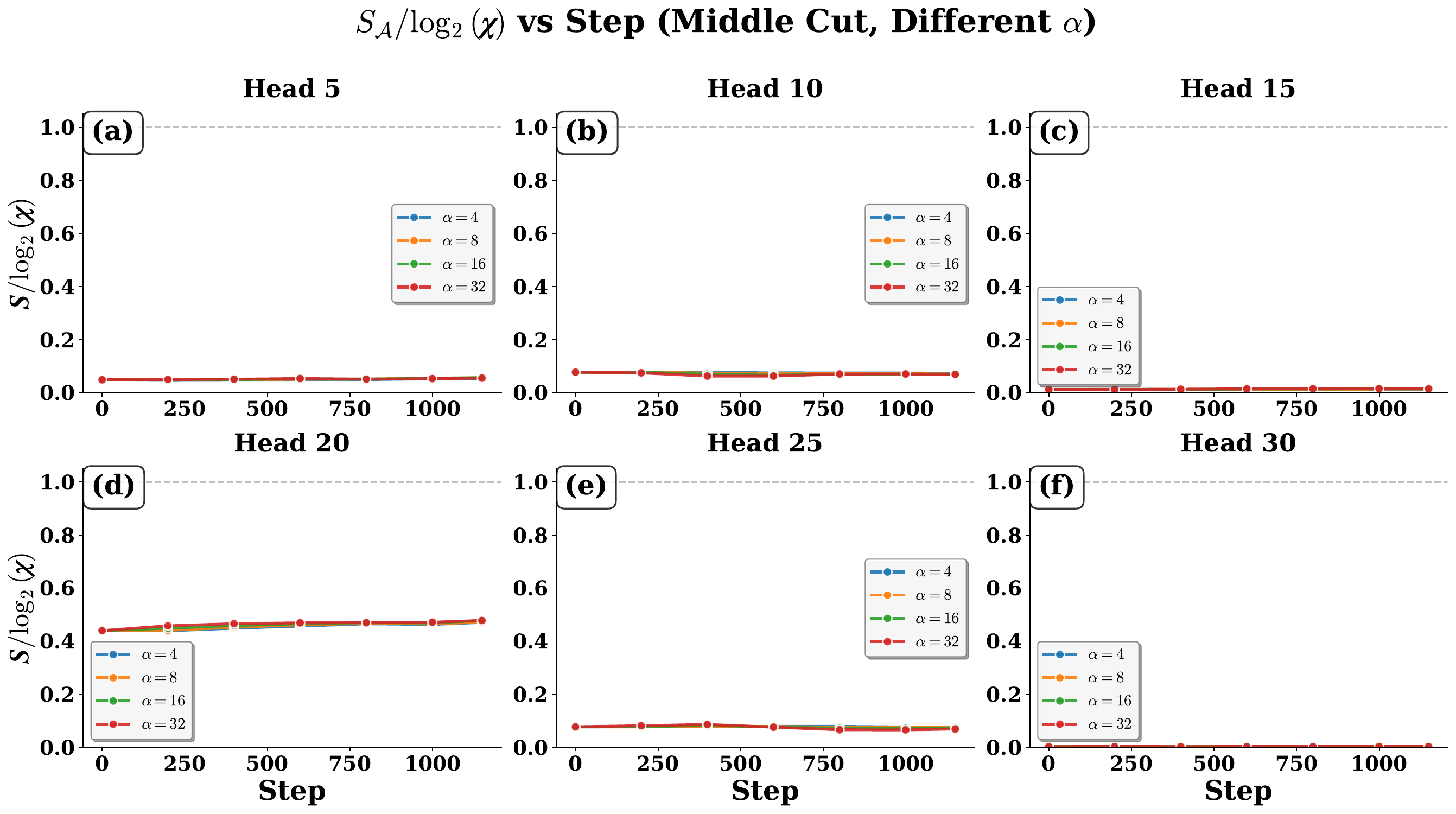}
    \caption{
Normalized entanglement entropy $S_A/\log(\chi)$ of six representative 
attention heads across training steps and LoRA scaling coefficients $\alpha$. 
\(S\) remains significantly below the theoretical maximum and shows negligible 
variation with either training time or the choice of $\alpha$, suggesting a no-hair like behavior.}
\label{fig:attention_entropy_normalized_by_alpha_all_heads}
\end{figure}

\begin{figure}[t!]
    \centering
\includegraphics[width=\linewidth]{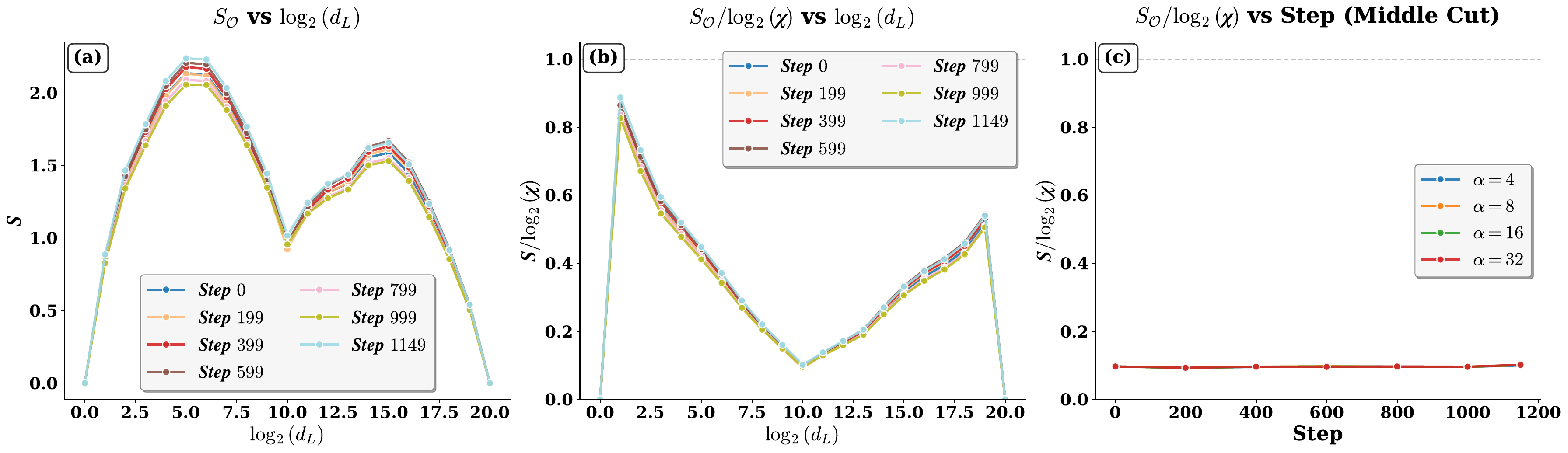}
\caption{
Entanglement entropy $S_{\mathcal{O}}$ of the output operator 
$\mathcal{O} = XX^{\top}$ across bi-partition positions and training steps.  
Panels (a) and (b) show that $S_{\mathcal{O}}$ remains significantly below the theoretical maximum. Panel (c) illustrates that this behavior is invariant across the variation of scaling coefficient $\alpha$ and training steps,
indicating a no-hair like behavior as well.}
\label{fig:attention_output_entropy_combined}
\end{figure}

\textbf{Artificial Entanglement Profiling of \(S_A\) and \(S_{XX^T}
\).} Generally, for an input $\mathcal{X}_0 \in \mathbb{R}^{B \times T \times d_{\mathrm{model}}}$, we have the attention tensor \(\mathcal{A}(\mathcal{X}_0) \in \mathbb{R}^{B \times H \times T \times T}\) where \(B\) is the batch size, \(T\) is the number of tokens, \(d_{\mathrm{model}}\) is the token embedding dimension and \(H\) is the number of heads. Here we select samples where the sequence length \(T\) can be directly factorized by 2 (\eg, \(T=1024\)). Then, we focus on each attention matrix in separate attention head. Therefore, for each head, we obtain a matrix \(A \in \mathbb{R}^{T \times T}\) that can be formalized to an MPS to profile the artificial entanglement. Here we show an entanglement profiling specifically in Head~10 (More heads are shown in normalized form in FIG.~\ref{fig:attention_entropy_normalized_evolution_all_heads}), see
FIG.~\ref{fig:attention_entropy_evolution_head10}. Across all training steps, the entanglement entropy $S_A$ remains uniformly low. Such a profile is consistent with an approximate \textbf{area law scaling}: \(S\) exhibits a shallow, nearly logarithmic rise as the bi-partitioning position moving (in this case from right to left across the MPS, followed by a saturation around the central partition). Though not a strict area law scaling, the profiling deviates much from volume law and can often be treated as area law with logarithmic correction~\cite{eisert2008area}. This behavior indicates that the correlations encoded by this attention head are highly localized and remains far from the high-entanglement (volume-law) regime. FIG.~\ref{fig:attention_entropy_normalized_evolution_all_heads} shows that the phenomenon is prevalent across multiple heads through normalizing the \(S\) via dividing it by the theoretical allowed entanglement \(\log(\chi)\), where \(\chi = \min(d_L, d_R)
\). Notably, \(\chi\) changes across different cuts, so as the theoretical allowed entanglement \(\log(\chi)\). Specifically, \(\chi\) reaches its maximum at the central partition (where \(d_L = d_R\)), making the theoretical allowed entanglement largest at the middle cuts. This variation in the normalization factor makes the dip in the middle of the normalized entropy curves more pronounced. Besides, a \textbf{``no hair'' property exhibits:} the entanglement structures in attention are largely insensitive with respect to training steps and \(\alpha\), see FIG.~\ref{fig:attention_entropy_normalized_by_alpha_all_heads}. In other words, even though different hyper-parameters or fine-tuning methods lead to different artificial entanglement profiling~(\eg, FIG.~\ref{fig:full_wq_wv}, FIG.~\ref{fig:lora_wq_wv}, FIG.~\ref{fig:entropy_evolution_by_alpha}), from the
resulting attention matrices those entanglement signature variations do not exhibit. Similar phenomena can be observed in the attention output operator \(\mathcal{O} = X X^T\), see FIG.~\ref{fig:attention_output_entropy_combined}. We suggest two reasons might prompt these phenomena: First, the entanglement between text tokens is significantly low; Second, masking operation results in many values to be zero after the \textit{softmax} operation due to the auto-regressive nature, further resulting in a relatively low entanglement structure. Regarding the second, we conduct an experiment to see how masking operation influences the results, see \textit{Appendix \ref{app:abla_mask}} for details.

\subsection{A Perspective Based on Random Matrix Theory}
\label{sec:random_matrix_theory}

We study the artificial entanglement of the attention matrix and the attention output under an asymptotic model based on random matrix theory. Under assumptions consistent with~\cite{saadamind}, we explain the area law scaling with a logarithmic correction observed in the artificial entanglement of the attention matrix. Furthermore, we identify the conditions under which the attention output exhibits low entropy $S$ as the number of tokens increases, and discuss the relationship between these findings and rank collapse in width~\cite{saadamind}. We sketch the theories here, and leave full preliminaries and details in \textit{Appendix~\ref{app:quantum_entanglement} and Appendix~\ref{app:theorem_and_proof}}. While the main focus is von Neumann entropy that is the standard measure in quantum information theory, our theoretical results also include the analysis of Rényi-2 entropy that depends only on the purity $\mathrm{Tr}(\rho^2)$ of the reduced state, since Rényi-2 entropy provides a more tractable quantity for theoretical analysis and offers simpler closed-form expressions. Additionally, Rényi-2 entropy provides lower bounds on von Neumann entropy, useful for theoretical characterization. For a detailed discussion of the relationship between von Neumann and Rényi entropies, see \textit{Appendix~\ref{app:von_neumann_renyi}}.

\textbf{The Attention Cardy Formula: Criticality in the Bulk.} Following the outlier–bulk law established for softmax attention at isotropic initialization~\cite{saadamind}, we decompose the attention matrix as
$
A = \tfrac{1}{T}\mathbf{1}\mathbf{1}^\top + A^\perp.
$
Here, the rank-$1$ component \(\tfrac{1}{T}\mathbf{1}\mathbf{1}^\top\) captures the mean-field behavior, while the residual $A^\perp$ governs the fluctuations, converging to a quartercircular law (a deterministic distribution of singular values shaped like a quarter-circle in the large-$T$ limit, see \textit{Appendix~\ref{app:theorem_and_proof}} for details). Below
we establish a logarithmic scaling for entanglement entropy, mathematically analogous to the celebrated Cardy formula in conformal field theory (CFT)~\cite{calabrese2004entanglement,calabrese2009entanglement}:

\begin{theorem}[The Attention Cardy Formula]
\label{thm:attn_entropy_main}
Under Assumption~\ref{ass:outlier_bulk}, let \(T\) be the sequence length (number of tokens).
In the limit $T\to\infty$, the von Neumann entropy scales as:
\begin{equation}
\label{eq:cardy_law}
    S(A) = \mathcal{C}_{\mathrm{attn}} \log T + \text{const} + o(1),
\end{equation}
where the prefactor $\mathcal{C}_{\mathrm{attn}}$ (termed as the \textbf{Effective Attention Charge}) is determined by the bulk spectrum spread parameter $\sigma>0$ of the limiting quartercircular law $Q_\sigma$ for the rescaled singular values of $\sqrt{T}\,A^\perp$:
\begin{equation}
    \mathcal{C}_{\mathrm{attn}} := \frac{\sigma^2}{1+\sigma^2}.
\end{equation}
See Assumption~\ref{ass:outlier_bulk} for details on $\sigma$.
\end{theorem}

Specifically, $\sigma$ controls the width of the bulk spectrum distribution: larger $\sigma$ corresponds to a more spread-out distribution of singular values in the bulk, leading to higher entanglement entropy. The parameter $\sigma$ is determined by the second moment $m_2$ of the quartercircular law via $m_2 = \sigma^2$, which in turn relates to the Frobenius norm of $A^\perp$ as $\|A^\perp\|_F^2 \to \sigma^2$ in the large-$T$ limit. For results on Rényi-2 entropy and its relationship to von Neumann entropy, see \textit{Appendix~\ref{app:von_neumann_renyi}}.

\textit{Sketch of proof.}
The decomposition $\|A\|_F^2 = 1 + \|A^\perp\|_F^2$ combined with the quartercircular limit implies $\|A^\perp\|_F^2 \to \sigma^2$. Since the operator norm of $A^\perp$ scales as $O(T^{-1/2})$, the bulk singular values satisfy $s_i(A^\perp) = O(T^{-1/2})$ for $i \ge 2$, and the entanglement spectrum weights are $p_i = s_i(A)^2/\|A\|_F^2 = O(1/T)$ for $i \ge 2$.
For the von Neumann entropy, the tail contribution $-\sum_{i\ge 2} p_i\log p_i$ is computed using the representation $s_i(A^\perp)^2 = (x_i^{(T)})^2/T$ where $x_i^{(T)}$ are the rescaled singular values converging to the quartercircular law. Expanding $\log p_i = \log(s_i(A^\perp)^2/\|A\|_F^2) = \log(x_i^{(T)})^2 - \log T - \log\|A\|_F^2$, the $\log T$ term emerges from the sum $\sum_i (x_i^{(T)})^2/T \to \sigma^2$, yielding the prefactor $\sigma^2/(1+\sigma^2)$ in the logarithmic scaling. The convergence follows from treating the normalized sum as a Riemann sum that converges to the integral defined by the limiting quartercircular distribution. This spectral accumulation is structurally identical to the density of states in gapless quantum systems. Full details are in \textit{Appendix~\ref{app:theorem_and_proof}}.

\begin{remark}[From CFT to Attention Cardy Formula]
We draw a parallel between our result and the Cardy formula ($S \sim \frac{c}{3}\log L$) for critical quantum systems. In physics, strictly gapped systems follows a strict ``area law" where \(S\) saturates to a constant ($S \sim O(1)$), implying a finite correlation length. In contrast, Eq.~\eqref{eq:cardy_law} shows that the attention mechanism follows an \textbf{area law with a logarithmic correction} ($S \sim \log T$).
This logarithmic divergence is the hallmark of a \textit{critical phase} \cite{eisert2008area}, allowing the model to maintain non-vanishing dependencies between distant tokens. Therefore
we refer to this critical scaling behavior as the ``Attention Cardy Formula'', highlighting its role in enabling long-context capabilities beyond the limitations of strict area-law saturation.
\end{remark}

\textbf{A bridge between stable rank and entanglement.} To study the \emph{output} $X = AV$, we first relate entanglement to the stable rank
of $\Sigma = XX^\top$.  The following lemma shows that when the stable rank~\cite{saadamind} is close to $1$, the entanglement must be small.

\begin{lemma}[Stable rank implies small entanglement]
\label{lem:sr_to_entropy_main}
Let $\Sigma = XX^\top$ with eigenvalues $\lambda_1 \ge \cdots \ge \lambda_T$ and stable rank $r_{\mathrm{stable}} = \mathrm{Tr}(\Sigma^2) / \lambda_1^2$. Let $\eta = r_{\mathrm{stable}}-1$ and $\delta_1 := \sum_{i\ge 2}\frac{\lambda_i}{\lambda_1}$ (the sum of ratios of non-dominant to dominant eigenvalues). Then
\begin{equation}
S(X) \;\le\; h_2(\delta) + \delta\log(T-1),
\end{equation}
where $\delta_1 \le \sqrt{(T-1)\eta}$, $\delta=\min\{1,\delta_1\}$, and $h_2(u):=-u\log u-(1-u)\log(1-u)$. For Rényi-2 entropy bounds, see \textit{Appendix~\ref{app:von_neumann_renyi}}.
\end{lemma}

\textit{Sketch of proof.}
By Cauchy--Schwarz, $(\sum_{i\ge2}\lambda_i)^2 \le (T-1)\sum_{i\ge2}\lambda_i^2$. Dividing by $\lambda_1^2$ and using the definition of $\eta = r_{\mathrm{stable}}-1 = \sum_{i\ge2}(\lambda_i/\lambda_1)^2$ gives $\delta_1 \le \sqrt{(T-1)\eta}$.
Since $\mathrm{Tr}(\Sigma) = \lambda_1(1+\delta_1)$, the normalized density matrix $\rho = \Sigma / \mathrm{Tr}(\Sigma)$ has top eigenvalue $\lambda_{\max}(\rho) = 1/(1+\delta_1)$, implying $1-\lambda_{\max}(\rho) = \delta_1/(1+\delta_1) \le \delta$.
With $p_1 \ge 1-\delta$ and tail mass $\sum_{i\ge2}p_i \le \delta$, the von Neumann entropy is maximized when the tail mass is uniformly distributed across the remaining $T-1$ eigenvalues, yielding the bound $S(\rho) \le h_2(\delta) + \delta\log(T-1)$.
For Rényi-2 entropy, using $\mathrm{Tr}(\rho^2) \ge p_1^2 = 1/(1+\delta_1)^2$ gives $S_2(\rho) \le 2\log(1+\delta_1)$.
\textit{Appendix~\ref{app:theorem_and_proof}} provides the full arguments.

\textbf{Output entanglement collapse under stable rank collapse.} Recent theory~\cite{saadamind} shows that in width-limited regimes, the
stable rank of $XX^\top$ may collapse to $1+O(T^{-3})$.
Plugging this into Lemma~\ref{lem:sr_to_entropy_main} yields:

\begin{theorem}[Output entanglement collapse]
\label{thm:output_entropy_main}
If $r_{\mathrm{stable}}(XX^\top)-1 = O(T^{-3})$ with overwhelming probability\footnote{Overwhelming probability means $1-T^{-c}$ for every fixed $c>0$, for all sufficiently large $T$.}, then in the limit $T\to\infty$,
\begin{equation}
S(X) = O\!\left(\frac{\log T}{T}\right)\to 0.
\end{equation}
Thus the token--feature entanglement vanishes in the rank-collapsed regime. For Rényi-2 entropy scaling, see \textit{Appendix~\ref{app:von_neumann_renyi}}.
\end{theorem}

\textit{Sketch of proof.}
Apply Lemma~\ref{lem:sr_to_entropy_main} with $\eta = r_{\mathrm{stable}}(\Sigma)-1 = O(T^{-3})$. 
The bound $\delta_1 \le \sqrt{(T-1)\eta}$ from the lemma gives $\delta_1 = O(T^{-1})$, and hence $\delta = \min\{1, \delta_1\} = O(T^{-1})$.
The lemma yields $S(X) \le h_2(\delta) + \delta\log(T-1)$. 
Since $h_2(\delta) = O(\delta\log(1/\delta))$ and $\delta = O(1/T)$, we have $h_2(\delta) = O((\log T)/T)$, while $\delta\log(T-1) = O((\log T)/T)$.
Therefore $S(X) = O((\log T)/T) \to 0$.
For Rényi-2 entropy, the lemma gives $S_2(X) \le 2\log(1+\delta_1) = O(\delta_1) = O(1/T) \to 0$, vanishing even faster.
This reflects that almost all spectral weight concentrates on the top eigenvalue when the stable rank approaches $1$.
Full details are given in \textit{Appendix~\ref{app:theorem_and_proof}}.

\begin{figure}[htbp]
    \centering
\includegraphics[width=\linewidth]{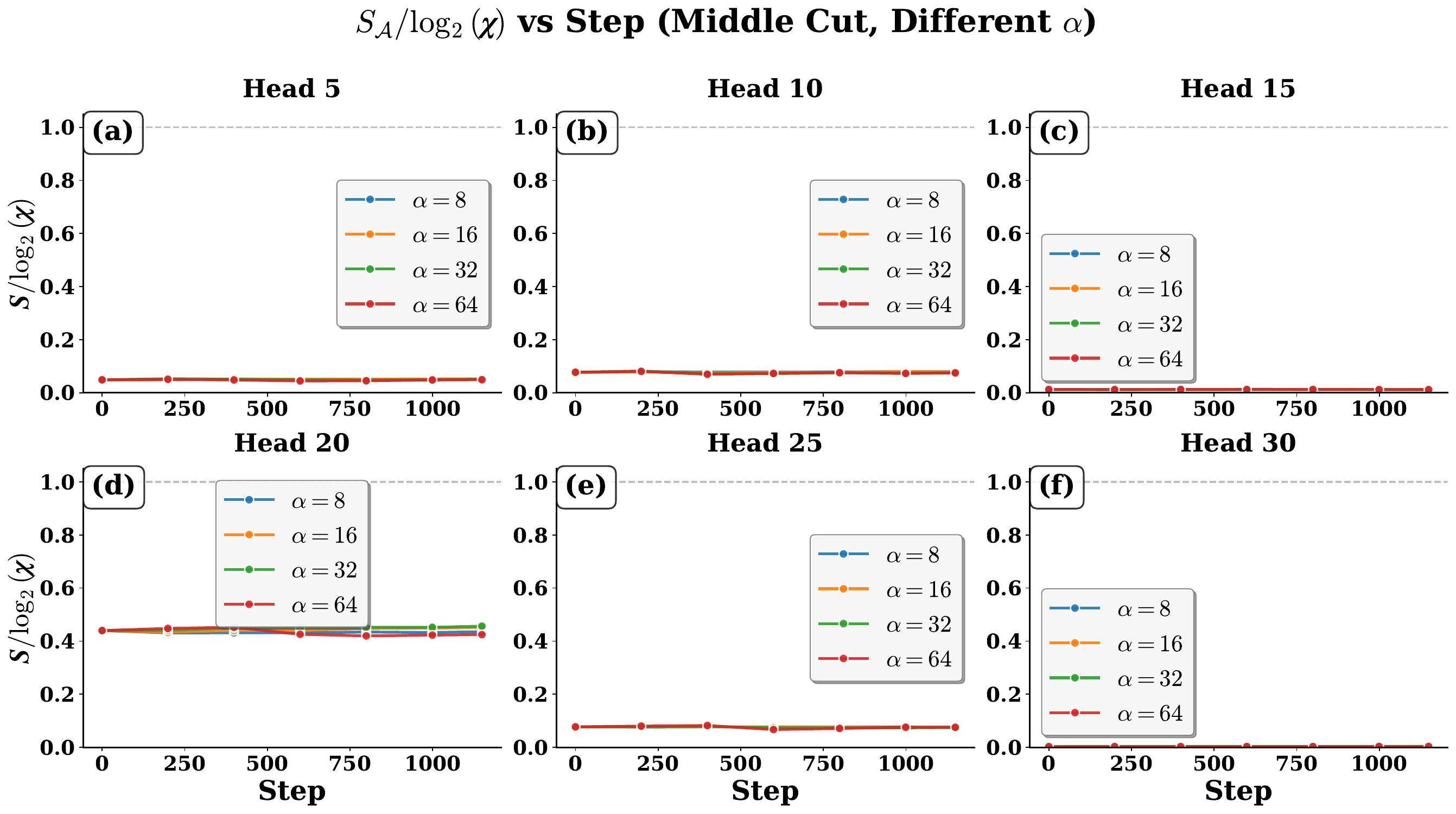}
\caption{
Entanglement entropy $S_A/\log(\chi)$ across training steps for several
attention heads under scaling coefficients $\alpha\in\{8,16,32,64\}$. Across all heads, the entanglement curves remain approximately invariant with respect to both $\alpha$ and the training step. This is similar to the behavior observed in LoRA (See FIG.~\ref{fig:attention_entropy_normalized_by_alpha_all_heads}).
}
\label{fig:MPS_attention_S_alpha_8_16_32_64}
\end{figure}

\begin{figure*}
    \centering
    \includegraphics[width=\textwidth]{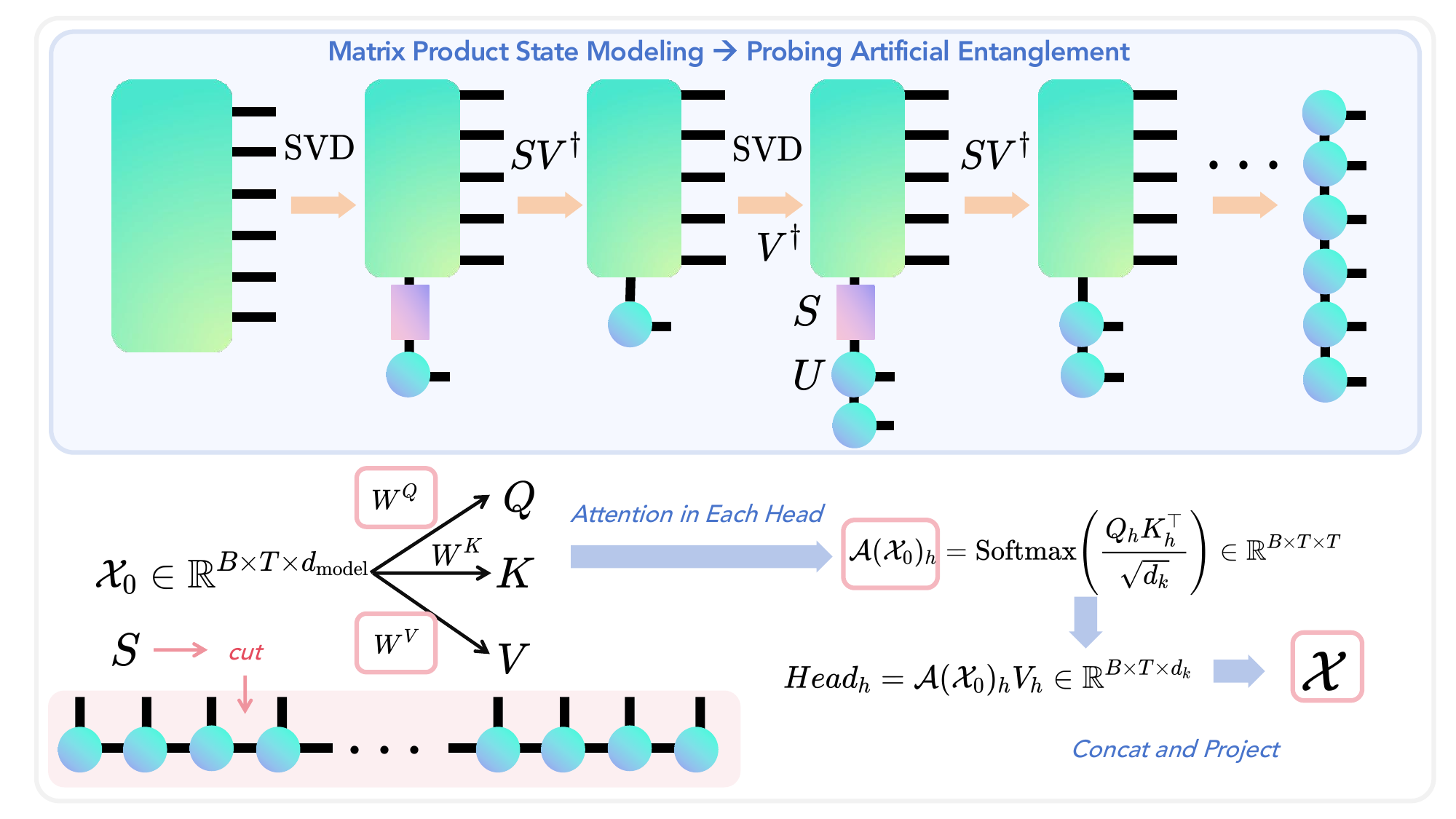}   
    \vspace{-10pt}
    \caption{Overall framework for artificial entanglement analysis. (i) Reshaping matrices including the updates of weight projection matrices such as $\Delta W_Q$ and $\Delta W_V$ from FFT or LoRA, and attention-related matrices into higher-order tensors by factorizing input and output dimensions into prime components; (ii) Performing a sequence of SVDs to decompose the tensor into a chain of Order-3 tensors  connected by virtual bonds, where virtual bond indices capture correlations between sites; (iii) Interpreting the MPS factorization as a mathematical formalism of many-body states, formally analogous to the MPS representation of the amplitude tensor of a quantum many-body state; (iv) Computing the von Neumann entanglement entropy (\(S\)) at each bond, yielding an artificial entanglement profile.}
    \label{fig:MPS_modeling}
\end{figure*}

\subsection{Artificial Entanglement Profiling in MPS Adaptation}
\label{sec:MPS_adaptation}

While the MPS modeling above is adopted for artificial entanglement profiling of LoRA and FFT (where we decompose their parameter updates into MPS representations to analyze entanglement structure), we can extend this analysis framework to a fine-tuning method that directly parameterizes weight updates using MPS structure, which is also parameter efficient and we term as \textit{MPS adaptation}. The key distinction is that MPS adaptation is a PEFT method (a way to parameterize trainable updates), whereas MPS modeling is an analysis tool (a way to decompose and analyze parameter structures). By applying the same entanglement profiling framework to MPS adaptation, we can investigate whether it exhibits similar entanglement signatures (such as the entanglement valley observed in LoRA) and compare the entanglement structures across different PEFT paradigms. For a detailed explanation of MPS adaptation, its relationship to MPS modeling, and why it can be parameter-efficient, see \textit{Appendix~\ref{app:MPS_adaptation_method}}. To investigate this, we implement a simple version of MPS adaptation: We formalize \(A \in \mathbb{R}^{r \times d_{\mathrm{in}}}\) into two Order-3 tensors whose contraction reconstructs the effective matrix \(A\), where we denote it as \(A_{\text{MPS}}\). Specifically, we factorize the input dimension as \(d_{\mathrm{in}} = d_1 d_2\) and parameterize
\begin{equation}
    A_{\text{MPS}} = \sum_{\alpha=1}^{\chi} \mathcal{A}^{[1]}_{r, \alpha, d_1} \mathcal{A}^{[2]}_{\alpha, 1, d_2},
\end{equation}
where each \(\mathcal{A}^{[i]}\) is an Order-3 tensor and the bond dimension \(\chi\) controls the entanglement capacity. Analogous to LoRA's $\Delta W = \frac{\alpha}{r} BA$ where $B \in \mathbb{R}^{d_{\mathrm{out}} \times r}$ and $A \in \mathbb{R}^{r \times d_{\mathrm{in}}}$, MPS adaptation replaces the full matrix $A$ with an MPS factorization $A_{\text{MPS}}$. The weight update is then:
\begin{equation}
    \Delta W = \frac{\alpha}{r} B A_{\text{MPS}} = \frac{\alpha}{r} B \left(\sum_{\alpha=1}^{\chi} \mathcal{A}^{[1]}_{r, \alpha, d_1} \mathcal{A}^{[2]}_{\alpha, 1, d_2}\right),
\end{equation}
where the contraction over the bond index $\alpha$ reconstructs $A_{\text{MPS}} \in \mathbb{R}^{r \times d_{\mathrm{in}}}$ from the tensor network factorization, and $B$ remains the same matrix as in LoRA. This MPS parameterization of $A$ can reduce the number of parameters compared to storing the full $r \times d_{\mathrm{in}}$ matrix under specific settings, see a concrete example in \textit{Appendix~\ref{app:MPS_adaptation_method}}.

\begin{figure}[t!]
    \centering
\includegraphics[width=\linewidth]{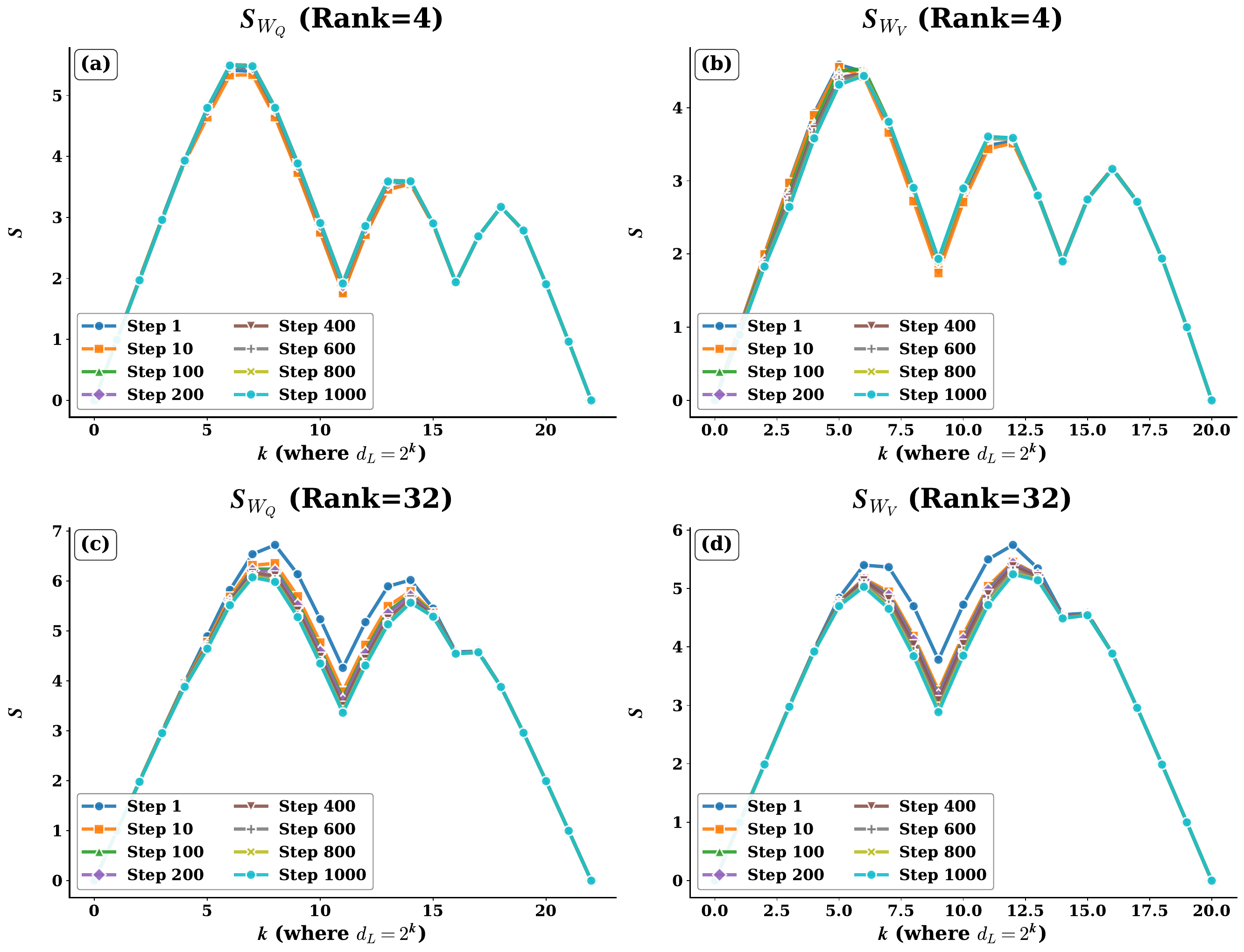}
\caption{
Artificial entanglement profiling of the MPS adaptation of $\Delta W$ for ranks 
$r \in \{4, 32\}$ with $\alpha = 16$ across training steps. 
Both rank settings exhibit a characteristic ``entanglement valley'' structure, 
while the rank-$4$ case displays a more pronounced \emph{bi-modal} (dual-valley) 
profile. In contrast, the rank-$32$ curves form a smoother single-valley landscape.
}
\label{fig:tensor_lora_a_rank_comparison}
\end{figure}

\begin{figure}[t!]
    \centering
\includegraphics[width=\linewidth]{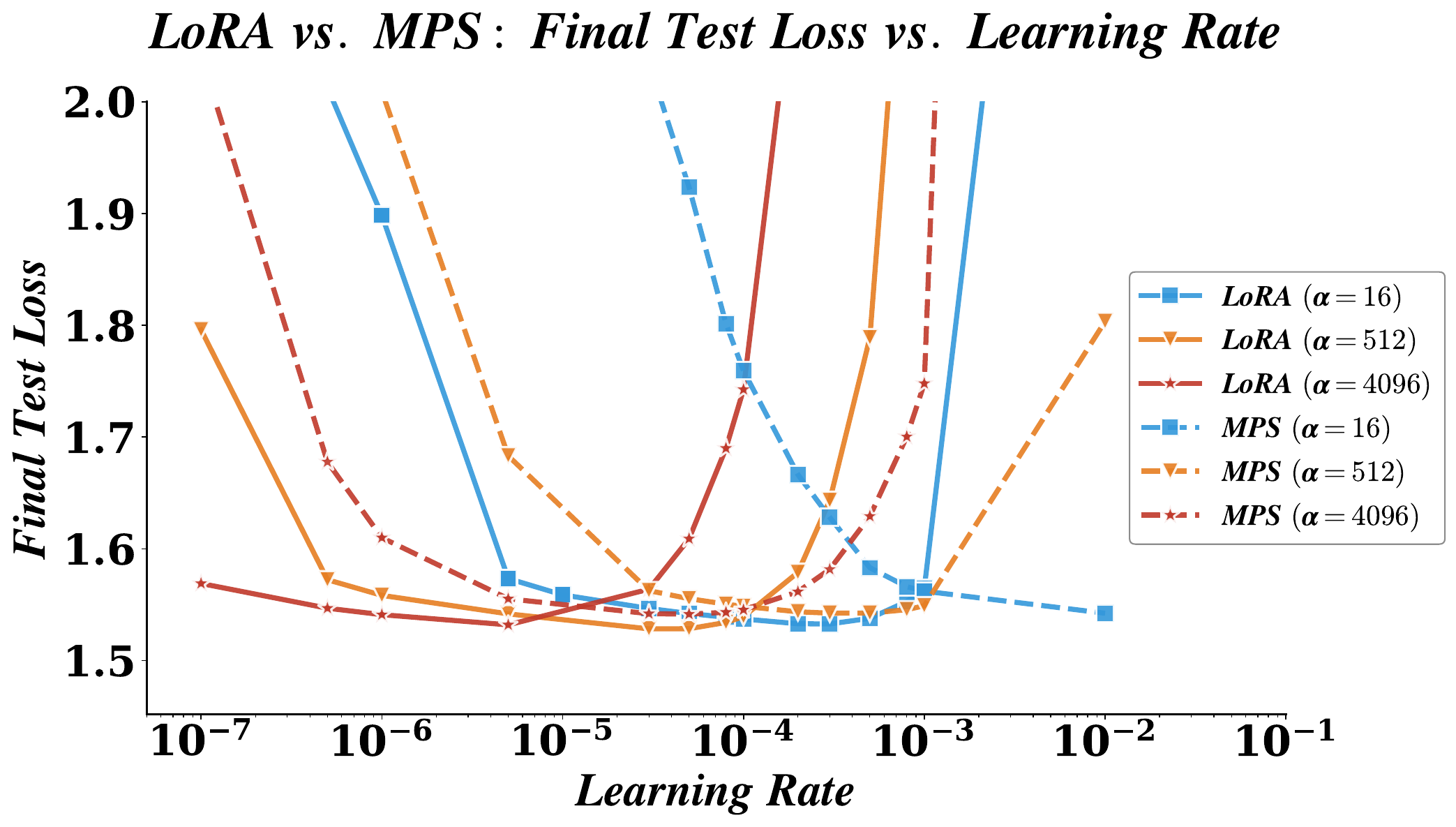}
\caption{Final test loss of LoRA and MPS adaptation across multiple learning rates and scaling coefficients $\alpha$. Compared to LoRA, the MPS adaptation exhibits an optimal learning-rate region that is consistently shifted toward larger learning rates.}
\label{fig:lora_vs_tensor_lora_a_comparison}
\end{figure}

\textbf{Artificial Entanglement Profiling in MPS adaptation.} We conduct below experiments:

\underline{\textbf{(i)}} We conduct the artificial entanglement profiling of \(\Delta W_Q\) and \(\Delta W_V\) under different \(r\) (or \(\chi\)),  see FIG.~\ref{fig:tensor_lora_a_rank_comparison}. To simplify the problem, we set \(r = \chi\), and therefore we do not set \(r\) or \(\chi\) exceeds \(\min\{d_{\mathrm{out}}, d_1, d_2\}\). Here we respectively set \(r = \{4, 32\}\), and learning rate to be \(2\times10^2\). Other settings follow \textbf{Default Experimental Setups} in Sec. \ref{sec:Entanglement Structure in FFT and LoRA}. 

\underline{\textbf{(ii)}} We conduct the artificial entanglement profiling of \(S_A\), see FIG.~\ref{fig:MPS_attention_S_alpha_8_16_32_64}.

\underline{\textbf{(iii)}} We vary \(\alpha\) and learning rates to see the shift of the optimal learning rates, see FIG.~\ref{fig:lora_vs_tensor_lora_a_comparison}.

Combining the results, in MPS adaptation we can observe rich entanglement profiling results in the updates of projection matrices ($\Delta W_Q$ and $\Delta W_V$) and the no-hair property. Besides, similar to LoRA, when increasing \(\alpha\), the optimal learning rates are shifted to be smaller as the increase of \(\alpha\), approaching the solutions of LoRA and FFT. Besides, though possessing different entanglement profiling, MPS adaptation does not induce an obvious test loss decrease in the optimal setting, approximating the test performance of LoRA.

\section{Methods}
\label{sec:method}

This section details the methodology for artificial entanglement analysis, especially we establish an analogy between the MPS factorization of matrices and the MPS representation of quantum many-body states: we here detail the interpretation of the MPS factorization of matrices as formally analogous to the amplitude tensor of a quantum many-body state $|\Psi\rangle$. We emphasize that the underlying systems are entirely classical: this is purely a mathematical analogy that provides a principled framework for quantifying entanglement-like correlations via Schmidt decomposition. See FIG.~\ref{fig:MPS_modeling} and preliminaries of tensor analysis in \textit{Appendix~\ref{sec:Preliminaries and Related Works}}. 

For a quantum state \(|\Psi\rangle\) partitioned into two subsystems A and B via Schmidt decomposition:
\begin{equation}
    |\Psi\rangle = \sum_{i} \lambda_{i}\,|i_A\rangle\,|i_B\rangle,
\end{equation}
the entanglement entropy \(S\) is calculated as:
\begin{equation}
    S_{A} = S_{B} = -\sum_{i} \lambda_{i}^2 \log \lambda_{i}^2
\end{equation}
For an MPS, the two subsystems are obtained by specifying a specific \textit{cut} in the MPS chain, where we investigate the entanglement between the left and right subsystems after the bi-partition. We apply this framework to relevant matrices, including the updates of projection matrices \(\Delta W_Q\) and \(\Delta W_V\), attention matrices \(A\), and attention output operators \(\mathcal{O}=XX^T\), by formalizing each matrix as an MPS. 

As an example, we focus on one of the projection matrices as \(\Delta W\). While $\Delta W \in \mathbb{C}^{d_{\mathrm{out}}\times d_{\mathrm{in}}}$ is originally an Order-2 tensor, we reshape it into a higher-order tensor by factorizing its input and output dimensions into their prime components. Specifically, we iteratively divide each dimension by its smallest prime factor until the value reduces to one:
\begin{equation}
\begin{aligned}
        d_{\mathrm{out}} = \prod_{k=1}^{n} f_k, \qquad
f_1 \le f_2 \le \cdots \le f_n ,\\
d_{\mathrm{in}} = \prod_{\ell=1}^{m} g_\ell, \qquad
g_1 \le g_2 \le \cdots \le g_m ,
\end{aligned}
\end{equation}
where $f_k$ and $g_\ell$ denote the prime factors sorted in non-decreasing order, and $n$ and $m$ denote the number of prime factors in $d_{\mathrm{out}}$ and $d_{\mathrm{in}}$, respectively. This decomposition yields the finest-grained tensorization of the input and output spaces and provides a natural way to reshape $\Delta W$ into an Order-$(n+m)$ tensor amenable to tensor-network factorizations, including the MPS.

For LoRA, we denote \(\mathcal{A}_i\) and \(\mathcal{B}_i\) to be each local Order-3 tensor in the MPS on the same rank space~$\mathcal{H}_r$, therefore \(\Delta W_{\text{LoRA}}\) can be viewed as the probability amplitude tensor of a quantum wave function represented below:
\begin{equation}
\begin{aligned}
    &|\Psi \rangle_{\text{LoRA}} = \Delta W_{\text{LoRA}} |o_1\cdots o_m\rangle_\mathrm{out}\otimes|i_1\cdots i_n\rangle_\mathrm{in}\\
    &= \sum_{\alpha_1, ..., \alpha_n}^\chi \mathcal{B}^{[1]  d_{o_1}}_{\alpha_0=1, \alpha_1} \mathcal{B}^{[2] d_{o_2}}_{\alpha_1, \alpha_2} ..., \mathcal{A}^{[n+1]d_{i_1}}_{\alpha_{n, n+1}} \mathcal{A}^{[n+m]d_{i_m}}_{\alpha_{n+m-1},\,\alpha_{n+m}=1}\, \cdot\\
    &|o_1\cdots o_m\rangle_\mathrm{out}\otimes|i_1\cdots i_n\rangle_\mathrm{in}.
\end{aligned}
\end{equation}
Although the above formulation is written in terms of the decomposed factors \(\mathcal{A}_i\) and \(\mathcal{B}_i\), in practice we directly decompose the equivalent \(\Delta W_{\text{LoRA}}\) itself into the MPS representation.
For FFT, the expression is the similar, we express 
\begin{equation}
\begin{aligned}
    &|\Psi \rangle_{\text{FFT}} = \Delta W_{\text{FFT}} |w_1\cdots w_{m+n}\rangle\\
    &= \sum_{\alpha_1, \cdots, \alpha_n}^\chi \mathcal{W}^{[1]  d_{o_1}}_{\alpha_0=1, \alpha_1} \cdots \mathcal{W}^{[n+1]d_{i_1}}_{\alpha_{n, n+1}} \cdots\mathcal{W}^{[n+m]d_{i_m}}_{\alpha_{n+m-1},\,\alpha_{n+m}=1}\, \cdot\\
    &|w_1\cdots w_{m+n}\rangle.
\end{aligned}
\end{equation}
See \textit{Appendix \ref{sec:Preliminaries and Related Works}} for more details about the formulation based on the factorization. 

Further, we calculate the von Neumann entanglement entropy by performing an SVD at each bond during the MPS decomposition. At bond $k$ (the $k$-th bond in the MPS chain), we obtain the singular values $\{\lambda_{\alpha_k}\}$, where $\alpha_k$ indexes the singular values (Schmidt indices) at bond $k$. The SVD is written in terms of each bipartite quantum state \(|\Psi\rangle \in \mathcal{H}_A \otimes \mathcal{H}_B\) component \(\Psi_{ij}\) below: 
\begin{equation}
    \Psi_{ij}
= \sum_{\alpha,\beta=1}^{\chi} U_{i\alpha}\,\Lambda_{\alpha\beta}\,(V^{\dagger})_{\beta j}
= \sum_{\alpha=1}^{\chi} U_{i\alpha}\,\lambda_{\alpha}\,(V^{\dagger})_{\alpha j},
\end{equation}
with at most \(\chi\) non-zero singular values. The bi-partition at bond $k$ corresponds to 
$\mathcal{H}_A = \mathcal{H}_{1:k}$ and $\mathcal{H}_B = \mathcal{H}_{k+1:(n+m)}$. Since the reduced density matrices $\rho_{1:k}$ and $\rho_{k+1:(n+m)}$ share the same non-zero spectrum, the bipartite entanglement entropy at bond $k$ is

\begin{equation}
    S_k = S(\rho_{1:k})
=
S(\rho_{k+1:(n+m)})
=
- \sum_{\alpha_k = 1}^{\chi}
\lambda_{\alpha_k}^2\,
\log\!\left(\lambda_{\alpha_k}^2\right),
\end{equation}
where $S_k$ quantifies the entanglement entropy at bond $k$. We can calculate $S_k$ for each bond $k$, and therefore we can obtain an $S_k$ curve (artificial entanglement profile) with respect to the bond position. 

\section{Discussion}
\label{sec:discussion}

Our results reveal that LoRA captures a richer structure than previously recognized. By quantifying \textit{artificial entanglement} through MPS modeling, we show that LoRA and FFT induce a fitted \textbf{volume-law} internal entanglement profile in the projection matrices \(\Delta W_Q\) and \(\Delta W_V\) with a distinctive ``entanglement valley". These findings highlight that the fitted volume-law  reflects high intrinsic correlation, thereby causing LoRA to develop internal artificial entanglement signatures that fundamentally differ from those produced by FFT.

Yet, in contrast to the volume-law scaling in projection matrices, attention matrices themselves exhibit an approximate \textbf{area-law} scaling. Despite these microscopic differences in the \textit{internal} structure (the volume-law entanglement profiles in projection matrices), the \textit{external} observable attention outputs remain remarkably invariant. Drawing an analogy to the No-Hair Theorem in black hole physics, we propose that the attention mechanism acts as an effective coarse-graining operator: it suppresses redundant correlations while preserving only low-correlation and task-relevant structure. This perspective offers an explanation for why LoRA performs comparably to FFT despite using far fewer parameters. 

The entanglement signatures we observe provide valuable diagnostic tools for PEFT methods. Specifically, the depth and evolution of the entanglement valley in projection matrices serve as indicators of hyperparameter choices: smaller scaling parameters $\alpha$ lead to deeper valleys that intensify during training, while larger $\alpha$ values mitigate the valley depth and bring the solution closer to FFT. The entanglement profile's sensitivity to rank $r$ and training dynamics further enables performance assessment---methods that preserve task-relevant structure should exhibit entanglement structures that align with task requirements. Moreover, the contrast between volume-law scaling in projection matrices and area-law scaling in attention matrices reveals that low-rank parameterizations can be effective even when internal structures exhibit high correlation, possibly related to the downstream attention mechanism providing effective coarse-graining. This suggests that entanglement analysis can help identify which components require careful rank selection versus those where low-rank approximations are naturally robust. These diagnostic capabilities suggest that entanglement analysis could guide hyperparameter tuning and help identify optimal adaptation strategies in PEFT.

Finally, our theoretical analysis based on random matrix theory and MPS Adaptation extension indicate that the no-hair behavior is not an artifact of LoRA but a more universal property. This opens several directions for future work, including probing how entanglement evolves during pretraining, whether entanglement-based regularization can guide more efficient adaptation, and whether similar coarse-graining phenomena arise in other architectures or multi-task settings. Together, our findings suggest that quantum-inspired analysis may offer a new foundation for understanding the mechanistic interpretability of large language models.

\section*{Acknowledgment}

We thank Jesse Thaler for helpful discussions. MC, ZW, and JL are supported in part by the University of Pittsburgh, School of Computing and Information, Department of Computer Science, Pitt Cyber, Pitt Momentum fund, PQI Community Collaboration Awards, John C. Mascaro Faculty Scholar in Sustainability, Thinking Machines Lab and Cisco Research. This research used resources of the Oak Ridge Leadership Computing Facility, which is a DOE Office of Science User Facility supported under Contract DE-AC05-00OR22725.

\bibliography{example_paper}
\onecolumngrid
\newpage
\let\addcontentsline\oldaddcontentsline
\tableofcontents

\section{Notations}
\label{sec:notations}

Table~\ref{tab:notations} provides a list of notations.

\begin{table*}[hpt!]
\centering
\caption{Notations}
\begin{tabular}{l|p{12cm}}
\hline
\textbf{Symbol} & \textbf{Description} \\
\hline
\(W_0\) & Pre-trained weight projection matrix. \\
\(\Delta W\) & Weight update matrix. \\
\(\Delta W_{\mathrm{LoRA}}\) & LoRA weight update: \(\Delta W_{\mathrm{LoRA}} = \frac{\alpha}{r} BA\). \\
\(\Delta W_{\mathrm{Full}}\) & Full fine-tuning weight update. \\
\(W_Q, W_K, W_V\) & Query, key, value projection matrices. \\
\(W^O\) & Output projection matrix. \\
\(A, B\) & Low-rank matrices in LoRA decomposition (\(A \in \mathbb{R}^{r \times d_{\mathrm{in}}}\), \(B \in \mathbb{R}^{d_{\mathrm{out}} \times r}\)). \\
\(A\) & Attention matrix (context-dependent). \\
\(\mathcal{O} = XX^\top\) & Attention output operator. \\
\hline
\(d_{\mathrm{in}}, d_{\mathrm{out}}\) & Input and output dimensions of weight projection matrices. \\
\(d_{\mathrm{model}}\) & Token embedding dimension (model dimension). \\
\(d_k\) & Dimension of each attention head. \\
\(d_L^{(k)}, d_R^{(k)}\) & Dimensions of left and right subsystems at cut \(k\). \\
\(d_{\min}, d_{\max}\) & Minimum and maximum dimensions (\(d_{\min} \le d_{\max}\)). \\
\(d_k\) & Local dimension at site \(k\) (context-dependent). \\
\(T\) & Sequence length (number of tokens). \\
\(H\) & Number of attention heads. \\
\(B\) & Batch size. \\
\(M, N\) & Order of tensors (number of indices). \\
\(n, m\) & Number of sites in MPS factorization. \\
\(f_k, g_\ell\) & Prime factors in factorization of \(d_{\mathrm{out}}\) and \(d_{\mathrm{in}}\). \\
\hline
\(r\) & LoRA rank (\(r \ll \min(d_{\mathrm{in}}, d_{\mathrm{out}})\)). \\
\(\alpha\) & LoRA scaling hyperparameter. \\
\hline
\(\chi, \chi_k\) & Bond dimension (at bond \(k\)). \\
\(\chi_{\max}\) & Maximum bond dimension (truncation threshold). \\
\(\chi\) & Maximum bond dimension: \(\chi := \max_k \chi_k\). \\
\(\mathcal{T}^{[k]}\) & \(k\)-th tensor in MPS decomposition. \\
\(\mathcal{A}^{[k]}, \mathcal{B}^{[k]}\) & MPS tensors for LoRA matrices \(A\) and \(B\). \\
\(\alpha_k\) & Virtual bond index (Schmidt index). \\
\(i_k\) & Physical index at site \(k\). \\
\(C^{(k)}\) & Matrix at cut \(k\) in MPS construction. \\
\(U^{(k)}, S^{(k)}, V^{(k)}\) & SVD factors at cut \(k\). \\
\hline
\(S, S_k\) & Von Neumann entanglement entropy (at cut \(k\)). \\
\(S_2\) & Rényi-2 entropy. \\
\(S_\alpha\) & Rényi-\(\alpha\) entropy. \\
\(\rho_L^{(k)}\) & Reduced density matrix of left subsystem at cut \(k\). \\
\(\lambda_{\alpha_k}, s_i^{(k)}\) & Singular values (Schmidt coefficients). \\
\(p_i^{(k)}\) & Eigenvalues of reduced density matrix. \\
\hline
\(\sigma\) & Bulk spectrum spread parameter. \\
\(\mathcal{C}_{\mathrm{attn}}\) & Effective Attention Charge: \(\mathcal{C}_{\mathrm{attn}} = \frac{\sigma^2}{1+\sigma^2}\). \\
\(r_{\mathrm{stable}}\) & Stable rank: \(r_{\mathrm{stable}} = \frac{\mathrm{Tr}(\Sigma^2)}{\lambda_1^2}\). \\
\(\eta\) & Stable rank excess: \(\eta = r_{\mathrm{stable}} - 1\). \\
\(\lambda_i\) & Eigenvalues (context-dependent). \\
\(x_i = d_{\min} \cdot \lambda_i\) & Normalized eigenvalues for MP distribution. \\
\(\rho(x)\) & Marchenko-Pastur density function. \\
\(c\) & Aspect ratio: \(c = d_{\min}/d_{\max}\). \\
\hline
\(|\Psi\rangle\) & Many-body quantum state. \\
\(\ket{i_L}, \ket{i_R}\) & Basis states for left and right subsystems. \\
\(\mathcal{H}_L, \mathcal{H}_R\) & Left and right Hilbert spaces. \\
\(\mathcal{H}_A, \mathcal{H}_B\) & Subsystem Hilbert spaces. \\
\hline
\(\mathcal{X}_0\) & Input tensor: \(\mathcal{X}_0 \in \mathbb{R}^{B \times T \times d_{\mathrm{model}}}\). \\
\(\mathcal{A}(\mathcal{X}_0)\) & Attention tensor: \(\mathcal{A}(\mathcal{X}_0) \in \mathbb{R}^{B \times H \times T \times T}\). \\
\(Q_h, K_h, V_h\) & Query, key, value matrices for head \(h\). \\
\(A^\perp\) & Centered attention matrix: \(A^\perp = A - \frac{1}{T}\mathbf{1}\mathbf{1}^\top\). \\
\(X\) & Attention layer output: \(X = AV\). \\
\hline
\(\mathbb{R}, \mathbb{C}\) & Real and complex number fields. \\
\(\lVert\cdot\rVert_F\) & Frobenius norm. \\
\(\lVert\cdot\rVert_{\mathrm{op}}\) & Operator norm (largest singular value). \\
\(\mathrm{Tr}(\cdot)\) & Trace operator. \\
\(\mathrm{vec}(\cdot)\) & Vectorization operator. \\
\(o(1)\) & Asymptotically vanishing term. \\
\hline
\end{tabular}
\label{tab:notations}
\end{table*}

\section{Preliminaries and Related Works}
\label{sec:Preliminaries and Related Works}

\textbf{Tensor, Tensor Network, and Matrix Product State (MPS).}
Tensors generalize matrices to multi-index arrays and provide a unified
language for representing high-dimensional structures. An Order-$N$ tensor 
$\mathcal{A}_{i_1\dots i_N}$ has $N$ indices, where each 
$i_k \in \{0,1,\dots,d_k-1\}$ ranges over a local dimension $d_k$. Graphically as shown in FIG.~\ref{fig:general_intro_tensor}   (a), one can visualize a tensor as a node with $N$ legs, each leg corresponding to an index of the tensor. A fundamental operation between tensors is \emph{tensor contraction}, which corresponds to summing over a pair of shared indices. In FIG.~\ref{fig:general_intro_tensor}  (b), contraction is depicted by connecting the legs 
associated with the contracted dimensions. More elaborate contraction 
patterns naturally give rise to \emph{tensor networks}, which provide 
structured factorization of large tensors into networks of smaller 
building blocks. Among various tensor network architectures, the \emph{Matrix Product State}
(MPS) is the most widely used 
one-dimensional tensor network. An MPS expresses a high-order tensor as a 
chain of Order-3 tensors connected by virtual bonds. In FIG.~\ref{fig:general_intro_tensor}  (c), each site emits one physical leg (the physical index) and two virtual legs linking to its neighbors, forming a 
one-dimensional chain as depicted in the figure. Consider a general $M$-body pure quantum state on a local Hilbert space of
dimension $d$:
\begin{equation}
    |\Psi\rangle 
    = \sum_{i_1,\dots,i_M=0}^{d-1}
      c_{i_1,\dots,i_M}\,|i_1,\dots,i_M\rangle,
\end{equation}
where $c_{i_1\dots i_M}$ is an amplitude tensor of Order-$M$ whose storage
cost grows exponentially with $M$. The MPS representation approximates this
tensor by factorizing it into a sequence of Order-3 tensors:
\begin{equation}
    c_{i_1,\dots,i_M}
    = \sum_{\alpha_1=1}^{\chi_1} \cdots \sum_{\alpha_{M-1}=1}^{\chi_{M-1}}
      \mathcal{T}^{[1]\, i_1}_{1,\alpha_1}\,
      \mathcal{T}^{[2]\, i_2}_{\alpha_1,\alpha_2}
      \cdots
      \mathcal{T}^{[M]\, i_M}_{\alpha_{M-1},1},
\end{equation}
where each $\mathcal{T}^{[k]}$ is an Order-3 tensor with one physical index $i_k$ 
and two virtual indices $\alpha_{k-1},\alpha_k$. The virtual indices are 
contracted, yielding the chain-like network structure. The parameter 
$\chi_k$ denotes the bond dimension at bond $k$, which may vary along the chain. The quantity $\chi := \max_k \chi_k$ is often referred to as the (maximum) \emph{bond dimension}, which determines the expressive power
of the MPS and controls the entanglement it can capture: larger $\chi$
allows more accurate approximations of $c$, while small $\chi$
corresponds to a compressed representation. In quantum physics, the effectiveness of the MPS formalism is deeply tied 
to the entanglement structure of the underlying state. For example, ground
states of gapped one-dimensional systems satisfy an entanglement 
area law, which implies that they can be faithfully approximated by an MPS
with modest bond dimension. This makes MPS not only a powerful conceptual 
framework but also a practical computational ansatz for representing 
high-dimensional quantum states with polynomial resources.

\begin{figure*}
    \centering
\includegraphics[width=0.6\textwidth]{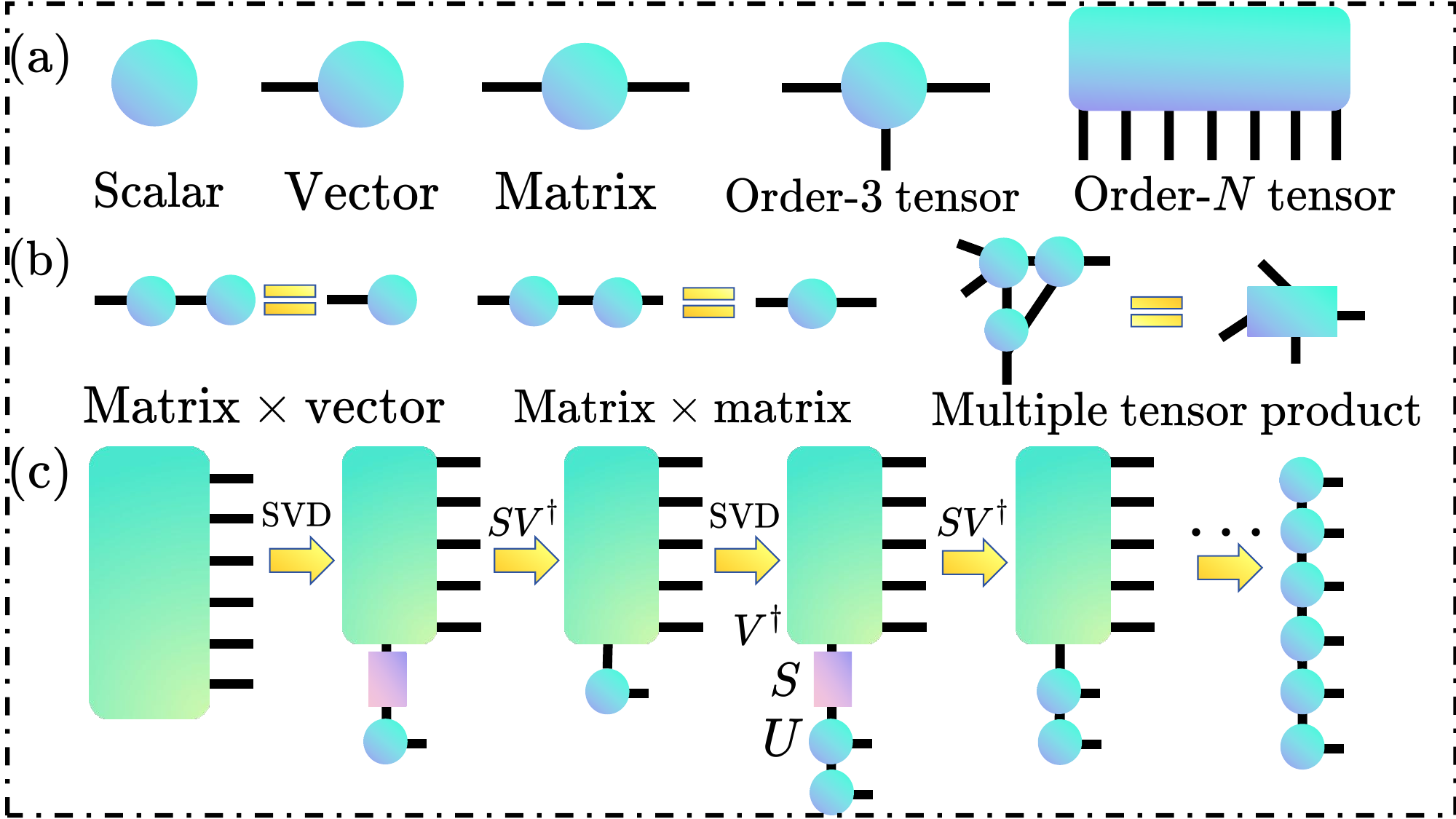}
    \vspace{-10pt}
    \caption{Overview of Tensor Diagrammatic Notation. (a) Visual representations of a scalar, vector, matrix, and an arbitrary order-$N$ tensor. (b) Tensor contractions illustrating matrix multiplication and multiple tensor products. (c) The process of modeling an order-$N$ tensor as a Matrix Product State.}
\label{fig:general_intro_tensor}
\end{figure*}

\textbf{Forming an MPS from a high-order tensor.}
As depicted in FIG.~\ref{fig:general_intro_tensor} (c), an MPS representation can be systematically constructed from a high-order tensor
by a sequence of reshaping and singular value decompositions (SVDs). Consider an Order-$M$ tensor $c_{i_1 i_2 \dots i_M}$, which may be viewed as the probability amplitude tensor of an $M$-body quantum state. The construction proceeds iteratively from left to right. In the first step, we group the first index
$i_1$ into a ``left'' subsystem and the remaining indices
$(i_2,\dots,i_M)$ into a ``right'' subsystem, reshaping the tensor into a
matrix
\begin{equation}
    C^{(1)}_{\, i_1,\; (i_2 \dots i_M)}
    \;\in\; \mathbb{C}^{d \times d^{M-1}}.
\end{equation}
Performing a singular value decomposition,
\begin{equation}
    C^{(1)} = U^{(1)} S^{(1)} V^{(1)\dagger},
\end{equation}
we absorb $U^{(1)}$ into the first MPS tensor
$\mathcal{T}^{[1]}$, while the product $S^{(1)} V^{(1)\dagger}$ is passed to the
next step. Explicitly,
\begin{equation}
    \mathcal{T}^{[1]\, i_1}_{\alpha_0=1,\alpha_1}
    := U^{(1)}_{i_1,\alpha_1},
\end{equation}
where $\alpha_1$ labels the singular vectors and defines the first virtual
bond.

In the second step, the remaining object
$S^{(1)} V^{(1)\dagger}$ is reshaped into a matrix by grouping
$(\alpha_1,i_2)$ as the left index and $(i_3,\dots,i_M)$ as the right index.
Applying another SVD yields
\begin{equation}
    C^{(2)}_{\, (\alpha_1 i_2),\; (i_3\dots i_M)}
    = U^{(2)} S^{(2)} V^{(2)\dagger},
\end{equation}
from which we define the second MPS tensor
\begin{equation}
    \mathcal{T}^{[2]\, i_2}_{\alpha_1 \alpha_2}
    := U^{(2)}_{(\alpha_1 i_2),\alpha_2}.
\end{equation}
This procedure is repeated sequentially, each time isolating one physical
index $i_k$ and generating a new virtual bond $\alpha_k$, until the final
site is reached.

After $M-1$ such decompositions, the original tensor is factorized as
\begin{equation}
    c_{i_1,\dots,i_M}
    = \sum_{\alpha_1,\dots,\alpha_{M-1}}
      \mathcal{T}^{[1]\, i_1}_{1,\alpha_1}
      \mathcal{T}^{[2]\, i_2}_{\alpha_1,\alpha_2}
      \cdots
      \mathcal{T}^{[M]\, i_M}_{\alpha_{M-1},1},
\end{equation}
which is precisely the MPS form. Each tensor
$\mathcal{T}^{[k]}$ carries one physical index $i_k$ and two virtual indices
$\alpha_{k-1},\alpha_k$, forming a one-dimensional chain as illustrated in
the figure. In practice, the singular values in each $S^{(k)}$ may be truncated,
retaining only the largest $\chi_k$ values. This truncation controls the bond
dimension at each bond and yields an efficient compressed representation. Crucially, the
maximum bond dimension required to accurately represent the state is
directly related to its entanglement across bi-partitions. 
 
\textbf{Self Attention in Large Language Models.} 
Generally, for an input $\mathcal{X}_0 \in \mathbb{R}^{B \times T \times d_{\mathrm{model}}}$, the model employs Multi-Head Attention to capture diverse contextual dependencies. The mechanism utilizes projection matrices $W^Q, W^K, W^V \in \mathbb{R}^{d_{\mathrm{model}} \times (H d_k)}$ to generate query, key, and value representations, where $d_k$ is the dimension of each head. These projected features are split into $H$ parallel heads. Let $Q_h, K_h, V_h$ denote the segments corresponding to head $h$. In modern LLMs such as LLaMA, the query and key vectors are typically augmented with \emph{Rotary Position Embedding (RoPE)}~\cite{su2024roformer}, which encodes positional information by applying rotational transformations to the query and key vectors. This allows the model to capture relative positional relationships effectively. After applying RoPE, the attention weights are computed via scaled dot-product:
\begin{equation}
    \mathcal{A}(\mathcal{X}_0)_h = \text{Softmax}\left( \frac{Q_h K_h^\top}{\sqrt{d_k}} \right) \in \mathbb{R}^{B \times T \times T}
\end{equation}
The attention output for each head is then derived as $Head_h = \mathcal{A}(\mathcal{X}_0)_h V_h \in \mathbb{R}^{B \times T \times d_k}$. Finally, all heads are concatenated and projected by an output matrix $W^O \in \mathbb{R}^{(H d_k) \times d_{\mathrm{model}}}$ to form the final layer output.

\textbf{Related Studies.} Some empirical evidence provided by Thinking Machines Lab~\cite{schulman2025lora} offers intuitive visualizations of the performance in LoRA and FFT. We echo some findings in our study and advance these observations. An important feature of ours can be concluded as identifying a  ``no-hair'' behavior that justifies the effectiveness of LoRA. \cite{shuttleworth2024lora} discusses different spectral properties of the resulting final weights obtained by LoRA and FFT, compared with pretrained weights. We focus on \underline{(i)} the direct fine-tuning solutions \(\Delta W\), the attention matrices and outputs; \underline{(ii)} the entanglement structure; \underline{(iii)} test performance w.r.t hyper-parameters, while they focus on generalization and forgetting. \cite{saadamind} shows a phenomenon termed as \textit{rank collapse in width}, both in attention matrices and attention outputs. We theoretically link our findings to this perspective (See Section~\ref{sec:random_matrix_theory}). Notably, TT-LoRA~\cite{anjum2024tensor} utilizes an equivalent Tensor Train decomposition. While it focuses on storage efficiency, we utilize the tensor network formalism to probe entanglement signatures. We do not focus on other variants of LoRA but the original LoRA itself. 

\section{Mathematical Foundations: Quantum Entanglement and Entropy Measures}
\label{app:quantum_entanglement}

This section establishes the mathematical foundations for computing and analyzing entanglement entropy in the context of MPS-decomposed weight projection matrices. We first explain how to compute entanglement entropy at each cut in the MPS construction, connecting the SVD procedure to the quantum information formalism of bipartite pure states and reduced density matrices. We then discuss important subtleties in understanding entanglement: why a spectral gap alone does not guarantee small entanglement, and how stable rank provides a quantitative measure of mass concentration that better characterizes the relationship between spectral structure and entanglement entropy. Finally, we review the definitions of von Neumann and Rényi entropies and their relationships for completeness.

\subsection{Computing Entanglement Entropy at Each MPS Cut}

Recall from the MPS construction procedure (\textit{Appendix~\ref{sec:Preliminaries and Related Works}}) that we systematically decompose a high-order tensor into an MPS by performing a sequence of SVDs. At each step $k$, we reshape the tensor to form a matrix $C^{(k)}$ that partitions the system into a ``left'' subsystem (indices $i_1,\dots,i_k$ or virtual bonds $\alpha_1,\dots,\alpha_{k-1}$ combined with $i_k$) and a ``right'' subsystem (remaining indices $i_{k+1},\dots,i_M$). This bi-partition defines a \emph{cut} in the MPS, and the SVD at this cut,
\begin{equation}
C^{(k)} = U^{(k)} S^{(k)} V^{(k)\dagger},
\end{equation}
directly reveals the entanglement structure across this bi-partition. In this subsection, we explain how to compute the entanglement entropy at each such cut.

\textbf{Bipartite Pure States from MPS Cuts.}
At cut $k$, the matrix $C^{(k)}$ has dimensions $d_L^{(k)} \times d_R^{(k)}$, where $d_L^{(k)}$ is the dimension of the left subsystem and $d_R^{(k)}$ is the dimension of the right subsystem. We view this matrix as representing a bipartite pure state between the left and right subsystems. To make this connection explicit, we fix orthonormal bases $\{\ket{i_L}\}_{i_L=1}^{d_L^{(k)}}$ for the left Hilbert space $\mathcal{H}_L \cong \mathbb{C}^{d_L^{(k)}}$ and $\{\ket{i_R}\}_{i_R=1}^{d_R^{(k)}}$ for the right Hilbert space $\mathcal{H}_R \cong \mathbb{C}^{d_R^{(k)}}$. The vectorization of $C^{(k)}$ embeds it into the tensor product space:
\begin{equation}
\mathrm{vec}(C^{(k)})
:= \sum_{i_L=1}^{d_L^{(k)}} \sum_{i_R=1}^{d_R^{(k)}} C^{(k)}_{i_L,i_R}\, \ket{i_L}\otimes\ket{i_R}.
\end{equation}
Normalizing yields the pure state
\begin{equation}
\ket{\psi^{(k)}}
:= \frac{\mathrm{vec}(C^{(k)})}{\|C^{(k)}\|_F}
\in \mathcal{H}_L \otimes \mathcal{H}_R,
\qquad
\|C^{(k)}\|_F^2 = \sum_{i_L,i_R} |C^{(k)}_{i_L,i_R}|^2.
\end{equation}

\textbf{Reduced Density Matrix and Entanglement Spectrum.}
The reduced density matrix of the left subsystem is obtained by tracing out the right subsystem:
\begin{equation}
\rho_L^{(k)} := \mathrm{Tr}_R\!\left(\ket{\psi^{(k)}}\bra{\psi^{(k)}}\right).
\end{equation}
A direct calculation shows that
\begin{equation}
\rho_L^{(k)} = \frac{C^{(k)} (C^{(k)})^\dagger}{\mathrm{Tr}(C^{(k)} (C^{(k)})^\dagger)} = \frac{C^{(k)} (C^{(k)})^\dagger}{\|C^{(k)}\|_F^2}.
\end{equation}
The entanglement spectrum is obtained directly from the SVD of $C^{(k)}$. Let $C^{(k)} = U^{(k)} S^{(k)} V^{(k)\dagger}$ with singular values $s_1^{(k)} \ge s_2^{(k)} \ge \cdots \ge s_r^{(k)} > 0$, where $r$ is the rank of $C^{(k)}$. Since
\begin{equation}
C^{(k)} (C^{(k)})^\dagger = U^{(k)} \,\mathrm{diag}((s_1^{(k)})^2,\dots,(s_r^{(k)})^2,0,\dots,0)\, (U^{(k)})^\dagger,
\end{equation}
the eigenvalues of the reduced density matrix $\rho_L^{(k)}$ are
\begin{equation}
p_i^{(k)} = \frac{(s_i^{(k)})^2}{\|C^{(k)}\|_F^2} = \frac{(s_i^{(k)})^2}{\sum_{j=1}^r (s_j^{(k)})^2},
\quad i=1,\dots,r,
\qquad
p_{r+1}^{(k)}=\cdots=p_{d_L^{(k)}}^{(k)}=0.
\end{equation}
These eigenvalues $\{p_i^{(k)}\}$ form the \emph{entanglement spectrum} at cut $k$, and they directly determine the entanglement entropy.

\textbf{Von Neumann Entanglement Entropy.}
Given the eigenvalues $\{p_i^{(k)}\}$ of the reduced density matrix $\rho_L^{(k)}$, the von Neumann entanglement entropy at cut $k$ is defined as
\begin{equation}
S^{(k)} := -\sum_{i=1}^{d_L^{(k)}} p_i^{(k)} \log p_i^{(k)} = -\sum_{i=1}^r p_i^{(k)} \log p_i^{(k)},
\end{equation}
where we adopt the convention $0 \log 0 = 0$. This entropy quantifies the amount of entanglement across the bi-partition at cut $k$: $S^{(k)} = 0$ indicates a product state (no entanglement), while larger values indicate stronger entanglement. The maximum possible entropy is $S_{\max}^{(k)} = \log(\min(d_L^{(k)}, d_R^{(k)}))$, which occurs when all non-zero eigenvalues are equal (maximally entangled state). In our analysis of weight projection matrices $\Delta W$, we reshape the flattened matrix into a tensor and construct an MPS representation. By computing $S^{(k)}$ at each cut $k$, we obtain an \emph{entanglement profile} that reveals how entanglement varies across different bi-partitions of the weight projection matrix. This profile is central to our characterization of the entanglement structure in LoRA versus full fine-tuning.

For definitions and discussion of Rényi-$\alpha$ entropies (including Rényi-2) and their relationship to von Neumann entropy, see \textit{Appendix~\ref{app:von_neumann_renyi}}.

\subsection{Discussion: Spectral Gap, Stable Rank, and Entanglement Entropy}
\label{subsec:stable_rank}
A ``spectral gap''~\cite{saadamind} indicates that one eigenvalue is separated from the rest,
but entanglement entropy is small only when the largest eigenvalue is
close to 1, \ie, when the total tail mass $\sum_{i\ge 2} p_i$ is tiny.
A large gap does not prevent $\Theta(\log T)$ entanglement:
this occurs whenever a constant fraction of the probability mass is spread over
$\Theta(T)$ small eigenvalues. To quantify this mass concentration more precisely, we review the definition of the \emph{stable rank}~\cite{saadamind} in the following.

For a PSD matrix $\Sigma\succeq 0$ with eigenvalues $\lambda_1\ge\cdots\ge\lambda_T\ge 0$,
the stable rank is defined as
\begin{equation}
\mathrm{sr}(\Sigma) := \frac{\mathrm{Tr}(\Sigma^2)}{\lambda_1^2}
= 1 + \sum_{i\ge 2}\Big(\frac{\lambda_i}{\lambda_1}\Big)^2.
\end{equation}
The stable rank provides a quantitative measure of mass concentration: $\mathrm{sr}(\Sigma)\approx 1$ means the spectrum is sharply concentrated in the top eigenvalue,
which strongly suggests small entanglement for $\rho=\Sigma/\mathrm{Tr}\Sigma$.
In contrast, a large stable rank indicates that significant probability mass is distributed across many eigenvalues, which can lead to substantial entanglement even in the presence of a spectral gap. We formalize the connection between stable rank and entanglement entropy in Lemma~\ref{lem:sr_to_entropy}.

\subsection{Von Neumann and Rényi Entropies: Definitions and Relationships}
\label{app:von_neumann_renyi}

Throughout this work, we primarily use the von Neumann entropy to quantify artificial entanglement. However, for completeness, we also consider the Rényi-$\alpha$ family of entropies, which provides complementary perspectives on entanglement structure. This subsection reviews the definitions and discusses the relationships between these entropy measures.

\textbf{Definitions.}
Given a reduced density matrix $\rho$ with eigenvalues $\{p_i\}_{i=1}^T$ satisfying $\sum_i p_i = 1$, the von Neumann entropy is defined as
\begin{equation}
S(\rho) := -\sum_{i=1}^T p_i \log p_i,
\end{equation}
while the Rényi-$\alpha$ entropy (for $\alpha > 0$, $\alpha \neq 1$) is defined as
\begin{equation}
S_\alpha(\rho) := \frac{1}{1-\alpha}\log\!\left(\sum_{i=1}^T p_i^\alpha\right).
\end{equation}
In particular, the Rényi-2 entropy takes the form
\begin{equation}
S_2(\rho) = -\log\sum_{i=1}^T p_i^2 = -\log\,\mathrm{Tr}(\rho^2).
\end{equation}

\textbf{Relationships and Properties.} We discuss Rényi-2 entropy in addition to von Neumann entropy because it depends only on the purity $\mathrm{Tr}(\rho^2)$ of the reduced state, making it particularly convenient for theoretical analysis and providing lower bounds on von Neumann entropy. Several of our main theorems provide explicit bounds or scaling for Rényi-2 entropy (see the discussion below). However, the von Neumann entropy captures more fine-grained information about the entanglement spectrum, as it weights each eigenvalue by its logarithm, making it more sensitive to the distribution of spectral weight across modes. The von Neumann entropy corresponds to the limit $\alpha \to 1$ of the Rényi-$\alpha$ entropy: $S(\rho) = \lim_{\alpha \to 1} S_\alpha(\rho)$. For any $\alpha > 0$, the Rényi entropies satisfy the monotonicity property: $S_\alpha(\rho) \ge S_\beta(\rho)$ for $\alpha < \beta$. In particular, we have $S(\rho) \ge S_2(\rho)$ for all density matrices $\rho$. Additionally, in our analysis of artificial entanglement in LLMs, we primarily use the von Neumann entropy because it is the standard measure of entanglement in quantum information theory, providing a direct connection to the quantum information perspective we adopt, it captures the full entanglement spectrum, making it more sensitive to subtle differences in correlation structures (\eg, the entanglement valley), and it satisfies the subadditivity property, which is crucial for understanding how entanglement scales with system size (volume law vs area law).

\textbf{Rényi-2 Entropy Results from Main Text Theorems.}
For completeness, we collect here the Rényi-2 entropy results that accompany the von Neumann entropy results in the main text. For the attention matrix (Theorem~\ref{thm:attn_entropy_main}), the Rényi-2 entropy satisfies $S_2(A) = 2\log(1+\sigma^2) + o(1)$, where $\sigma$ is the bulk spectrum spread parameter, providing a simpler characterization than the full von Neumann entropy scaling. For the stable rank bound (Lemma~\ref{lem:sr_to_entropy_main}), we have $S_2(X) \le 2\log(1+\delta_1)$ where $\delta_1 \le \sqrt{(T-1)\eta}$ and $\eta = r_{\mathrm{stable}}-1$. For the output entanglement collapse (Theorem~\ref{thm:output_entropy_main}), the Rényi-2 entropy satisfies $S_2(X) = O(1/T) \to 0$, vanishing faster than the von Neumann entropy $S(X) = O((\log T)/T) \to 0$. See details in \textit{Appendix~\ref{app:theorem_and_proof}}. 

\section{Theorems and Proofs}
\label{app:theorem_and_proof}

\subsection{Setup: Attention Matrix and Output}

We follow the assumptions from \cite{saadamind}. Let $X_0\in\mathbb{R}^{T\times d}$ have orthonormal rows, \ie,
\begin{equation}
X_0X_0^\top = I_T.
\end{equation}
Define
\begin{equation}
Q := X_0W_Q\in\mathbb{R}^{T\times d_{qk}},\qquad
K := X_0W_K\in\mathbb{R}^{T\times d_{qk}},\qquad
V := X_0W_V\in\mathbb{R}^{T\times d_v},
\end{equation}
softmax attention
\begin{equation}
A := A(X_0) = \mathrm{softmax}\!\left(\frac{QK^\top}{\sqrt{d_{qk}}}\right)\in\mathbb{R}^{T\times T},
\end{equation}
and the single layer output
\begin{equation}
X := AV = A(X_0)\,X_0W_V \ \in\ \mathbb{R}^{T\times d_v}.
\end{equation}

\subsection{Modeling Assumptions}
We prove the theorems under a asymptotic model for softmax
attention at isotropic initialization, as studied in~\cite{saadamind}.

\begin{assumption}[Outlier + quartercircular bulk for attention~\cite{saadamind}]
\label{ass:outlier_bulk}
We note the fact that $A$ is row-stochastic (rows sum to $1$) due to the softmax function, and define
\begin{equation}
A^\perp := A - \frac{1}{T}\mathbf{1}\mathbf{1}^\top,
\end{equation}
where $\mathbf{1} \in \mathbb{R}^T$ denotes the all-ones vector. Assume the following hold as $T\to\infty$:

\textbf{Outlier.}
The top singular value satisfies $s_1(A)\to 1$ and the remainder obeys
\begin{equation}
\|A^\perp\|_{\mathrm{op}} = O(T^{-1/2}).
\end{equation}
Here, $\|\cdot\|_{\mathrm{op}}$ denotes the operator norm, which equals the largest singular value of the matrix. The condition $s_1(A) \to 1$ arises directly from the row-stochastic nature of the softmax attention, where the all-ones vector $\mathbf{1}$ is an eigenvector with eigenvalue $1$. The scaling $\|A^\perp\|_{\mathrm{op}} = O(T^{-1/2})$ reflects the behavior of large random matrices: since the squared Frobenius norm of $A^\perp$, defined as $\|A^\perp\|_F^2 := \sum_{i,j} (A^\perp_{ij})^2$, remains bounded (approx. $O(1)$, see \textbf{Step 2} in \textit{Proof} for \textbf{Theorem} \ref{thm:attn_entropy_full}) and is spread roughly evenly across $T$ dimensions, the strength of the noise along any specific direction scales inversely with the square root of the dimension.

\textbf{Bulk law.}
Let $x_1^{(T)},\dots,x_T^{(T)}$ be the singular values of $\sqrt{T}\,A^\perp$,\footnote{Note that the singular values of $A^\perp$ naturally decay as $s_i(A^\perp) \sim O(T^{-1/2})$, therefore by defining the rescaled values $x_i^{(T)} := \sqrt{T} s_i(A^\perp)$, we bring them back to the macroscopic scale $O(1)$ that allows their empirical distribution to converge to a non-trivial, fixed shape (the quartercircular law) rather than collapsing to a Dirac delta at zero.} their empirical measure $\nu_T:=\frac{1}{T}\sum_{i=1}^T \delta_{x_i^{(T)}}$
converges weakly to the quartercircular law $Q_\sigma$ with parameter $\sigma>0$. Intuitively, this means that as the dimension $T$ grows, the histogram of the singular values becomes a smooth, deterministic curve shaped like a quarter-circle. Moreover, the moments needed below are uniformly integrable:
\begin{equation}
\begin{aligned}
\frac{1}{T}\sum_{i=1}^T (x_i^{(T)})^2 &\to m_2, \\
\frac{1}{T}\sum_{i=1}^T (x_i^{(T)})^4 &\to m_4, \\
\frac{1}{T}\sum_{i=1}^T (x_i^{(T)})^2\log (x_i^{(T)})^2 &\to m_{2\log},
\end{aligned}
\end{equation}
where for the quartercircular law one has $m_2=\sigma^2$ and $m_4<\infty$.
\end{assumption}

\subsection{Entanglement of the Attention Matrix}
\begin{theorem}[Entanglement of the attention matrix: generally $\Theta(\log T)$]
\label{thm:attn_entropy_full}
Under Assumption~\ref{ass:outlier_bulk}~\cite{saadamind}, let
\begin{equation}
\rho_A := \rho_{\mathrm{row}}(A) = \frac{AA^\top}{\mathrm{Tr}(AA^\top)}.
\end{equation}
Then, as $T\to\infty$, the von Neumann and Rényi-2 entanglement entropies satisfy
\begin{equation}
\begin{aligned}
S(A) &= \frac{\sigma^2}{1+\sigma^2}\,\log T \;+\; C_A \;+\; o(1), \\
S_2(A) &= 2\log(1+\sigma^2)\;+\; o(1),
\end{aligned}
\label{eq:SA_S2A_full}
\end{equation}
where the constant term is
\begin{equation}
C_A
=
\log(1+\sigma^2)
+\frac{\sigma^2}{1+\sigma^2}\log(1+\sigma^2)
-\frac{m_{2\log}}{1+\sigma^2}.
\end{equation}
In particular, $S(A)$ typically grows like $\log T$ even when $A$ has a macroscopic singular-value outlier
(equivalently, a macroscopic eigenvalue gap in $AA^\top$). For a discussion of the relationship between von Neumann and Rényi entropies, see \textit{Appendix~\ref{app:von_neumann_renyi}}.
\end{theorem}

\begin{proof}
\textbf{Step 1: Orthogonality and Frobenius decomposition.}
Because $A$ is row-stochastic, $\sum_{j}A_{ij}=1$ for each $i$, hence $\sum_j A^\perp_{ij} = 0$ (since the subtracted term $\frac{1}{T}\mathbf{1}\mathbf{1}^\top$ also has unit row sums) and therefore
\begin{equation}
\begin{aligned}
\left\langle A^\perp, \tfrac{1}{T}\mathbf{1}\mathbf{1}^\top\right\rangle_F
&= \sum_{i,j} A^\perp_{ij}\cdot \frac{1}{T}
= \frac{1}{T}\sum_i\sum_j A^\perp_{ij}
= 0.
\end{aligned}
\end{equation}
Thus,
\begin{equation}
\begin{aligned}
\|A\|_F^2
&= \Big\|\frac{1}{T}\mathbf{1}\mathbf{1}^\top\Big\|_F^2 + \|A^\perp\|_F^2
= 1 + \|A^\perp\|_F^2,
\end{aligned}
\label{eq:Frob_decomp}
\end{equation}
since $\|\frac{1}{T}\mathbf{1}\mathbf{1}^\top\|_F^2=T^2\cdot(1/T^2)=1$.

\textbf{Step 2: Limit of $\|A^\perp\|_F^2$.}
Let $x_i^{(T)}$ be the singular values of $\sqrt{T}A^\perp$.
Then
\begin{equation}
\begin{aligned}
\|A^\perp\|_F^2
&= \sum_{i=1}^T s_i(A^\perp)^2
= \sum_{i=1}^T \frac{(x_i^{(T)})^2}{T}
= \frac{1}{T}\sum_{i=1}^T (x_i^{(T)})^2
\;\to\; m_2=\sigma^2,
\end{aligned}
\end{equation}
where we used Assumption~\ref{ass:outlier_bulk}(2)~\cite{saadamind}.
Combining with \eqref{eq:Frob_decomp} gives
\begin{equation}
\|A\|_F^2 \to 1+\sigma^2.
\label{eq:A_frob_limit}
\end{equation}

\textbf{Step 3: Entanglement spectrum from singular values.}
By the matrix--state correspondence reviewed in the preliminaries,
the eigenvalues of $\rho_A$ are
\begin{equation}
p_i^{(T)}=\frac{s_i(A)^2}{\|A\|_F^2},\quad i=1,\dots,T.
\end{equation}
Assumption~\ref{ass:outlier_bulk}(1) states $s_1(A)\to 1$ and
$\|A^\perp\|_{\mathrm{op}}=O(T^{-1/2})$, and hence for $i\ge 2$ we have the rank-one approximation bound
\begin{equation}
\begin{aligned}
s_i(A)
&\le \left\|A-\tfrac{1}{T}\mathbf{1}\mathbf{1}^\top\right\|_{\mathrm{op}}
= \|A^\perp\|_{\mathrm{op}}
= O(T^{-1/2}),
\end{aligned}
\end{equation}
so $p_i^{(T)}=O(1/T)$ for $i\ge 2$.
Moreover, \eqref{eq:A_frob_limit} implies
\begin{equation}
p_1^{(T)}=\frac{s_1(A)^2}{\|A\|_F^2}\to \frac{1}{1+\sigma^2}.
\label{eq:p1_limit}
\end{equation}

\textbf{Step 4: Rényi-2 entropy.}
We have
\begin{equation}
\begin{aligned}
\mathrm{Tr}(\rho_A^2)
&= \sum_{i=1}^T (p_i^{(T)})^2
= (p_1^{(T)})^2 + \sum_{i\ge 2}(p_i^{(T)})^2.
\end{aligned}
\end{equation}
For the tail term, using $p_i^{(T)}=s_i(A)^2/\|A\|_F^2$ and the bound $s_i(A)=O(T^{-1/2})$ for $i\ge 2$,
\begin{equation}
\begin{aligned}
\sum_{i\ge 2}(p_i^{(T)})^2
&= \frac{1}{\|A\|_F^4}\sum_{i\ge 2} s_i(A)^4
\le \frac{T-1}{\|A\|_F^4}\,\|A^\perp\|_{\mathrm{op}}^4
= O\!\left(\frac{1}{T}\right),
\end{aligned}
\end{equation}
where we used $\|A^\perp\|_{\mathrm{op}}=O(T^{-1/2})$ from Assumption~\ref{ass:outlier_bulk}
and $\|A\|_F^2\to 1+\sigma^2$ from \eqref{eq:A_frob_limit}.
Therefore,
\begin{equation}
\begin{aligned}
\mathrm{Tr}(\rho_A^2)
&= (p_1^{(T)})^2 + O(1/T)
\to \frac{1}{(1+\sigma^2)^2},
\end{aligned}
\end{equation}
and hence
\begin{equation}
S_2(A)=-\log \mathrm{Tr}(\rho_A^2) \to 2\log(1+\sigma^2),
\end{equation}
which matches \eqref{eq:SA_S2A_full}. For definitions and properties of Rényi-2 entropy, see \textit{Appendix~\ref{app:von_neumann_renyi}}.

\textbf{Step 5: von Neumann entropy.}
Write
\begin{equation}
S(A)= -p_1^{(T)}\log p_1^{(T)} \;-\; \sum_{i\ge 2} p_i^{(T)}\log p_i^{(T)}.
\end{equation}
The first term converges using \eqref{eq:p1_limit}:
\begin{equation}
-p_1^{(T)}\log p_1^{(T)} \to \frac{1}{1+\sigma^2}\log(1+\sigma^2).
\label{eq:head_term}
\end{equation}
For the tail, we use the representation $s_i(A^\perp) = x_i^{(T)}/\sqrt{T}$ and work directly at the level of normalized sums. 
Under Assumption~\ref{ass:outlier_bulk} (uniform integrability of the relevant moments), replacing $\{s_i(A)\}_{i\ge 2}$ by $\{s_i(A^\perp)\}_{i=1}^T$ in the forthcoming Riemann–sum expressions incurs only an $o(1)$ error as $T\to\infty$. 
This is justified by the theory of finite-rank perturbations. 
Since $A$ is a rank-1 perturbation of $A^\perp$, the rank inequalityfor singular value distributions implies that their empirical cumulative distribution functions differ by at most $1/T$ in Kolmogorov distance (defined as \(\|F-G\|_\infty
:= \sup_{x\in\mathbb{R}} |F(x)-G(x)|\) for two distribution $F$ and $G$):
\begin{equation}
    \bigl\|F_A - F_{A^\perp}\bigr\|_\infty 
\;\le\; \frac{\operatorname{rank}(A-A^\perp)}{T}
\;=\; \frac{1}{T},\footnote{Although the classical rank inequality (see Theorem~A.43 in \cite{bai2010spectral}) is
stated for eigenvalues of Hermitian matrices, the corresponding result for
singular values of arbitrary matrices follows immediately from Hermitian
dilation. Recall that the singular values of $A$ are the eigenvalues of the
Hermitian matrix \(
H(A)
=
\begin{pmatrix}
0 & A \\
A^\top & 0
\end{pmatrix}. \)
If $A = A^\perp + uv^\top$ with $\operatorname{rank}(uv^\top)=1$, then \(
\operatorname{rank}\!\bigl(H(A)-H(A^\perp)\bigr)=2.\) Applying the rank inequality to $H(A)$ and $H(A^\perp)$ yields

\begin{equation}
    \|F_A - F_{A^\perp}\|_\infty
=
\|F_{H(A)} - F_{H(A^\perp)}\|_\infty
\le \frac{2}{2T}
= \frac{1}{T}.
\end{equation}

Equivalently, a rank-1 perturbation of a matrix can change at most one singular
value index, so the empirical singular-value distribution can shift vertically
by at most $1/T$.}
\end{equation}
where the empirical singular value distribution (or empirical spectral distribution (ESD)) of $A$ is defined as
\(
F_A(x)
:=\frac{1}{T}\#\{\, i : s_i(A)\le x \}\) where \(\# E\)
denotes the cardinality of the set \(E\), and similarly for $F_{A^\perp}$. 
Here $s_1(A)\ge \cdots \ge s_T(A)$ denote the singular values. Consequently, only $O(1)$ singular values can deviate by a non-negligible amount, which is negligible compared to the ambient dimension $T$\footnote{
Formally, the ``index-shift inequality'' for singular
values proved by \cite{thompson1976behavior} (See Theorem 1, 3 in \cite{thompson1976behavior}) guarantee stability of individual singular values, while the rank inequality ensures that 
$\|F_A - F_{A^\perp}\|_\infty \le \mathrm{rank}(A-A^\perp)/T = 1/T$. 
Therefore, the normalized counting measures of $A$ and $A^\perp$ are asymptotically identical.}.
Because our sums are normalized by $1/T$ (implicitly through the Frobenius-norm–based weights\footnote{The $1/T$ normalization arises because the Frobenius norm induces weights of the 
form $p_i = s_i(A)^2 / \|A\|_F^2$, and bulk singular values satisfy 
$s_i(A^\perp) = O(T^{-1/2})$, so that $p_i = O(1/T)$. Consequently, any spectral sum of the form $\sum_i p_i\,\Psi(s_i)$ is effectively an average over $T$ terms. Since a rank-1 perturbation can alter only $O(1)$ singular values, the total contribution of these mismatched
indices is at most $O(1)\cdot (1/T) = O(1/T)$, which vanishes as $T\to\infty$.
}), the contribution of these $O(1)$ mismatched indices vanishes at rate $O(1/T)$. The remaining sum over the bulk can thus be treated as a Riemann sum that converges to the integral defined by the limiting spectral distribution of $\sqrt{T}A^\perp$. Concretely,
\begin{equation}
\begin{aligned}
\sum_{i\ge 2} p_i^{(T)}\log p_i^{(T)}
&=
\frac{1}{\|A\|_F^2}\sum_{i\ge 2} s_i(A)^2 \log\Big(\frac{s_i(A)^2}{\|A\|_F^2}\Big) \\
&=
\frac{1}{\|A\|_F^2}\sum_{i=1}^{T} s_i(A^\perp)^2 \log\Big(\frac{s_i(A^\perp)^2}{\|A\|_F^2}\Big)
+ o(1).
\end{aligned}
\end{equation}
Now substitute $s_i(A^\perp)^2 = (x_i^{(T)})^2/T$:
\begin{equation}
\begin{aligned}
\sum_{i\ge 2} p_i^{(T)}\log p_i^{(T)}
&=
\frac{1}{\|A\|_F^2}\sum_{i=1}^{T} \frac{(x_i^{(T)})^2}{T}\,
\log\Big(\frac{(x_i^{(T)})^2}{T\|A\|_F^2}\Big) + o(1) \\
&=
\frac{1}{\|A\|_F^2}\sum_{i=1}^{T} \frac{(x_i^{(T)})^2}{T}\,
\Big(\log (x_i^{(T)})^2 - \log T - \log\|A\|_F^2\Big) + o(1).
\end{aligned}
\end{equation}
Therefore,
\begin{equation}
\begin{aligned}
-\sum_{i\ge 2} p_i^{(T)}\log p_i^{(T)}
&=
\frac{\log T}{\|A\|_F^2}\cdot \frac{1}{T}\sum_{i=1}^T (x_i^{(T)})^2
+\frac{\log\|A\|_F^2}{\|A\|_F^2}\cdot \frac{1}{T}\sum_{i=1}^T (x_i^{(T)})^2 \\
&\qquad
-\frac{1}{\|A\|_F^2}\cdot \frac{1}{T}\sum_{i=1}^T (x_i^{(T)})^2\log (x_i^{(T)})^2
+ o(1).
\end{aligned}
\label{eq:tail_entropy_expansion}
\end{equation}
Using Assumption~\ref{ass:outlier_bulk},
$\frac{1}{T}\sum_i (x_i^{(T)})^2\to m_2=\sigma^2$ and
$\frac{1}{T}\sum_i (x_i^{(T)})^2\log(x_i^{(T)})^2\to m_{2\log}$,
and using $\|A\|_F^2\to 1+\sigma^2$, we obtain
\begin{equation}
\begin{aligned}
-\sum_{i\ge 2} p_i^{(T)}\log p_i^{(T)}
&=
\frac{\sigma^2}{1+\sigma^2}\log T
+\frac{\sigma^2}{1+\sigma^2}\log(1+\sigma^2)
-\frac{m_{2\log}}{1+\sigma^2}
+o(1).
\end{aligned}
\end{equation}
Combining with \eqref{eq:head_term} yields \eqref{eq:SA_S2A_full} with the stated constant term.
\end{proof}

\textbf{The Attention Cardy Formula: Critical Scaling and Physical Interpretation.}
Here we provides the full proof of the Attention Cardy Formula stated in Theorem~\ref{thm:attn_entropy_main} (see Section~\ref{sec:random_matrix_theory}). The complete derivation shows that $S(A) = \mathcal{C}_{\mathrm{attn}} \log T + C_A + o(1)$, where the constant term $C_A$ is explicitly given in \eqref{eq:SA_S2A_full}. For the physical interpretation, connection to conformal field theory, and discussion of criticality and long-range correlations, see the main text (Section~\ref{sec:random_matrix_theory} and Remark following Theorem~\ref{thm:attn_entropy_main}).

\subsection{A General Bridge Between Stable Rank and Entanglement Bounds}

To have \textbf{Theorem \ref{thm:output_entropy_full}}, we firstly have \textbf{Lemma \ref{lem:sr_to_entropy}} shown below:

\begin{lemma}[Stable rank control implies small entanglement]
\label{lem:sr_to_entropy}
Let $M\in\mathbb{R}^{T\times d}$, $\Sigma:=MM^\top\succeq 0$ with eigenvalues
$\lambda_1\ge\cdots\ge\lambda_T\ge 0$.
Let $\rho:=\Sigma/\mathrm{Tr}(\Sigma)$ and define the stable rank
\begin{equation}
r_{\mathrm{stable}}(\Sigma)
:=\frac{\mathrm{Tr}(\Sigma^2)}{\lambda_1^2}
= 1+\sum_{i\ge 2}\Big(\frac{\lambda_i}{\lambda_1}\Big)^2.
\end{equation}
Let
\begin{equation}
\eta := r_{\mathrm{stable}}(\Sigma)-1
= \sum_{i\ge 2}\Big(\frac{\lambda_i}{\lambda_1}\Big)^2.
\end{equation}
Then the tail mass satisfies
\begin{equation}
\delta_1 := \sum_{i\ge 2}\frac{\lambda_i}{\lambda_1}
\ \le\ \sqrt{(T-1)\eta}.
\label{eq:delta1_bound}
\end{equation}
Consequently, letting $\delta:=\min\{1,\delta_1\}$,
\begin{equation}
\begin{aligned}
1-\lambda_{\max}(\rho) &\le \frac{\delta_1}{1+\delta_1}\le \delta, \\
S(M) = S(\rho) &\le h_2(\delta) + \delta\log(T-1), \\
S_2(M)=S_2(\rho) &\le 2\log(1+\delta_1),
\end{aligned}
\label{eq:bridge_bounds}
\end{equation}
where $h_2(u):=-u\log u-(1-u)\log(1-u)$.
\end{lemma}

\begin{proof}
\textbf{Step 1: Tail $\ell_1$ control from stable rank.}
By Cauchy--Schwarz\footnote{Applying the Cauchy-Schwarz inequality to the vector of tail eigenvalue ratios $u=(\lambda_2/\lambda_1, \dots, \lambda_T/\lambda_1)$ and the all-ones vector $\mathbf{1} \in \mathbb{R}^{T-1}$ yields $(\sum_{i\ge 2} \lambda_i/\lambda_1)^2 = \langle u, \mathbf{1} \rangle^2 \le \|u\|_2^2 \|\mathbf{1}\|_2^2$. Since $\|u\|_2^2 = \eta$ and $\|\mathbf{1}\|_2^2 = T-1$, we have $\delta_1^2 \le \eta(T-1)$.},
\begin{equation}
\Big(\sum_{i\ge 2}\lambda_i\Big)^2 \le (T-1)\sum_{i\ge 2}\lambda_i^2.
\end{equation}
Dividing by $\lambda_1^2$ gives
\begin{equation}
\begin{aligned}
\delta_1^2
&= \Big(\sum_{i\ge 2}\frac{\lambda_i}{\lambda_1}\Big)^2
\le (T-1)\sum_{i\ge 2}\Big(\frac{\lambda_i}{\lambda_1}\Big)^2
= (T-1)\eta,
\end{aligned}
\end{equation}
which proves \eqref{eq:delta1_bound}.

\textbf{Step 2: Bound the top eigenvalue of $\rho$.}
Since $\mathrm{Tr}(\Sigma)=\lambda_1+\sum_{i\ge2}\lambda_i=\lambda_1(1+\delta_1)$,
\begin{equation}
\lambda_{\max}(\rho)=\frac{\lambda_1}{\mathrm{Tr}(\Sigma)}=\frac{1}{1+\delta_1}.
\end{equation}
Therefore,
\begin{equation}
1-\lambda_{\max}(\rho)=1-\frac{1}{1+\delta_1}=\frac{\delta_1}{1+\delta_1}\le \min\{1,\delta_1\}=\delta,
\end{equation}
establishing the first inequality in \eqref{eq:bridge_bounds}.

\textbf{Step 3: Von Neumann entropy upper bound.}
Let $p_1\ge p_2\ge \cdots\ge p_T$ be eigenvalues of $\rho$.
From Step 2, $p_1\ge 1-\delta$ and thus $\sum_{i\ge2}p_i\le \delta$ due to normalization condition. Among all distributions on $T$ points with tail mass at most $\delta$, the Von Neumann
entropy is maximized by $q_1=1-\delta$ and $q_2=\cdots=q_T=\delta/(T-1)$.
Hence $S(\rho)\le H(q)$, and
\begin{equation}
\begin{aligned}
H(q)
&= -(1-\delta)\log(1-\delta) -(T-1)\frac{\delta}{T-1}\log\Big(\frac{\delta}{T-1}\Big) \\
&= h_2(\delta) + \delta\log(T-1),
\end{aligned}
\end{equation}
which gives the second inequality in \eqref{eq:bridge_bounds}.

\textbf{Step 4: Rényi-2 bound.}
Since $\mathrm{Tr}(\rho^2)=\sum_i p_i^2 \ge p_1^2$ and $p_1=\lambda_{\max}(\rho)=1/(1+\delta_1)$,
\begin{equation}
S_2(\rho)=-\log \mathrm{Tr}(\rho^2)\le -\log p_1^2 = 2\log(1+\delta_1),
\end{equation}
which is the third inequality in \eqref{eq:bridge_bounds}. For definitions and properties of Rényi-2 entropy, see \textit{Appendix~\ref{app:von_neumann_renyi}}.
\end{proof}

\subsection{Output Entanglement Collapse from Stable Rank Collapse}

\begin{theorem}[Output entanglement collapse (token--feature), conditional]
\label{thm:output_entropy_full}
Let $X=AV\in\mathbb{R}^{T\times d_v}$ with $V=X_0W_V$, and define
\begin{equation}
\Sigma:=XX^\top.
\end{equation}
Assume that with overwhelming probability%
\footnote{Overwhelming probability means $1-T^{-c}$ for every fixed $c>0$, for all sufficiently large $T$.}
\begin{equation}
r_{\mathrm{stable}}(\Sigma)-1 = O(T^{-3}),
\label{eq:sr_rate_assumption}
\end{equation}
as established in rank-collapsed regimes in~\cite{saadamind}.
Then, with overwhelming probability,
\begin{equation}
\begin{aligned}
S(X)&=O\!\left(\frac{\log T}{T}\right)\to 0,\\
S_2(X)&=O\!\left(\frac{1}{T}\right)\to 0.
\end{aligned}
\end{equation}
For a discussion of Rényi-2 entropy and its relationship to von Neumann entropy, see \textit{Appendix~\ref{app:von_neumann_renyi}}.
\end{theorem}

\begin{proof}
Apply Lemma~\ref{lem:sr_to_entropy} with $\eta=r_{\mathrm{stable}}(\Sigma)-1$.
Under \eqref{eq:sr_rate_assumption} we have $\delta_1\le \sqrt{(T-1)\eta}=O(T^{-1})$ and hence
$\delta=\min\{1,\delta_1\}=O(T^{-1})$.
Lemma~\ref{lem:sr_to_entropy} yields
\begin{equation}
S(X)\le h_2(\delta)+\delta\log(T-1).
\end{equation}
Using $h_2(\delta)=O(\delta\log(1/\delta))=O((\log T)/T)$ for $\delta=O(1/T)$ gives
$S(X)=O((\log T)/T)\to 0$.
Similarly, by Lemma~\ref{lem:sr_to_entropy},
\begin{equation}
S_2(X)\le 2\log(1+\delta_1)=O(\delta_1)=O(1/T)\to 0.
\end{equation}
For definitions and properties of Rényi-2 entropy, see \textit{Appendix~\ref{app:von_neumann_renyi}}.
\end{proof}

\begin{remark}
\label{rem:init_vs_collapse}
Theorem~\ref{thm:output_entropy_full} is \emph{conditional}. Under isotropic random initialization of $W_V$
and $X_0X_0^\top=I_T$, one often has $VV^\top\approx I_T$ and hence $\Sigma\approx AA^\top$~\cite{saadamind},
so the output entanglement scaling can mirror that of $A$ in
Theorem~\ref{thm:attn_entropy_full} rather than vanishing.
The stable rank collapse condition $r_{\mathrm{stable}}(\Sigma)-1\ll 1$ should therefore be interpreted as a
property that can emerge in a trained or otherwise constrained regime and can be validated empirically.
\end{remark}

\begin{figure*}[htbp]
    \centering
\includegraphics[width=0.8\linewidth]{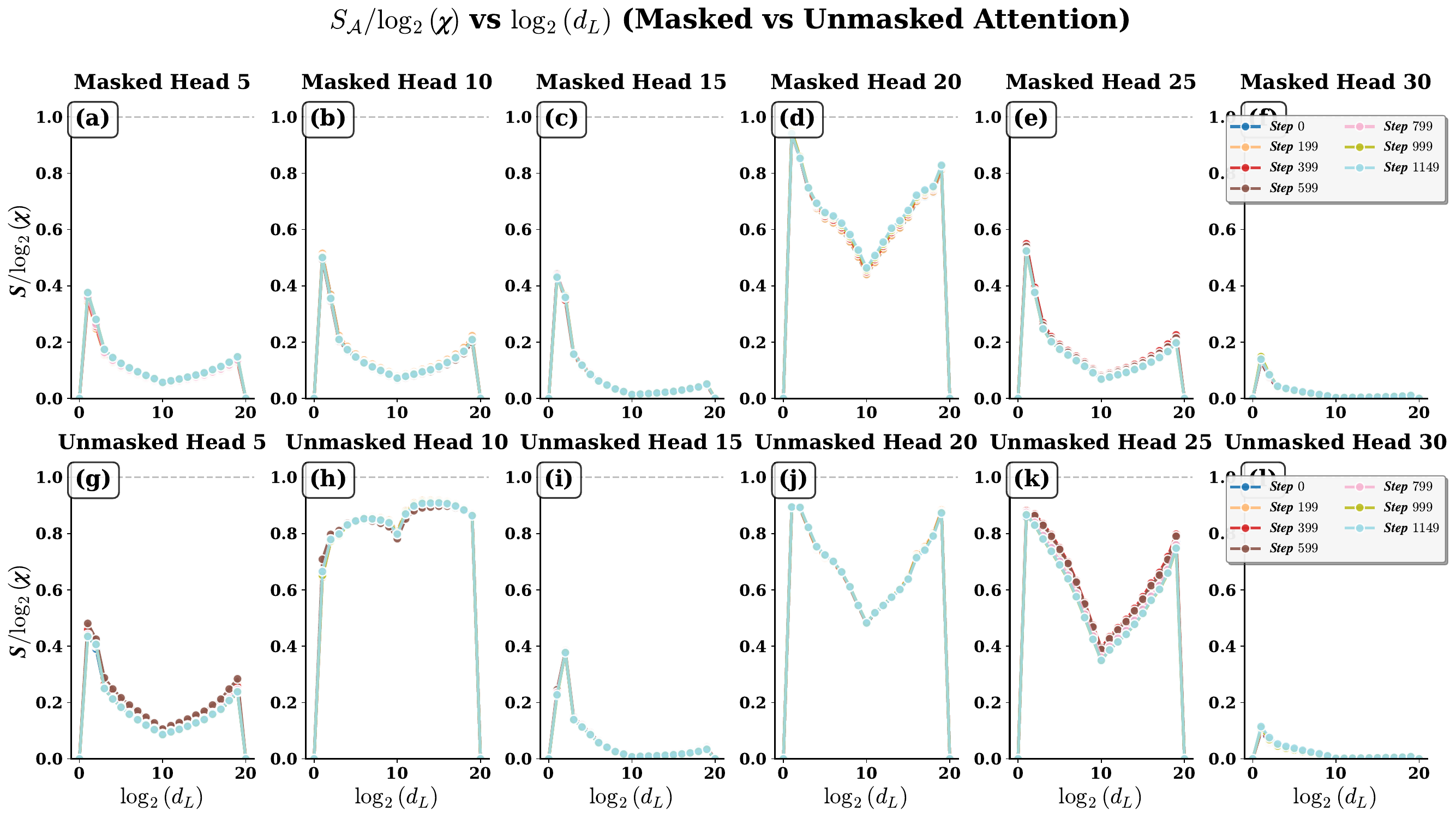}
\caption{
Comparison of masked and unmasked attention entanglement entropy across multiple heads. For each selected head, we plot the normalized entanglement entropy $S_A/\log(\chi)$ as with respect to the bi-partition during training. Removing the causal mask leads to markedly different behaviors across heads: some heads (Head 10, 25) exhibit a substantial increase in entanglement entropy when future positions become accessible, indicating a strong reliance on the causal constraint, while others remain almost unchanged, revealing intrinsically low-entanglement attention patterns that are insensitive to masking.}
\label{fig:masked_unmasked_attention_entropy_normalized_comparison}
\end{figure*}

\section{MPS Adaptation: Method and Relationship to MPS Modeling}
\label{app:MPS_adaptation_method}

This section clarifies the distinction between \textit{MPS modeling} (an analysis tool) and \textit{MPS adaptation} (a PEFT method), explains their relationship, and demonstrates why MPS adaptation can be parameter-efficient.

\textbf{MPS Modeling vs. MPS Adaptation.} It is crucial to distinguish between two different uses of MPS in this work. \textbf{MPS Modeling (Analysis Tool):} We use MPS decomposition as an \textit{analysis framework} to study the entanglement structure of parameter updates from LoRA and FFT. Specifically, we decompose the learned weight update matrices $\Delta W$ into MPS representations post-hoc to compute entanglement entropy profiles. This is purely for analysis purposes: the training process itself does not use MPS structure. \textbf{MPS Adaptation (PEFT Method):} We term and use MPS adaptation that directly parameterizes weight updates using MPS structure during training. Unlike LoRA which uses low-rank matrices $BA$, MPS adaptation uses tensor network structure as the fundamental parameterization, making the MPS structure intrinsic to the fine-tuning method rather than a tool for analysis.

\textbf{Relationship and Connection.} The connection between MPS modeling and MPS adaptation is that both utilize the same mathematical structure (MPS), but for different purposes: MPS modeling uses it to \textit{analyze} existing methods, while MPS adaptation uses it to \textit{define} a fine-tuning method. This allows us to apply the same entanglement analysis framework to MPS adaptation, enabling direct comparison of entanglement structures across different fine-tuning paradigms.

\textbf{Parameter Efficiency of MPS Adaptation.} To illustrate why MPS adaptation can be parameter-efficient, consider a concrete example. For a weight update matrix $\Delta W \in \mathbb{R}^{d_{\mathrm{out}} \times d_{\mathrm{in}}}$ with $d_{\mathrm{out}} = 4096$ and $d_{\mathrm{in}} = 4096$, full fine-tuning requires $4096 \times 4096 = 16{,}777{,}216$ parameters. 

LoRA with rank $r=256$ parameterizes $\Delta W = \frac{\alpha}{r} BA$ where $B \in \mathbb{R}^{4096 \times 256}$ and $A \in \mathbb{R}^{256 \times 4096}$, requiring $4096 \times 256 + 256 \times 4096 = 2{,}097{,}152$ parameters.

MPS adaptation factorizes the input dimension as $d_{\mathrm{in}} = d_1 \times d_2$ (e.g., $4096 = 64 \times 64$) and parameterizes $A \in \mathbb{R}^{r \times d_{\mathrm{in}}}$ as an MPS with bond dimension $\chi$, denoted as $A_{\text{MPS}}$:
\begin{equation}
A_{\text{MPS}} = \sum_{\alpha=1}^{\chi} \mathcal{A}^{[1]}_{r, \alpha, d_1} \mathcal{A}^{[2]}_{\alpha, 1, d_2},
\end{equation}
where $\mathcal{A}^{[1]} \in \mathbb{R}^{r \times \chi \times d_1}$ and $\mathcal{A}^{[2]} \in \mathbb{R}^{\chi \times 1 \times d_2}$ are Order-3 tensors. The weight update is then $\Delta W = \frac{\alpha}{r} B A_{\text{MPS}}$ where $B \in \mathbb{R}^{d_{\mathrm{out}} \times r}$ is the same matrix as in LoRA. With $r=256$ and $\chi=32$, this requires $256 \times 32 \times 64 + 32 \times 1 \times 64 = 524{,}288 + 2{,}048 = 526{,}336$ parameters for $A_{\text{MPS}}$, plus $4096 \times 256 = 1{,}048{,}576$ parameters for $B$, totaling approximately $1{,}574{,}912$ parameters---significantly fewer than LoRA while maintaining similar expressivity and test performance (see FIG.~\ref{fig:lora_vs_tensor_lora_a_comparison}) through the tensor network structure. The parameter efficiency arises from the hierarchical factorization: instead of storing a full $r \times d_{\mathrm{in}}$ matrix, we factorize it into smaller tensors whose total parameter count scales as $O(r \chi d_1 + \chi d_2)$ rather than $O(r d_{\mathrm{in}})$, where typically $\chi \ll r$ and $d_1, d_2 \ll d_{\mathrm{in}}$.

\section{More Experimental Results}
\label{app:more_experimental_results}

\subsection{Ablation Study on Causal Masking in Autoregressive Attention}
\label{app:abla_mask}

In autoregressive language models, the causal mask restricts each token to attend only to previous tokens, which may influence the entanglement structure we observe in attention matrices. To understand how the masking operation affects our entanglement analysis, we investigate whether the low entanglement entropy in attention matrices is intrinsic to the learned attention patterns or is an artifact of the causal constraint. To illustrate the contribution of the causal mask to the model's attention 
structure, we compute \textit{unmasked attention matrices} during evaluation. 
We attach forward hooks (callback functions that intercept intermediate activations during the forward pass) to each attention layer to capture the 
pre-RoPE query and key projections,

\begin{equation}
    Q_{\mathrm{raw}} = X W_Q, \qquad 
K_{\mathrm{raw}} = X W_K,
\end{equation}
together with the corresponding position ids. The projections are 
reshaped into \([B, H, T, d]\) format with grouped-query attention handled 
accordingly, after which we reconstruct the rotary position embedding by 
computing the appropriate cosine and sine rotations, and applying them to each head. With the RoPE-rotated queries and keys, we recompute attention logits
\begin{equation}
    S = \frac{QK^\top}{\sqrt{d}},
\end{equation}
using the same numerical conventions as the model but without applying the causal mask that normally eliminates future positions. A softmax 
over these logits yields the unmasked attention matrix
\begin{equation}
    A_{\mathrm{unmasked}} = \mathrm{softmax}(S),
\end{equation}

As shown in FIG.~\ref{fig:masked_unmasked_attention_entropy_normalized_comparison}, removing the 
causal mask reveals a clear head-dependent entanglement structure. For some 
heads (\eg, Head~10 and Head~25), the entanglement entropy increases 
substantially once future positions become accessible, indicating that the 
causal boundary had been suppressing a higher effective Schmidt rank. In 
contrast, other heads show nearly identical masked and unmasked entanglement, 
implying that their correlation patterns are intrinsically low-entangled and 
generate only weak entanglement even when the accessible Hilbert space is 
expanded. This contrast highlights a heterogeneity in the ``entanglement 
capacity'' of attention heads.

\begin{figure*}[htbp]
    \centering
\includegraphics[width=\linewidth]{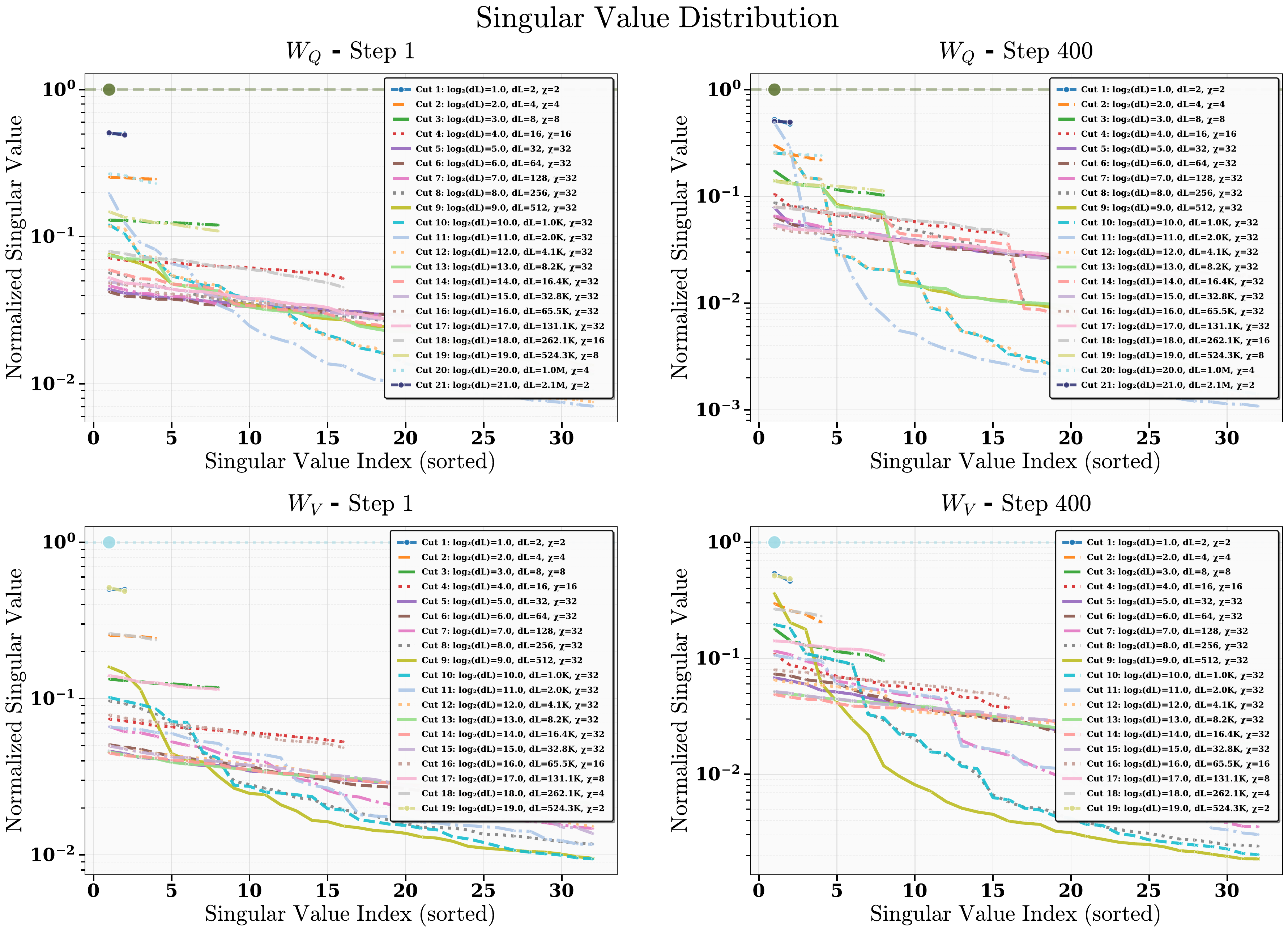}
\caption{Normalized singular value spectra $\lambda_i^{(\ell)}$ at different cut positions $\ell$ for $W_Q$ and $W_V$ matrices at the first and last training steps. Each curve corresponds to a different cut position, with the singular values sorted in descending order and normalized by their $\ell_2$ norm. The pronounced suppression of intermediate and small singular values near the mid-cut (row--column bi-partition) provides a direct spectral explanation for the entanglement valley observed in the entanglement entropy profile.}
\label{fig:singular_dist_combined_q_v_first_last}
\end{figure*}

\subsection{The Entanglement Valley}
\label{app:SVD_dist_cut}

In the main text, we characterize the internal correlation structure of the updates of weight projection matrices (specifically, the updates of query and value projection matrices $\Delta W_Q$ and $\Delta W_V$) by computing the entanglement entropy as a function of the bi-partition (cut) position.
By modeling these update matrices as MPS and sweeping the cut from left to right, one
obtains an artificial entanglement profile $S(\ell)$. Empirically, this profile exhibits a pronounced depression near the middle cut
(\eg, FIG.~\ref{fig:lora_wq_wv}), which we refer to as the
\emph{entanglement valley}. For this valley, we consider the LoRA update matrix
\begin{equation}
\Delta W = \frac{\alpha}{r} B A \in \mathbb{R}^{d_{\mathrm{out}}\times d_{\mathrm{in}}},
\qquad
B\in\mathbb{R}^{d_{\mathrm{out}}\times r},\;
A\in\mathbb{R}^{r\times d_{\mathrm{in}}},
\end{equation}
at initialization, where the entries of $A$ and $B$ are i.i.d.\ with zero mean and
finite variance, and $r \ll \min(d_{\mathrm{out}}, d_{\mathrm{in}})$. As discussed in our methods, we flatten $\Delta W$ into a vector using row-major order,
\begin{equation}
\mathrm{vec}(\Delta W) \in \mathbb{R}^{d_{\mathrm{out}} d_{\mathrm{in}}},
\end{equation}
and factorize the total dimension into powers of two,
\begin{equation}
d_{\mathrm{out}} = 2^{m},\qquad d_{\mathrm{in}} = 2^{n},
\end{equation}
(up to an inessential residual factor).
We then reshape $\mathrm{vec}(\Delta W)$ into an order-$(m+n)$ tensor with local
dimension $2$ at each site and perform a left-to-right MPS sweep. The bond $\ell$ corresponds to a bi-partition between the first $\ell$ binary
indices and the remaining $m+n-\ell$ indices. In particular, the bond \(\ell = m\) corresponds exactly to the \emph{row--column bi-partition}
$\mathbb{R}^{d_{\mathrm{out}}}\otimes\mathbb{R}^{d_{\mathrm{in}}}$, where the first $m$ binary indices encode the row dimension (output space) and the remaining $n$ indices encode the column dimension (input space), effectively separating the matrix into its row and column subsystems.

\begin{lemma}[Entanglement bound at the row--column cut]
\label{lem:lora_row_col_bound}
At the row--column cut $\ell=m$, the Schmidt rank of
$\mathrm{vec}(\Delta W)$ is at most $r$.
Consequently, the von Neumann and Rényi entanglement entropies satisfy
\begin{equation}
S_{\ell=m}(\Delta W) \;\le\; \log r,
\qquad
S_{\alpha,\ell=m}(\Delta W) \;\le\; \log r,
\quad \forall\,\alpha>0.
\end{equation}
For a discussion of Rényi entropies and their relationship to von Neumann entropy, see \textit{Appendix~\ref{app:von_neumann_renyi}}.
\end{lemma}

\begin{proof}
Write the rank-$r$ decomposition explicitly:
\begin{equation}
\Delta W
= \frac{\alpha}{r}\sum_{a=1}^{r} b_a a_a^\top,
\end{equation}
where $b_a\in\mathbb{R}^{d_{\mathrm{out}}}$ is the $a$-th column of $B$ and
$a_a^\top\in\mathbb{R}^{d_{\mathrm{in}}}$ is the $a$-th row of $A$.
Using the identity $\mathrm{vec}(uv^\top)=u\otimes v$, we obtain
\begin{equation}
\mathrm{vec}(\Delta W)
= \frac{\alpha}{r}\sum_{a=1}^{r} b_a \otimes a_a
\;\in\;
\mathbb{R}^{d_{\mathrm{out}}}\otimes\mathbb{R}^{d_{\mathrm{in}}}.
\end{equation}
Thus $\mathrm{vec}(\Delta W)$ lies in the span of at most $r$ product states.
Hence its Schmidt rank across the row--column bi-partition is at most $r$. A pure state with Schmidt rank at most $r$ has at most $r$ nonzero Schmidt
coefficients. The entanglement entropy is therefore maximized by the uniform
distribution $1/r$, yielding the bound $S\le \log r$ for von Neumann entropy and,
more generally, for all Rényi entropies.
\end{proof}

\begin{proposition}[Entanglement valley of $\Delta W$ at initialization]
\label{thm:lora_valley}
Assume $d_{\mathrm{out}}=2^{m}$ and $d_{\mathrm{in}}=2^{n}$ with $m,n\gg 1$.
Let $S_\ell(\Delta W)$ denote the entanglement entropy at bond $\ell$ in the MPS
sweep defined above. Then:

\textbf{Hard bottleneck.} At the row--column cut $\ell=m$,
\begin{equation}
S_{\ell=m}(\Delta W) \;\le\; \log r.
\end{equation}

\textbf{Typical growth away from the bottleneck.}
For cuts $\ell$ that are well inside the row region or the column region,
the entanglement entropy typically follows a volume-law scaling. More precisely, for $ \log r \ll \ell \ll m $ (row interior), where
$\ell$ is large enough that the local Hilbert-space dimension
$2^\ell$ is not the limiting factor, yet sufficiently far from $m$
so that the row--column low-rank structure is not fully exposed,
and for $ m \ll \ell \ll m+n-\log r $ (column interior), one expects with high probability over the random initialization of $A$ and $B$ that
\begin{equation}
S_{\ell}(\Delta W) \approx \min(\ell,\, m+n-\ell)\log 2 - O(1),
\end{equation}
and in particular $S_{\ell}(\Delta W) > S_{\ell=m}(\Delta W)$ since
$S_{\ell=m}(\Delta W)\le \log r$.
Consequently, the entanglement profile $\ell\mapsto S_\ell(\Delta W)$ exhibits
a suppression near $\ell=m$, forming an entanglement valley.
If $m\approx n$, this suppression appears near the middle of the sweep.
\end{proposition}

\begin{proof}
The hard bottleneck statement is exactly Lemma~\ref{lem:lora_row_col_bound},
which gives $S_{\ell=m}(\Delta W)\le \log r$.  For the typical growth statement, we give an intuition based on typical
high-dimensional concentration (the phenomenon that random high-dimensional vectors exhibit predictable, concentrated statistical properties, as exemplified by Page's theorem~\cite{page1993average} for random quantum states). Consider first cuts
$\ell<m$, which split only the $m$ row bits. For each $a\in[r]$, the column
vector $b_a\in\mathbb{R}^{2^m}$ has i.i.d.\ entries with zero mean and finite
variance, hence after normalization it is an isotropic random vector.
When reshaped into an order-$m$ tensor with local dimension $2$, such a random
vector is well-approximated in typical entanglement statistics by a Haar-random
pure state on $m$ qubits. By Page-type concentration~\cite{page1993average}, for any internal cut
$1\ll \ell<m$, one typically has
\begin{equation}
S_\ell(b_a)=\min(\ell,m-\ell)\log 2 - O(1).
\label{eq:page_like_ba}
\end{equation}
Now consider $\Delta W=\frac{\alpha}{r}\sum_{a=1}^r b_a a_a^\top$. At a cut
$\ell<m$, the bi-partition is entirely within the row subsystem, so the special
row--column low-rank structure is not yet fully exposed as a \emph{hard}
bottleneck. Empirically and heuristically, in the interior regime
$\log r \ll \ell \ll m$, the effective Schmidt rank available across the cut
is dominated by the local Hilbert-space dimension $2^\ell$ rather than by the
global factorization rank $r$, and the reduced spectrum is close to that of a
typical high-dimensional state. This leads to a near-maximal, volume-law scaling
\begin{equation}
S_\ell(\Delta W)\approx \ell\log 2 - O(1),
\qquad \log r \ll \ell \ll m,
\end{equation}
and hence $S_\ell(\Delta W) > S_{\ell=m}(\Delta W)$ since
$S_{\ell=m}(\Delta W)\le \log r$. As $\ell$ approaches $m$ from below, the cut begins to probe the global
row--column factorization of $\Delta W$, and the entanglement is progressively
constrained, interpolating between the interior volume-law behavior and the
bottleneck bound at $\ell=m$. An analogous argument applies to cuts $\ell>m$
inside the column subsystem, yielding
\begin{equation}
S_\ell(\Delta W)\approx (m+n-\ell)\log 2 - O(1),
\qquad m \ll \ell \ll m+n-\log r.
\end{equation}
Therefore the profile $\ell\mapsto S_\ell(\Delta W)$ is suppressed near
$\ell=m$, forming an entanglement valley; when $m\approx n$, this suppression
appears near the middle of the sweep.
\end{proof}

\textbf{Why the specific valley position: a symmetry perspective.}
From a symmetry perspective, the appearance of the entanglement valley at a
specific bi-partition can be understood as a consequence of an explicit
structural symmetry breaking induced by the low-rank parameterization.
For a generic full-rank matrix, the vectorized state
$\mathrm{vec}(W)$ does not distinguish one internal bi-partition from another,
and the entanglement profile is approximately invariant under translations of
the cut position.
In contrast, the factorized form $\Delta W = BA$ singles out a preferred
decomposition between the row and column spaces,
$\mathbb{R}^{d_{\mathrm{out}}}\otimes\mathbb{R}^{d_{\mathrm{in}}}$.
The corresponding cut $\ell = m$ is the unique bi-partition that respects this
structural decomposition, while all other cuts mix the two subspaces and are
therefore insensitive to the low-rank constraint.
As a result, the effective cut-translation symmetry is explicitly broken, and
the entanglement bottleneck associated with the rank-$r$ coupling becomes
manifest only at the row–column cut, giving rise to a localized entanglement
valley.

\textbf{Spectral Interpretation of the Entanglement Valley.}
The entanglement entropy across a given cut is fully determined by the singular value spectrum of the corresponding bipartite decomposition.
Let $\{\lambda_i^{(\ell)}\}_{i=1}^{r_\ell}$ denote the normalized singular values
at cut $\ell$, satisfying $\sum_i (\lambda_i^{(\ell)})^2 = 1$.
The entanglement entropy is
\(S(\ell) = -\sum_i (\lambda_i^{(\ell)})^2 \log (\lambda_i^{(\ell)})^2\),
therefore directly reflects how spectral weight is distributed across modes: \textbf{flat spectra yield high \(S\), whereas concentrated spectra yield low \(S\).} To probe this mechanism, we visualize the singular value spectra at different cut
positions. For each cut $\ell$, we perform an SVD of the induced bipartite matrix, sort the singular values in descending order, and normalize them by their $\ell_2$ norm. We then plot the normalized singular values $\lambda_i^{(\ell)}$ against their sorted index. See FIG.~\ref{fig:singular_dist_combined_q_v_first_last}. Comparing these spectra across cuts, especially between boundary and middle regions, reveals a pronounced suppression of intermediate and small singular values near the mid-cut, providing a direct empirical spectral explanation for the observed entropy dip.

\textbf{Measuring Distance from Random Matrix Behavior.} To better understand the statistical properties of these singular value distributions and distinguish them from random matrix behavior, we compare them against theoretical predictions from random matrix theory. Specifically, we consider the \emph{Marchenko-Pastur (MP) distribution}~\cite{marvcenko1967distribution}, which describes the limiting eigenvalue distribution of large random matrices. For a random matrix $R \in \mathbb{R}^{d_{\min} \times d_{\max}}$ with i.i.d.\ entries of zero mean and finite variance, where $d_{\min} \le d_{\max}$ and $c = d_{\min}/d_{\max}$, the normalized eigenvalues $x_i = d_{\min} \cdot \lambda_i$ of the sample covariance matrix $RR^\top/d_{\max}$ converge to the MP distribution as the dimensions grow. Here, $\lambda_i$ denotes the normalized eigenvalues (or singular values squared) of the reduced density matrix, and the scaling $x_i = d_{\min} \cdot \lambda_i$ is the standard normalization in MP theory that makes the distribution shape depend only on the aspect ratio $c$ rather than the absolute dimensions, ensuring the theoretical results are scale-invariant. The MP distribution $\rho(x)$ represents the probability density function describing how these scaled eigenvalues $x_i$ are distributed. The MP distribution provides a natural baseline for understanding how the singular value spectra of our MPS-decomposed matrices deviate from purely random behavior. FIG.~\ref{fig:mp_baseline_theoretical} shows the theoretical MP density curves $\rho(x)$ for different values of $c$, where the vertical axis represents the probability density and the horizontal axis represents the normalized eigenvalue $x = d_{\min} \cdot \lambda$. The figure illustrates how the distribution shape changes from highly concentrated (small $c$) to more spread out (large $c$, approaching $c=1$). FIG.~\ref{fig:mp_comparison_combined_q_v_random} compares the singular value distributions from our MPS-decomposed weight matrices across different training steps with the theoretical Marchenko-Pastur distribution. In this figure, we plot the empirical density of normalized eigenvalues $\lambda$ (from the reduced density matrices at different cuts) against the theoretical MP density $\rho(x)$, where $x = d_{\min} \cdot \lambda$. A key observation is that, while the entanglement entropy in regions far from the mid-cut follows a volume law and can reach the theoretical maximum (see FIG.~\ref{fig:lora_wq_wv}), the internal structure of these matrices is not just that of a fully random matrix. The empirical distributions from $W_Q$ and $W_V$ exhibit deviations from the MP baseline, particularly in the tail regions and the overall shape of the spectrum. This indicates that despite achieving high entanglement entropy, the learned weight matrices have acquired structured correlations that distinguish them from purely random configurations. We include the random matrix distribution as a baseline reference to highlight these structural differences, demonstrating that the volume-law scaling observed in the entanglement profile does not imply random matrix statistics, but rather reflects learned structured patterns encoded in the weight matrices.

\begin{figure*}[htbp]
    \centering
    \includegraphics[width=\linewidth]{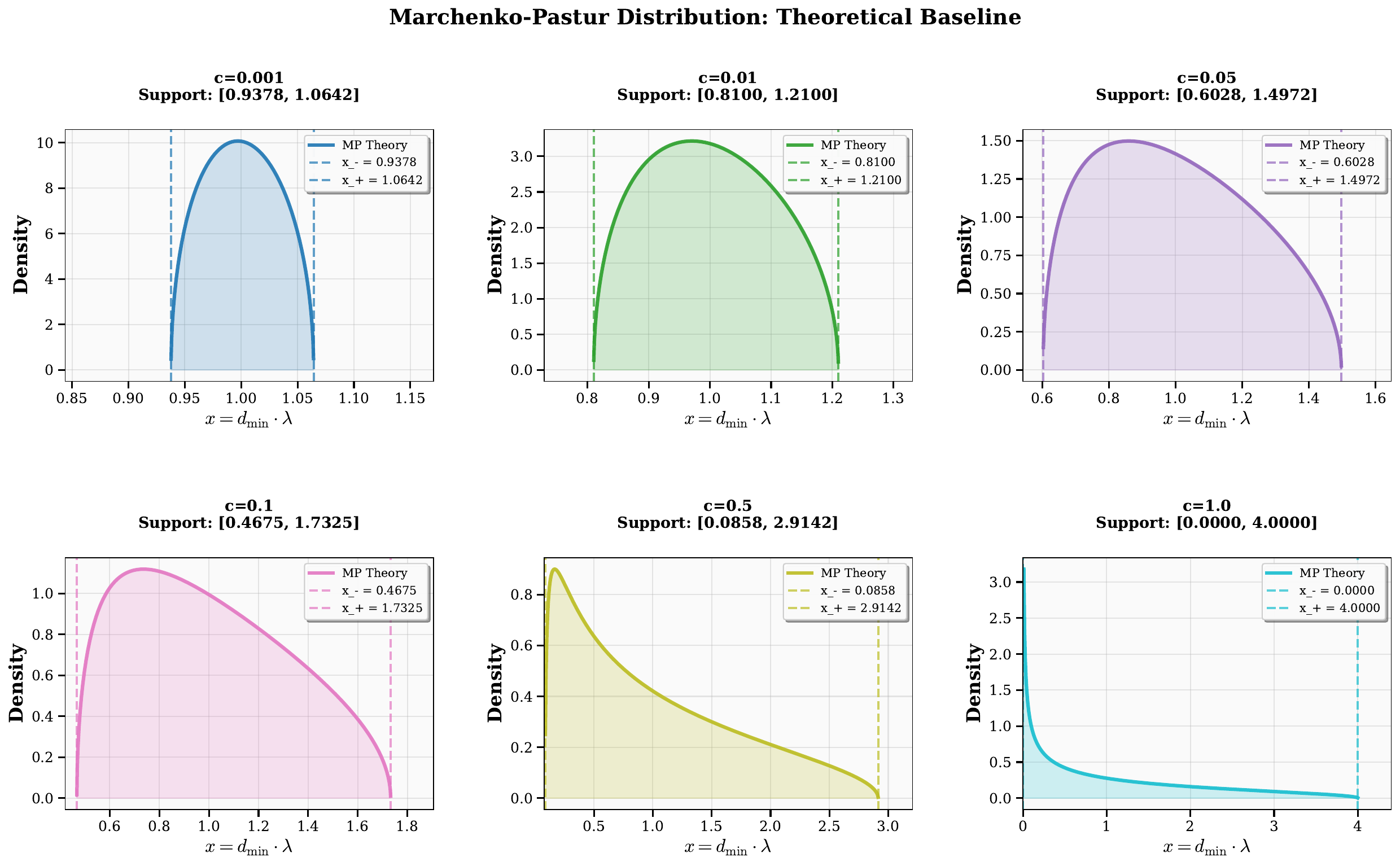}
    \caption{Marchenko-Pastur distribution: theoretical baseline. The figure shows the theoretical density curves for different values of the parameter $c = d_{\min}/d_{\max}$, where $d_{\min}$ and $d_{\max}$ are the smaller and larger dimensions of a random matrix, respectively. The MP distribution describes the limiting eigenvalue distribution of large random matrices with i.i.d.\ entries, providing a reference for comparing the singular value spectra observed in our MPS-decomposed weight projection matrices.}
    \label{fig:mp_baseline_theoretical}
\end{figure*}

\begin{figure*}[htbp]
    \centering
    \includegraphics[width=\linewidth]{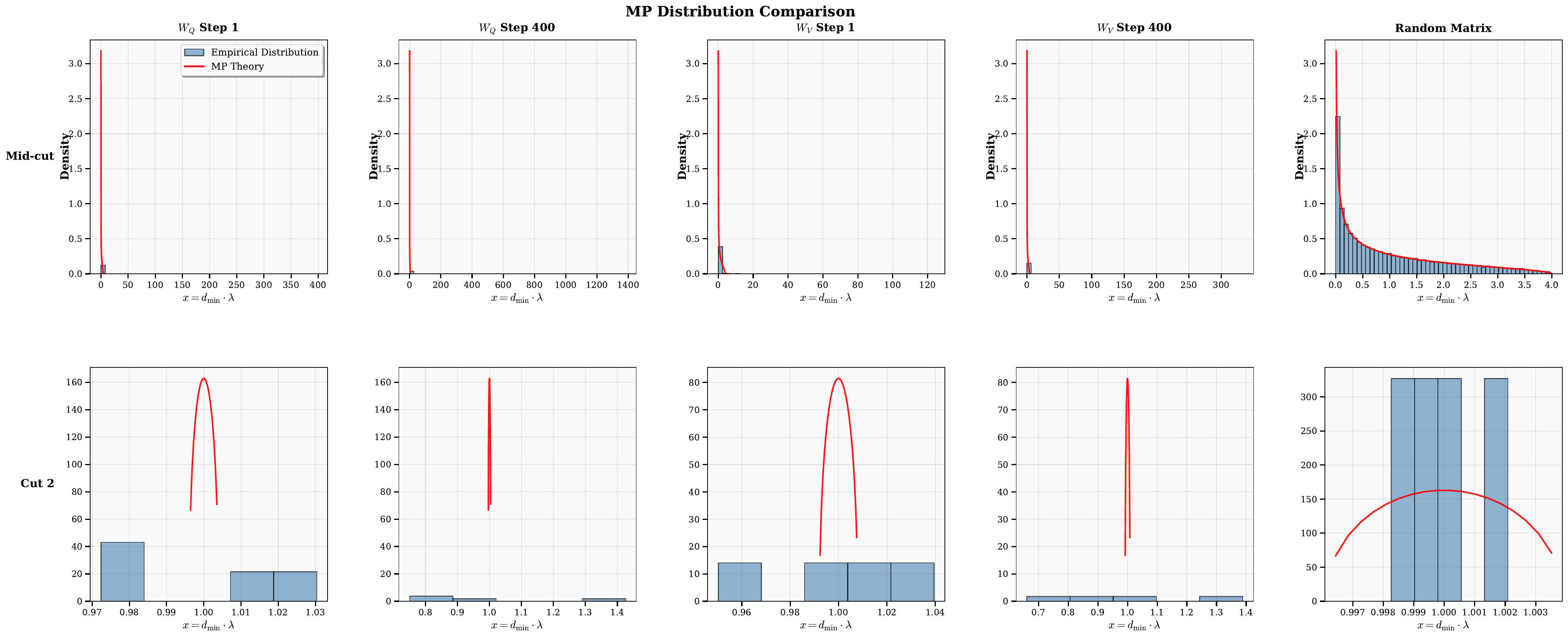}
    \caption{Comparison of normalized eigenvalue distributions from reduced density matrices at different cuts in MPS-decomposed weight projection matrices with the theoretical Marchenko-Pastur distribution. The figure shows empirical density distributions of normalized eigenvalues $\lambda$ (from reduced density matrices at different cuts) from $\Delta W_Q$ and $\Delta W_V$ matrices, plotted against the theoretical MP density $\rho(x)$ where $x = d_{\min} \cdot \lambda$, illustrating how the observed eigenvalue spectra deviate from random matrix behavior.}
\label{fig:mp_comparison_combined_q_v_random}
\end{figure*}

\subsection{Limits of the No-Hair Property Under Extreme Scaling Coefficients}

In the main text, we demonstrate that the attention matrices exhibit a no-hair property, where the entanglement entropy remains robust to variations in the scaling coefficient $\alpha$ within a moderate range (see FIG.~\ref{fig:attention_entropy_normalized_by_alpha_all_heads} and FIG.~\ref{fig:MPS_attention_S_alpha_8_16_32_64}). However, this robustness may break down when $\alpha$ deviates substantially from typical values used in practice. To probe the limits of this no-hair behavior, we investigate the entanglement structures of attention matrices under largely deviated scaling coefficients $\alpha \in \{16, 512, 4096\}$, extending beyond the regime explored in the main text. FIG.~\ref{fig:MPS_attention_entropy_normalized_by_alpha_all_heads} shows the normalized entanglement entropy $S_A/\log(\chi)$ across training steps for several attention heads under these extreme $\alpha$ values. Compared to the moderate regime where $\alpha \in \{8, 16, 32, 64\}$ (FIG.~\ref{fig:MPS_attention_S_alpha_8_16_32_64}), we observe that when $\alpha$ increases to $512$, the entanglement curves remain approximately invariant, maintaining the no-hair-like property observed in the main text. However, when $\alpha$ scales further to $4096$, pronounced differences emerge: for several attention heads, $\alpha=4096$ produces substantially higher entanglement entropy $S$, with a sudden increase during early training steps, while $\alpha=16$ and $\alpha=512$ remain almost flat throughout training. These results highlight that the no-hair property observed in the main text holds within a certain range of $\alpha$ values, but breaks down under extreme scaling. Nevertheless, compared to the dramatic variations in entanglement structures observed in the token embedding space (see FIG.~\ref{fig:entropy_evolution_by_alpha} and FIG.~\ref{fig:large_alpha_WQ_WV}), where the embedding space exhibits substantial sensitivity to $\alpha$ with entanglement profiles varying significantly across different scaling coefficients, we still consider the attention behavior as exhibiting a no-hair-like property. 

\begin{figure}[H]
    \centering
\includegraphics[width=0.8\linewidth]{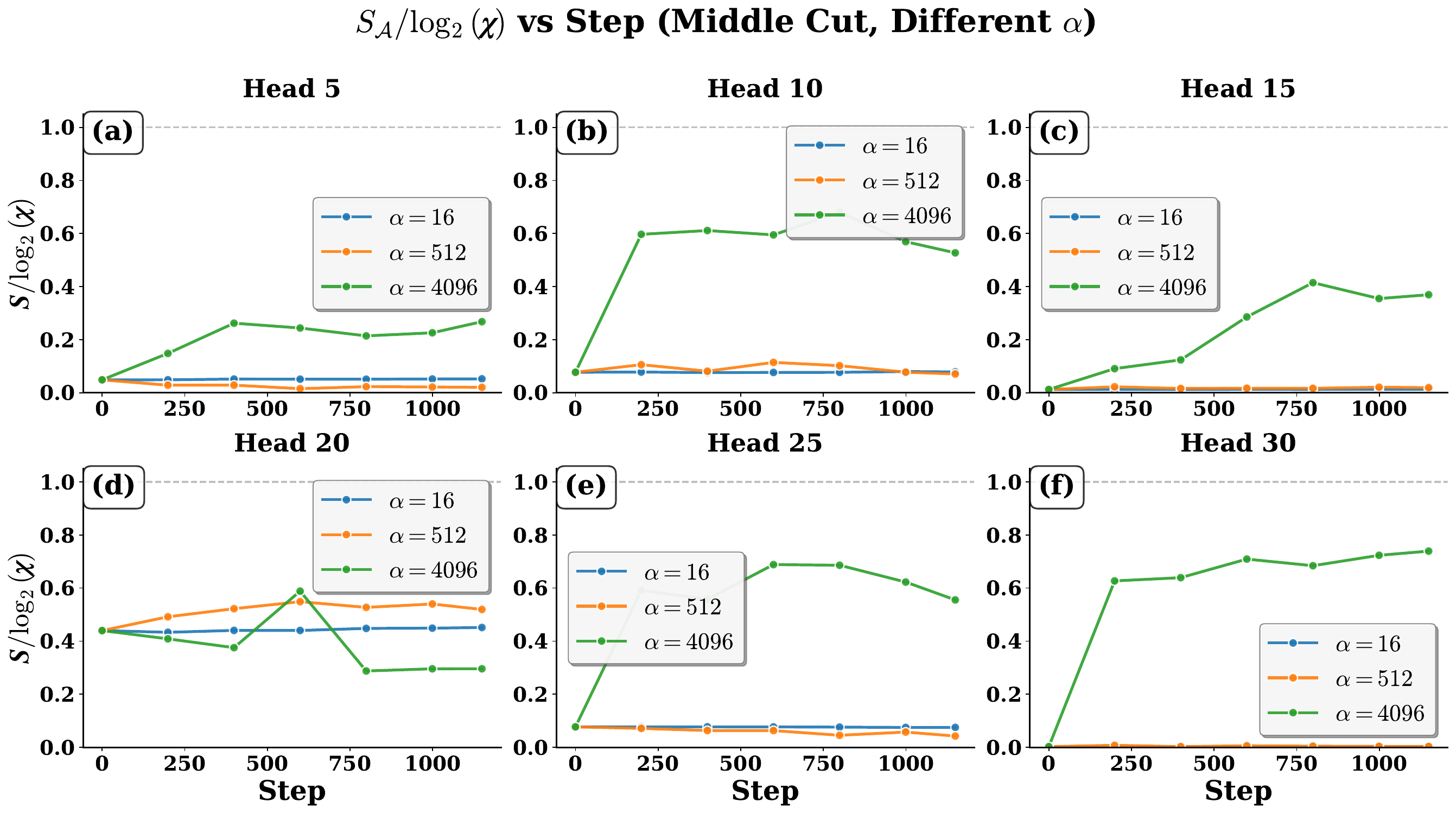}
\caption{
Normalized entanglement entropy $S_A/\log(\chi)$ for attention heads under extreme scaling coefficients $\alpha\in\{16,512,4096\}$, probing the limits of the no-hair property. Compared to the moderate regime (FIG.~\ref{fig:MPS_attention_S_alpha_8_16_32_64}), when $\alpha=512$ the entanglement curves remain approximately invariant, maintaining the no-hair-like property. However, when $\alpha$ scales further to $4096$, pronounced differences emerge: for several attention heads, $\alpha=4096$ produces substantially higher entanglement entropy $S$, with a sudden increase during early training steps, while $\alpha=16$ and $\alpha=512$ remain almost flat throughout training. These results demonstrate that the no-hair property holds within a certain range of $\alpha$ values but breaks down under extreme scaling.
}
\label{fig:MPS_attention_entropy_normalized_by_alpha_all_heads}
\end{figure}

\subsection{Artificial Entanglement Profiling: LLaMA-3.2-1B, OpenThoughts3 Dataset}

We extend our empirical analysis to the OpenThoughts3 dataset. The results verify that the entanglement structures and no-hair phenomena observed in the Tulu3 dataset (Sec.~\ref{sec:Entanglement Structure in FFT and LoRA}) remain consistent across different data distributions.

\textbf{Artificial entanglement profiling in Attention matrix.}
FIG.~\ref{3.2_openthought:attention_entropy_evolution_head30},~\ref{3.2_openthought:attention_entropy_normalized_by_alpha_all_heads}, and~\ref{3.2_openthought:attention_entropy_normalized_evolution_all_heads} demonstrate that the attention matrices on OpenThoughts3 exhibit the same approximate area-law scaling with logarithmic correction as observed on Tulu3. The normalized entanglement entropy $S_A/\log(\chi)$ remains significantly below the theoretical maximum and robust to variations in the scaling coefficient $\alpha$, confirming that the no-hair property holds across datasets. Similarly, FIG.~\ref{3.2_openthought:attention_output_entropy_combined} shows that the output operator $\mathcal{O} = XX^{\top}$ maintains entanglement significantly below the theoretical maximum, consistent with the filtering behavior observed in the main text.

\begin{figure}[H]
    \centering
    \includegraphics[width=0.5\linewidth]{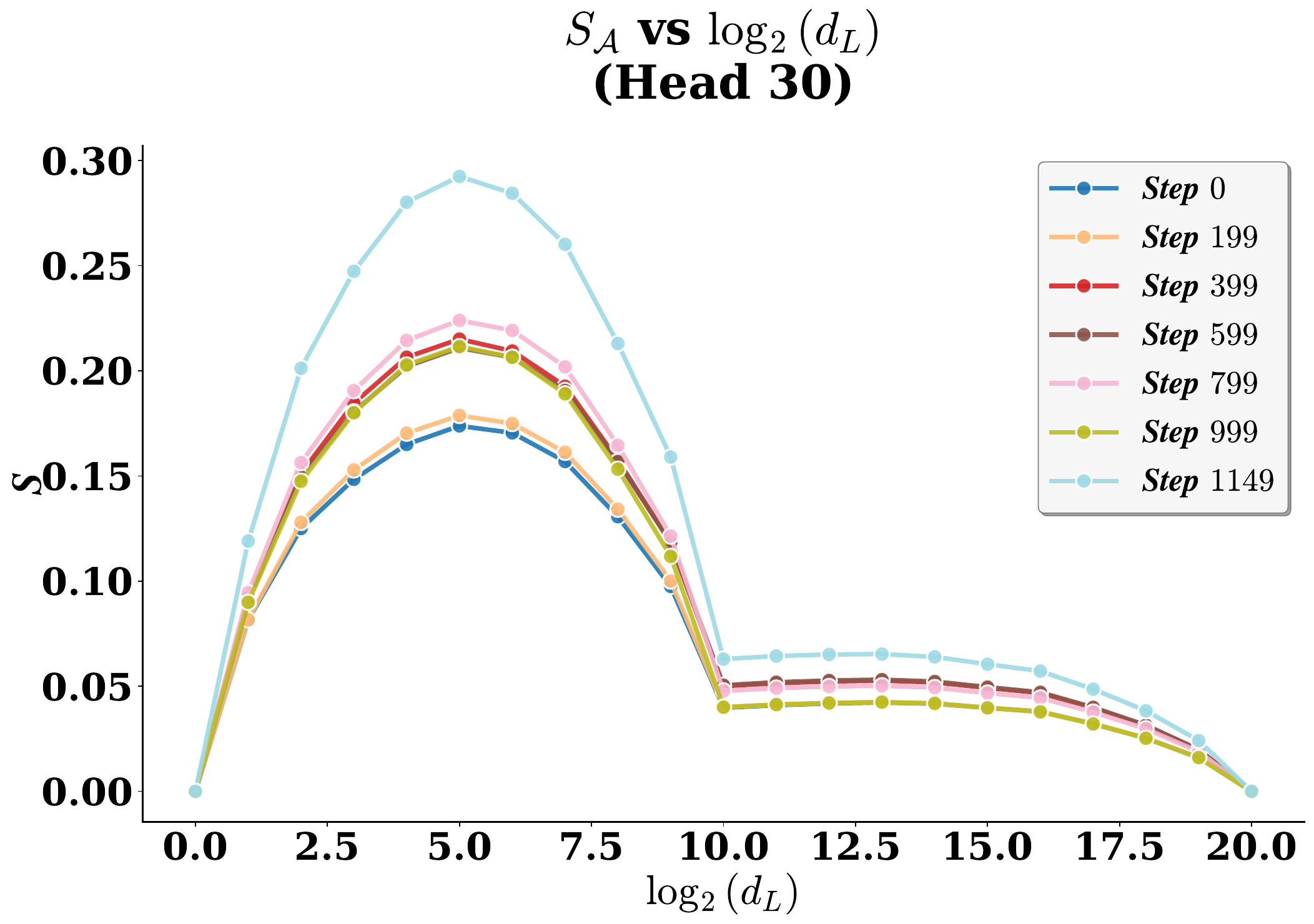}
    \caption{Entanglement entropy $S_A$ of the attention matrix (Head 30) with respect to the bi-partition position across training steps for LLaMA-3.2-1B fine-tuned on OpenThoughts3 dataset.}
    \label{3.2_openthought:attention_entropy_evolution_head30}
\end{figure}

\begin{figure}[H]
    \centering
    \includegraphics[width=0.8\linewidth]{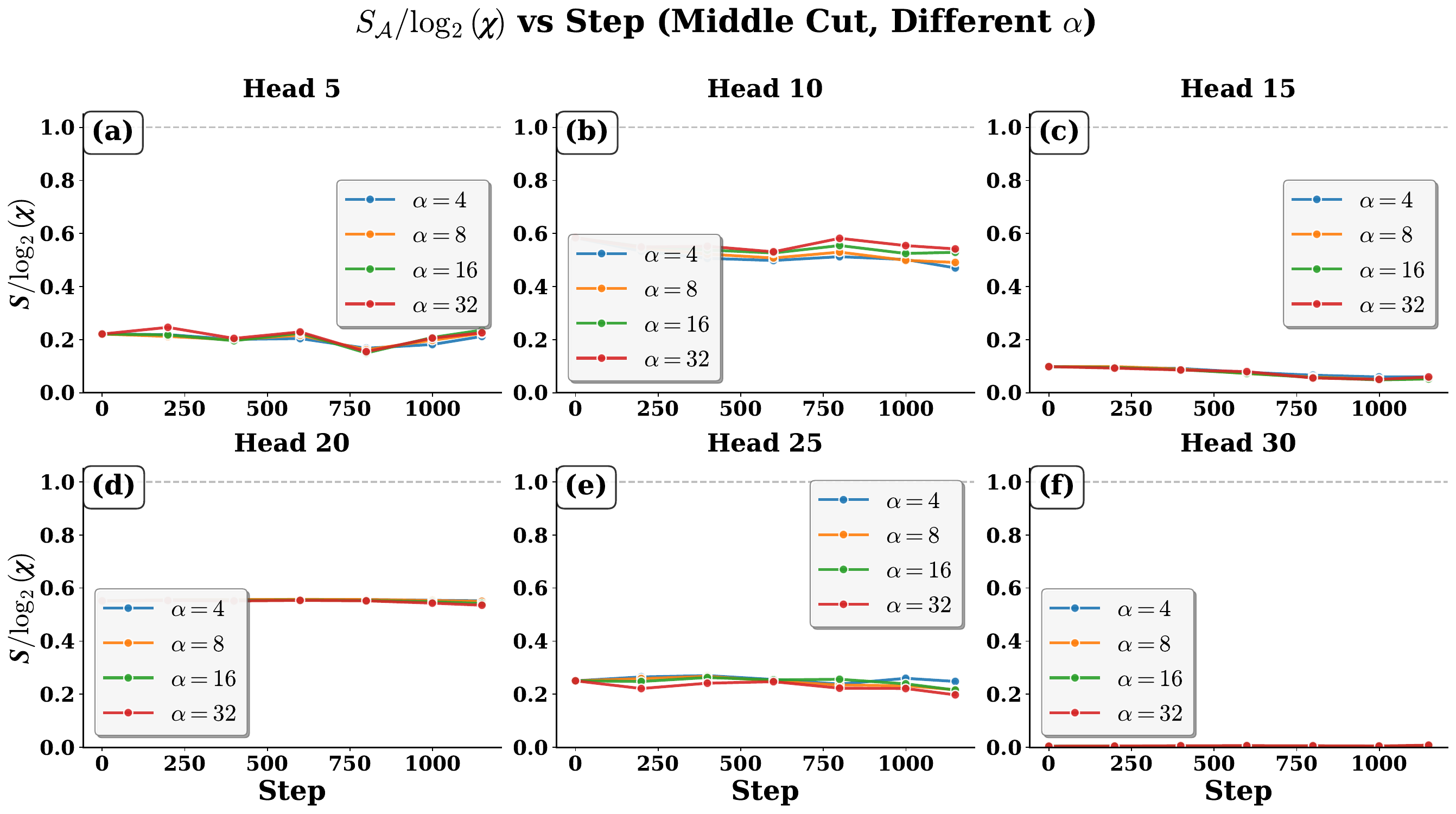}
    \caption{Entanglement entropy $S_A/\log(\chi)$ across training steps for several attention heads under less-deviated scaling coefficients $\alpha\in\{4, 8,16,32\}$ for LLaMA-3.2-1B fine-tuned on OpenThoughts3 dataset. See FIG.~\ref{fig:attention_entropy_normalized_by_alpha_all_heads} for details.}
    \label{3.2_openthought:attention_entropy_normalized_by_alpha_all_heads}
\end{figure}

\begin{figure}[H]
    \centering
    \includegraphics[width=\linewidth]{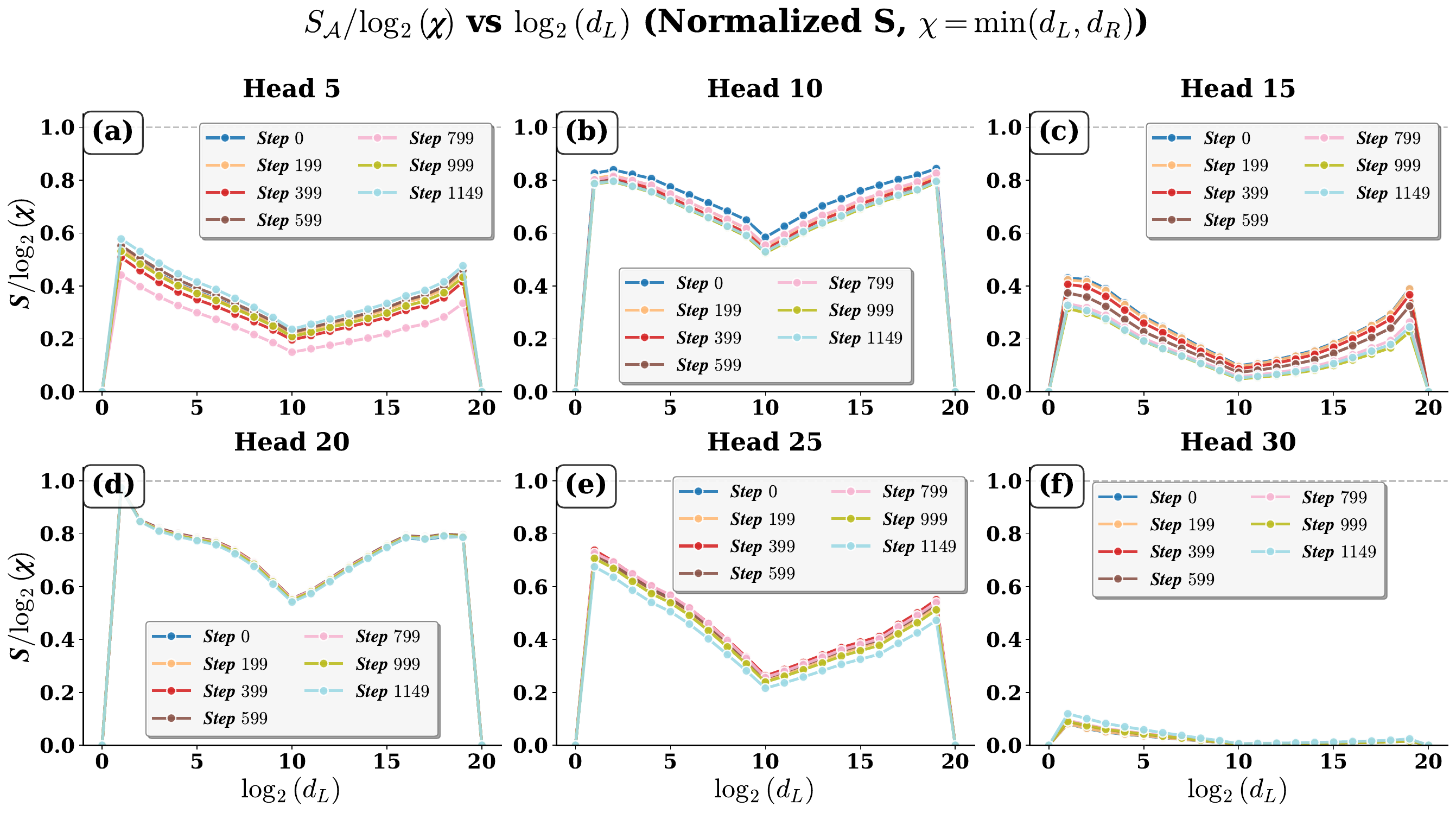}
    \caption{Normalized entanglement entropy $S_A/\log(\chi)$ of six representative attention heads across training steps for LLaMA-3.2-1B fine-tuned on OpenThoughts3 dataset, where $\chi=\min(d_L,d_R)$ denotes the maximal entanglement capacity allowed by the bi-partition.}
    \label{3.2_openthought:attention_entropy_normalized_evolution_all_heads}
\end{figure}

\begin{figure}[H]
    \centering
    \includegraphics[width=\linewidth]{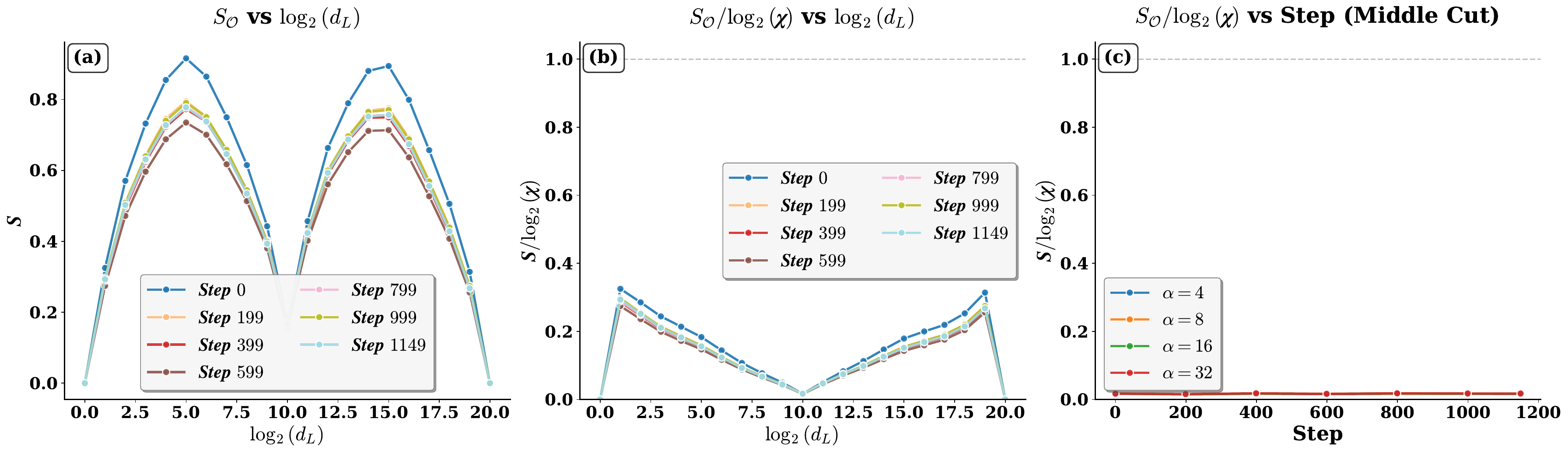}
    \caption{Entanglement entropy $S_{\mathcal{O}}$ of the output operator $\mathcal{O} = XX^{\top}$ across bi-partition positions and training steps for LLaMA-3.2-1B fine-tuned on OpenThoughts3 dataset.}
    \label{3.2_openthought:attention_output_entropy_combined}
\end{figure}

\textbf{Token embedding space artificial entanglement profiling: volume-law and entanglement valley.}
FIG.~\ref{3.2_openthought:entropy_vs_cut_position_by_rank_lr001_alpha16} shows that MPS adaptation on OpenThoughts3 exhibits the same characteristic entanglement valley structure as observed in the main text, confirming that the volume-law behavior in embedding space is robust across datasets.

\begin{figure}[H]
    \centering
    \includegraphics[width=\linewidth]{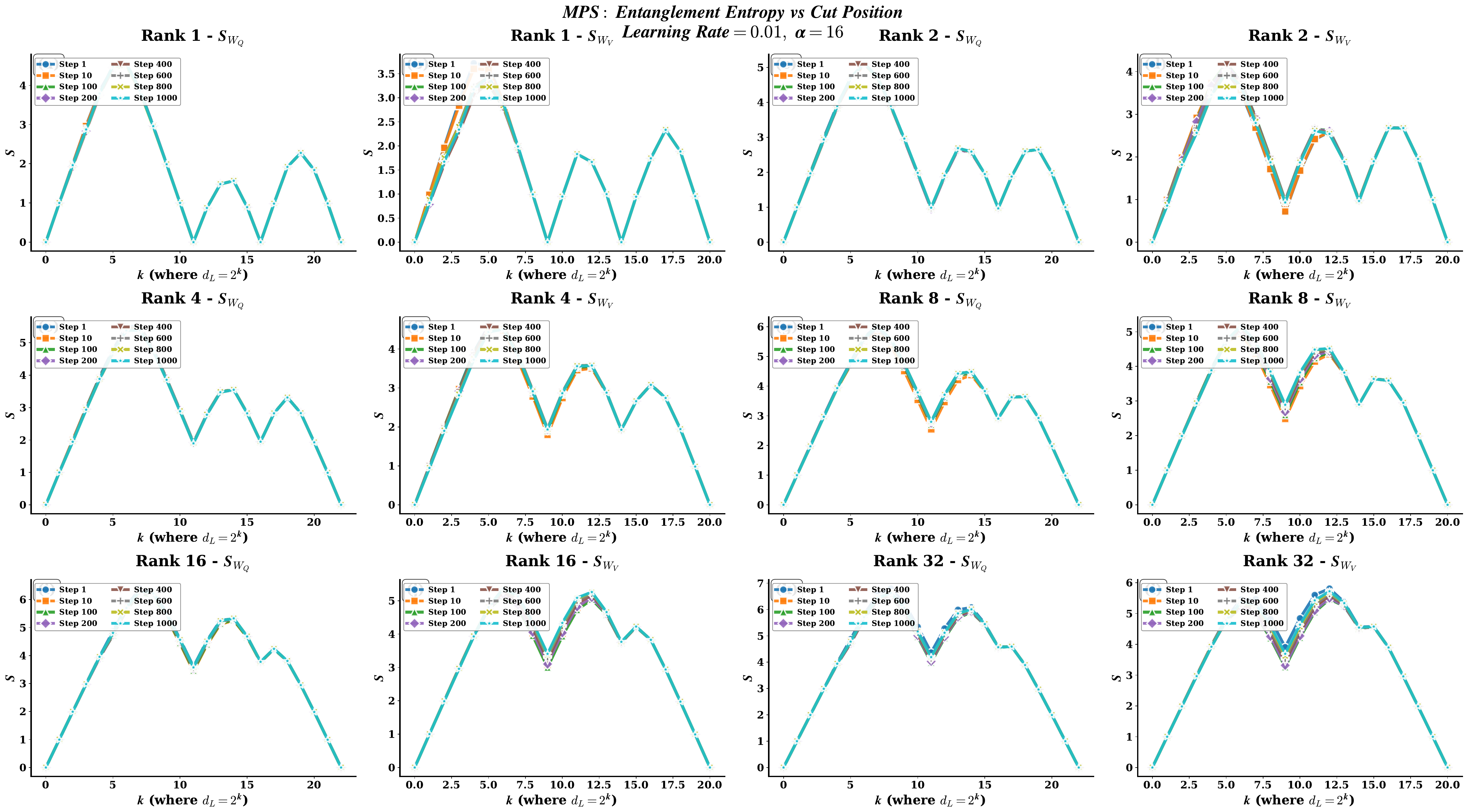}
    \caption{Artificial entanglement profiling of the MPS adaptation of $\Delta W$ for ranks $r \in \{1, 2, 4, 8, 16, 32\}$ with $\alpha = 16$ across training steps for LLaMA-3.2-1B fine-tuned on OpenThoughts3 dataset.}
    \label{3.2_openthought:entropy_vs_cut_position_by_rank_lr001_alpha16}
\end{figure}


\begin{figure}[H]
    \centering
    \includegraphics[width=0.4\linewidth]{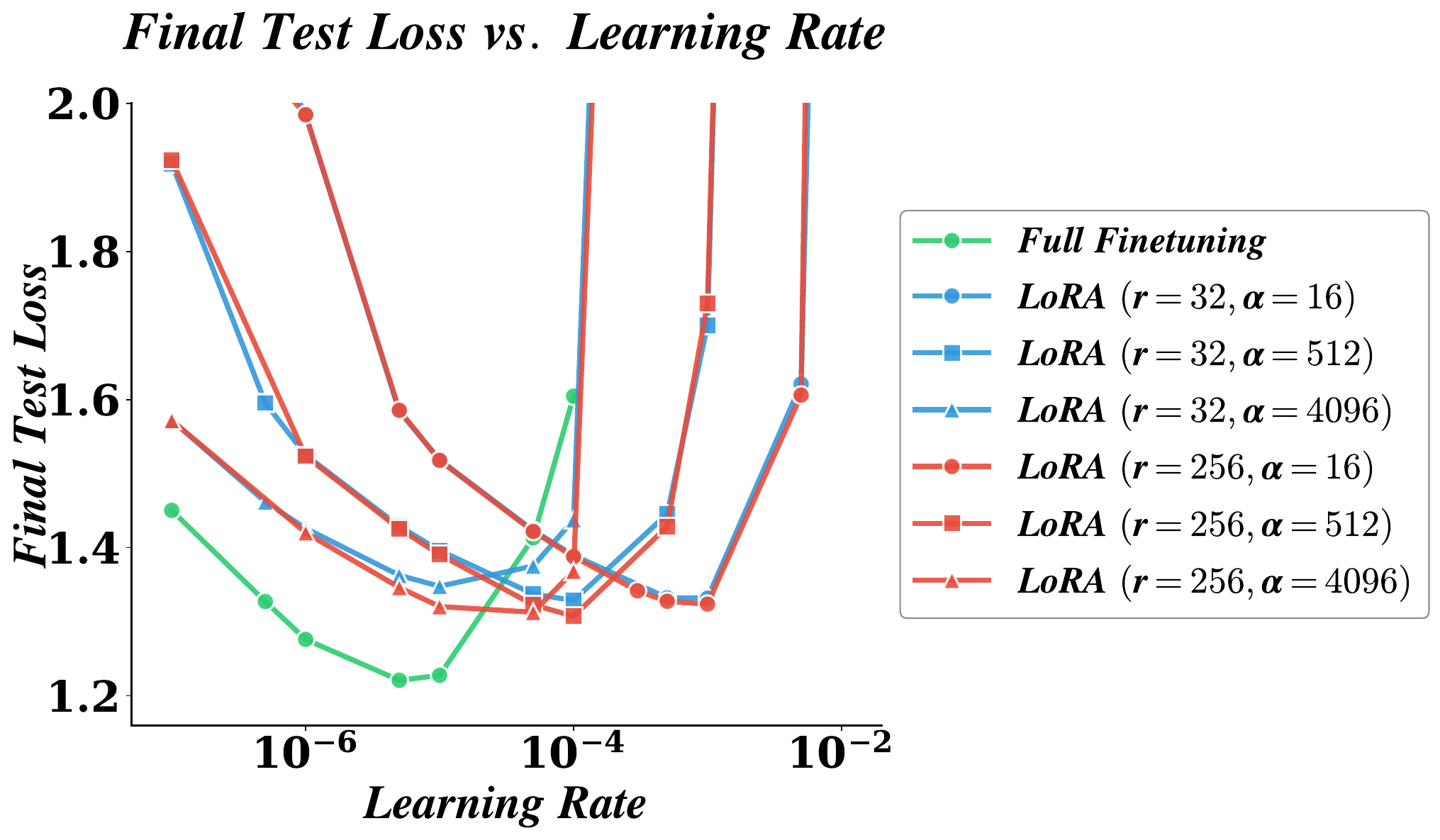}
    \caption{Final test loss with respect to learning rate under different LoRA configurations for LLaMA-3.2-1B fine-tuned on OpenThoughts3 dataset, including a comparison with full fine-tuning.}
\label{3.2_openthought:final_test_loss_vs_lr_combined_20251124_20251125}
\end{figure}
\vspace{-5pt}
\begin{figure}[H]
    \centering
    \includegraphics[width=0.4\linewidth]{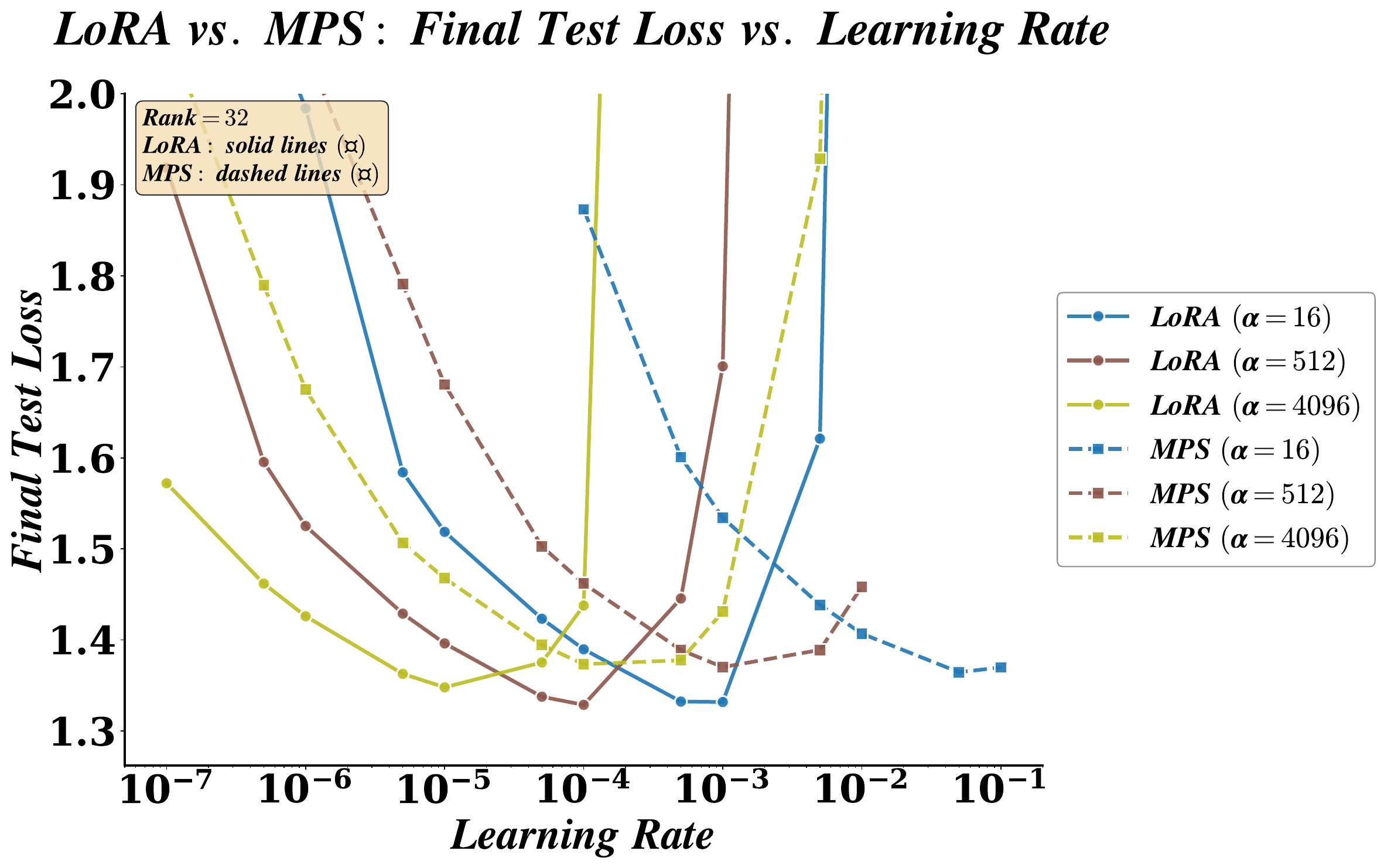}
    \caption{Final test loss of LoRA and MPS adaptation across multiple learning rates and scaling coefficients $\alpha$ for LLaMA-3.2-1B fine-tuned on OpenThoughts3 dataset.}
\label{3.2_openthought:lora_vs_mps_final_test_loss_vs_lr_rank32_combined_20251125_20251126}
\end{figure}

\textbf{Optimization landscape comparison.}
FIG.~\ref{3.2_openthought:final_test_loss_vs_lr_combined_20251124_20251125} and~\ref{3.2_openthought:lora_vs_mps_final_test_loss_vs_lr_rank32_combined_20251125_20251126} demonstrate that the optimization dynamics and relative performance of LoRA and MPS adaptation on OpenThoughts3 mirror the observations on Tulu3. The optimal learning rate regions and the shift in optimal learning rates with $\alpha$ follow the same patterns as described in Sec.~\ref{sec:Entanglement Structure in FFT and LoRA}, confirming the consistency of these phenomena across datasets.


\subsection{Artificial Entanglement Profiling: LLaMA-3.1-8B, Tulu3 Dataset}

We extend our analysis to the larger LLaMA-3.1-8B model fine-tuned on the Tulu3 dataset to validate the universality of our findings across model scales. The results demonstrate that the entanglement structures and no-hair phenomena observed in the 1B model (Sec.~\ref{sec:Entanglement Structure in FFT and LoRA}) persist consistently in larger architectures.

\begin{figure}[H]
    \centering
    \includegraphics[width=0.5\linewidth]{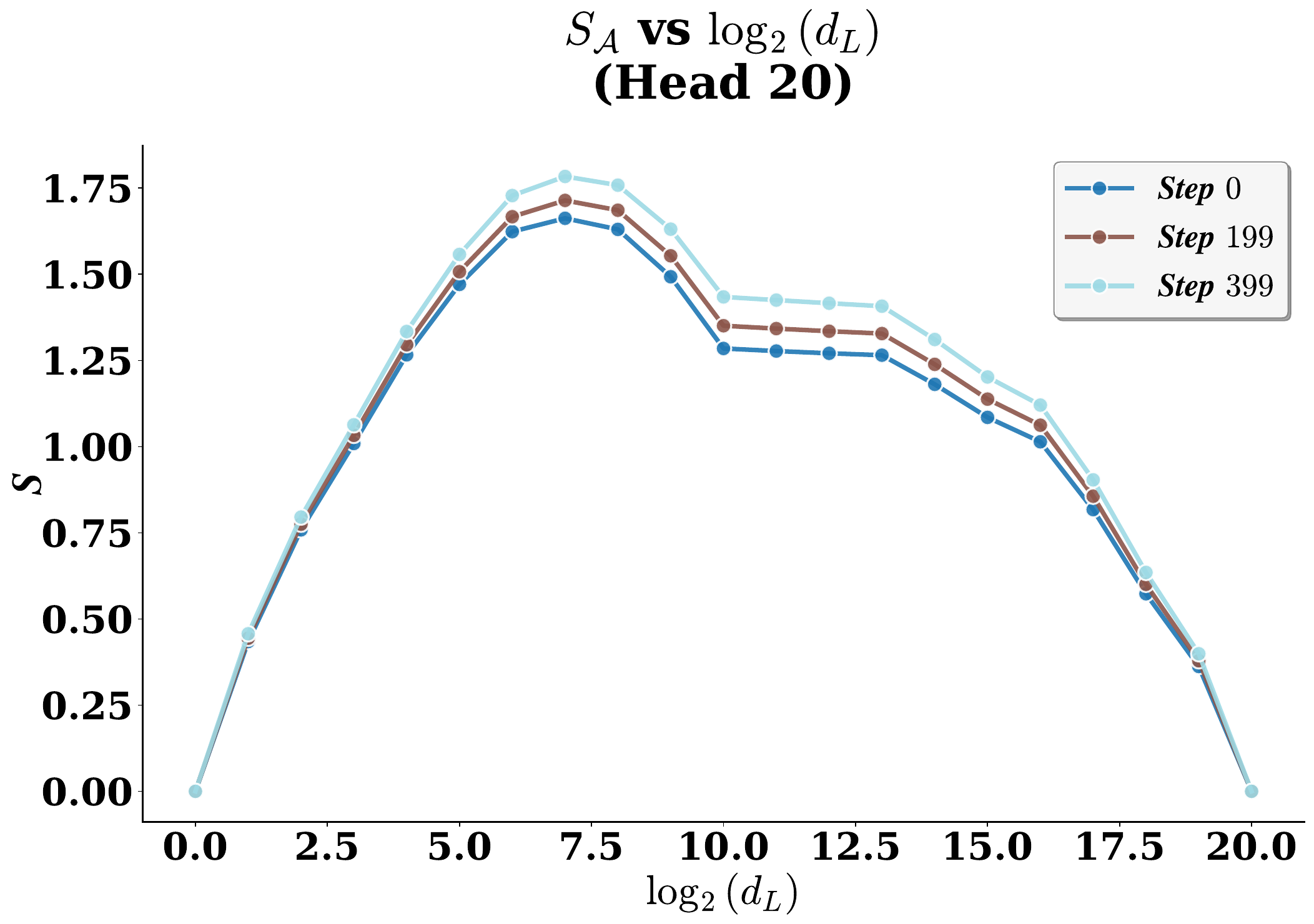}
    \caption{Entanglement entropy $S_A$ of the attention matrix (Head 20) with respect to the bi-partition position across training steps for LLaMA-3.1-8B fine-tuned on Tulu3 dataset.}
    \label{8B_tutu:attention_entropy_evolution_head20}
\end{figure}

\begin{figure}[H]
    \centering
    \includegraphics[width=0.8\linewidth]{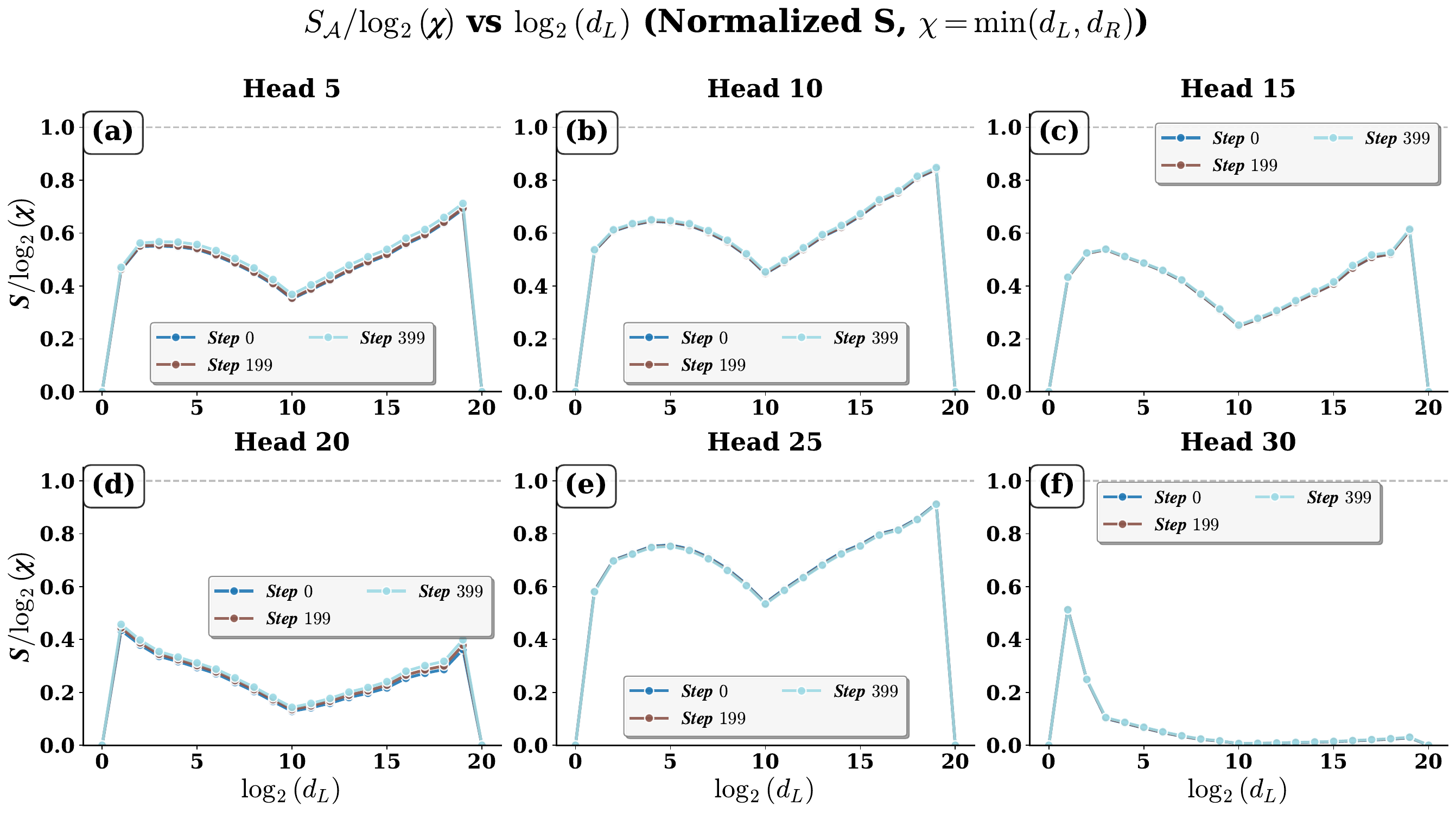}
    \caption{Normalized entanglement entropy $S_A/\log(\chi)$ of six representative attention heads across training steps for LLaMA-3.1-8B fine-tuned on Tulu3 dataset, where $\chi=\min(d_L,d_R)$ denotes the maximal entanglement capacity allowed by the bi-partition.}
    \label{8B_tutu:attention_entropy_normalized_evolution_all_heads}
\end{figure}

\begin{figure}[H]
    \centering
    \includegraphics[width=0.9\linewidth]{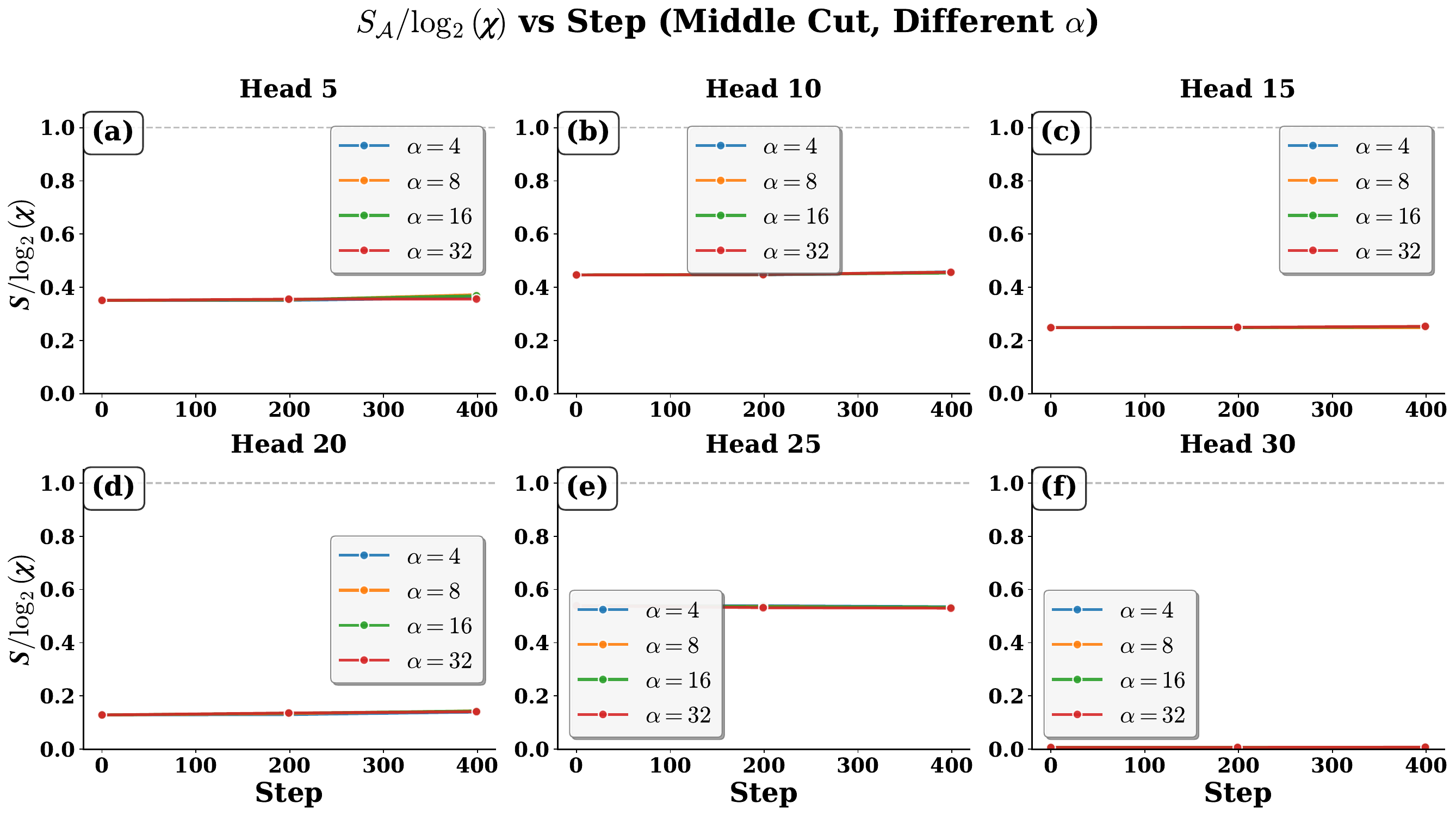}
    \caption{Normalized entanglement entropy $S_A/\log(\chi)$ across training steps for several attention heads under different scaling coefficients $\alpha$ for LLaMA-3.1-8B fine-tuned on Tulu3 dataset.}
    \label{8B_tutu:attention_entropy_normalized_by_alpha_all_heads}
\end{figure}

\textbf{Artificial entanglement profiling in Attention matrix.}
FIG.~\ref{8B_tutu:attention_entropy_evolution_head20},~\ref{8B_tutu:attention_entropy_normalized_by_alpha_all_heads}, and~\ref{8B_tutu:attention_entropy_normalized_evolution_all_heads} show that the attention matrices in the 8B model exhibit the same approximate area-law scaling with logarithmic correction as observed in the 1B model. The normalized entanglement entropy $S_A/\log(\chi)$ remains significantly below the theoretical maximum and robust to variations in the scaling coefficient $\alpha$, confirming that the no-hair property holds across model scales. Similarly, FIG.~\ref{8B_tutu:attention_output_entropy_combined} demonstrates that the output operator $\mathcal{O} = XX^{\top}$ maintains entanglement significantly below the theoretical maximum, consistent with the filtering behavior observed in smaller models.

\begin{figure}[H]
    \centering
    \includegraphics[width=\linewidth]{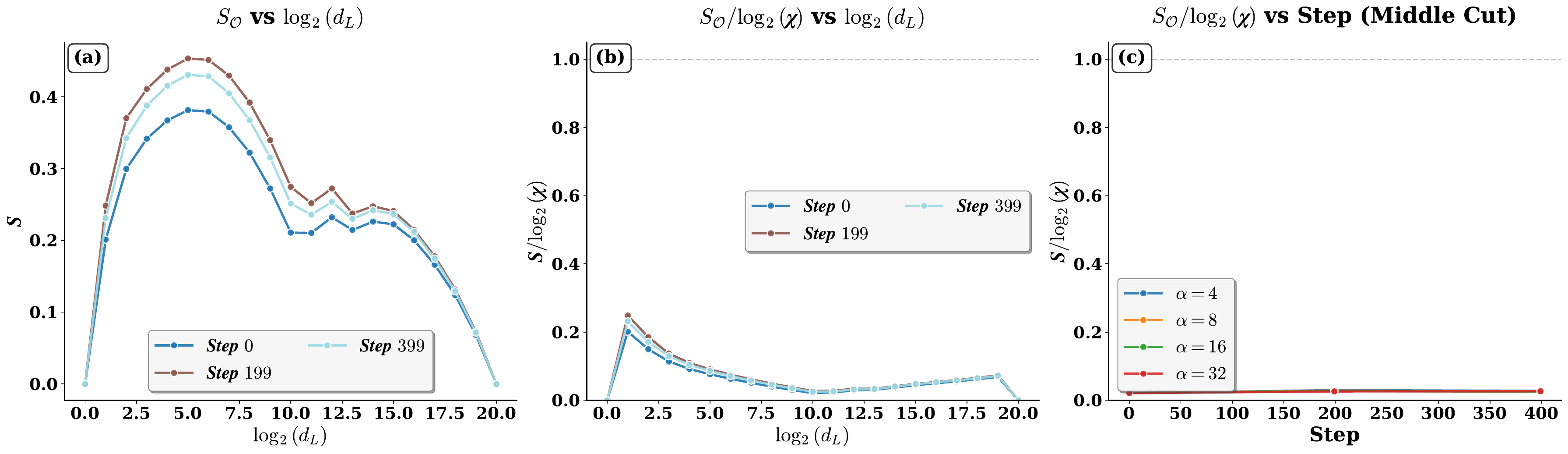}
    \caption{Entanglement entropy $S_{\mathcal{O}}$ of the output operator $\mathcal{O} = XX^{\top}$ across bi-partition positions and training steps for LLaMA-3.1-8B fine-tuned on Tulu3 dataset.}
    \label{8B_tutu:attention_output_entropy_combined}
\end{figure}

\textbf{Token embedding space artificial entanglement profiling: volume-law and entanglement valley.}
FIG.~\ref{8B_tutu:delta_w_entropy_lr5e4_rank256_alpha16_512_combined} and~\ref{8B_tutu:entropy_vs_cut_position_by_rank_lr001_alpha16} reveal that the embedding space ($W_Q$ and $W_V$) in the 8B model exhibits the same volume-law behavior with a characteristic entanglement valley as observed in the 1B model. The valley deepens during training in LoRA and shows sensitivity to the scaling coefficient $\alpha$, consistent with the findings in Sec.~\ref{sec:Entanglement Structure in FFT and LoRA}. This confirms that the volume-law entanglement profile and its hyperparameter sensitivity are fundamental properties independent of model scale.

\begin{figure}[H]
    \centering
    \includegraphics[width=0.8\linewidth]{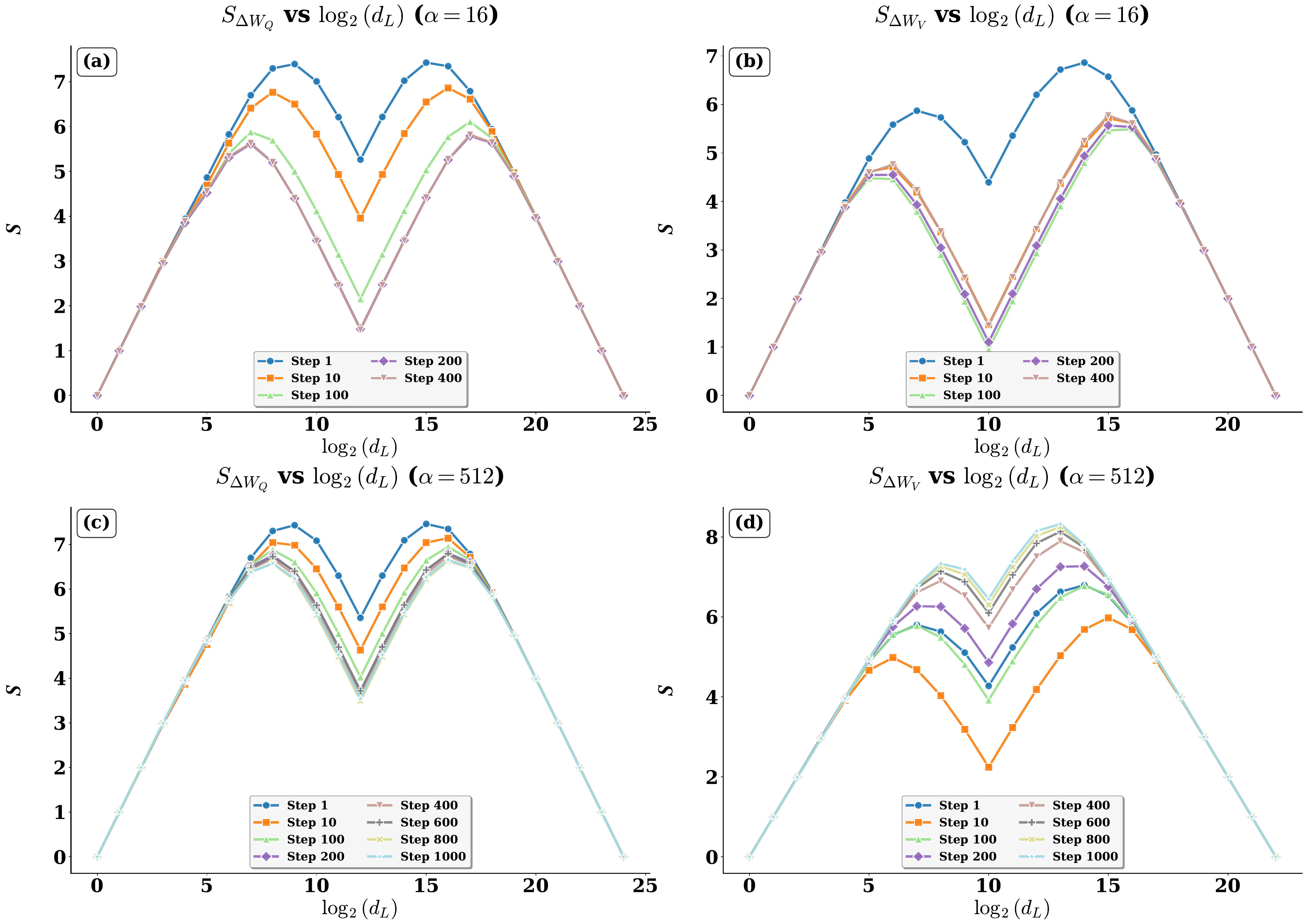}
    \caption{Artificial entanglement profiling of $\Delta W_Q$ and $\Delta W_V$ with respect to the cut position across training steps and scaling coefficients $\alpha \in \{16, 512\}$ for LLaMA-3.1-8B fine-tuned on Tulu3 dataset.}
    \label{8B_tutu:delta_w_entropy_lr5e4_rank256_alpha16_512_combined}
\end{figure}

\textbf{Optimization landscape comparison.}
FIG.~\ref{8B_tutu:lora_vs_mps_final_test_loss_vs_lr_llama3.1_8b} compares LoRA and MPS adaptation strategies for the 8B model. Both methods achieve comparable performance, with optimal learning rates shifting as $\alpha$ increases, mirroring the behavior observed in the 1B model (Sec.~\ref{sec:Entanglement Structure in FFT and LoRA}). This consistency further validates the generalizability of our theoretical framework across model scales.

\begin{figure}[H]
    \centering
    \includegraphics[width=0.5\linewidth]{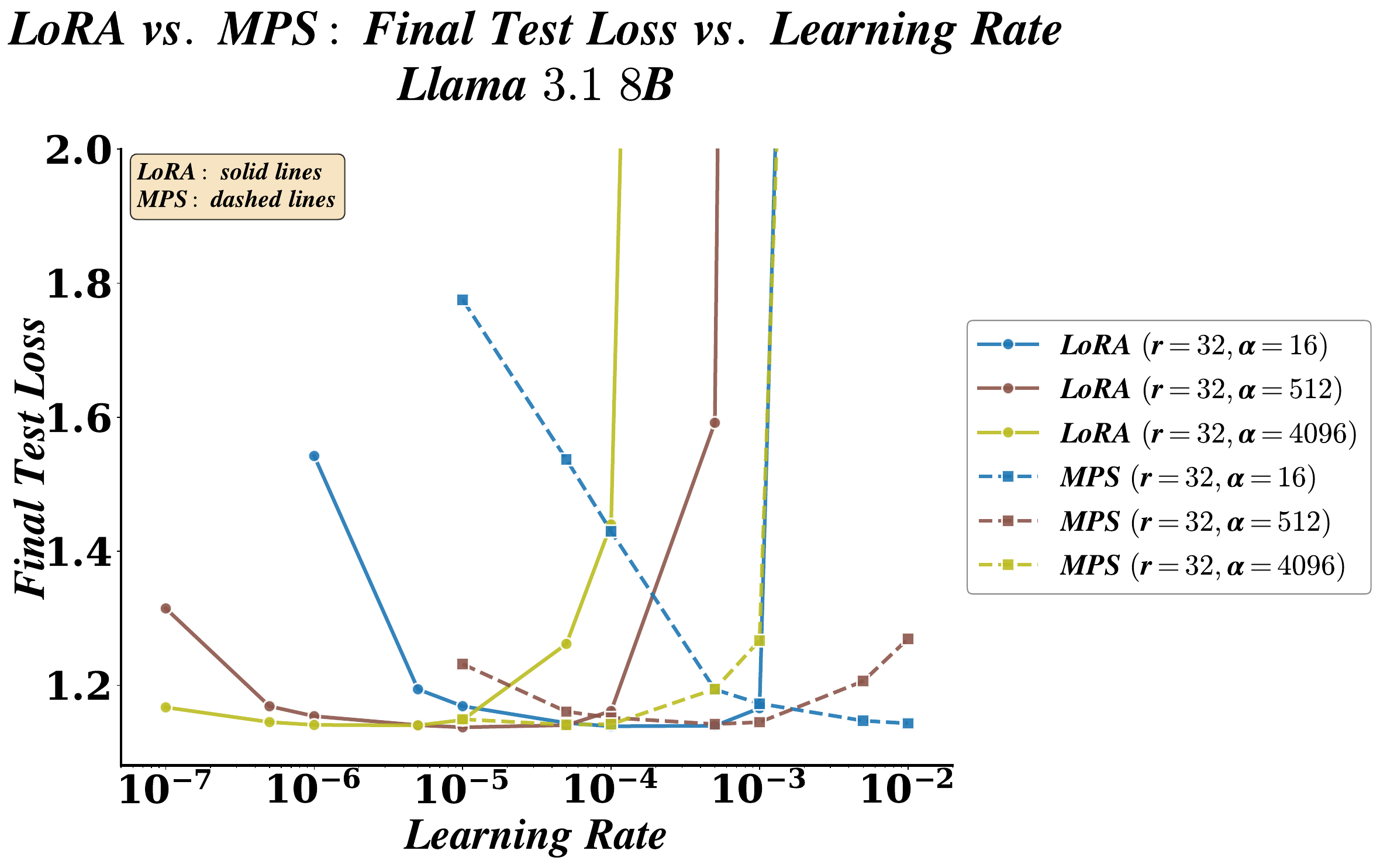}
    \caption{Final test loss of LoRA and MPS adaptation across multiple learning rates and scaling coefficients $\alpha$ for LLaMA-3.1-8B fine-tuned on Tulu3 dataset.}
    \label{8B_tutu:lora_vs_mps_final_test_loss_vs_lr_llama3.1_8b}
\end{figure}

\begin{figure}[H]
    \centering
    \includegraphics[width=\linewidth]{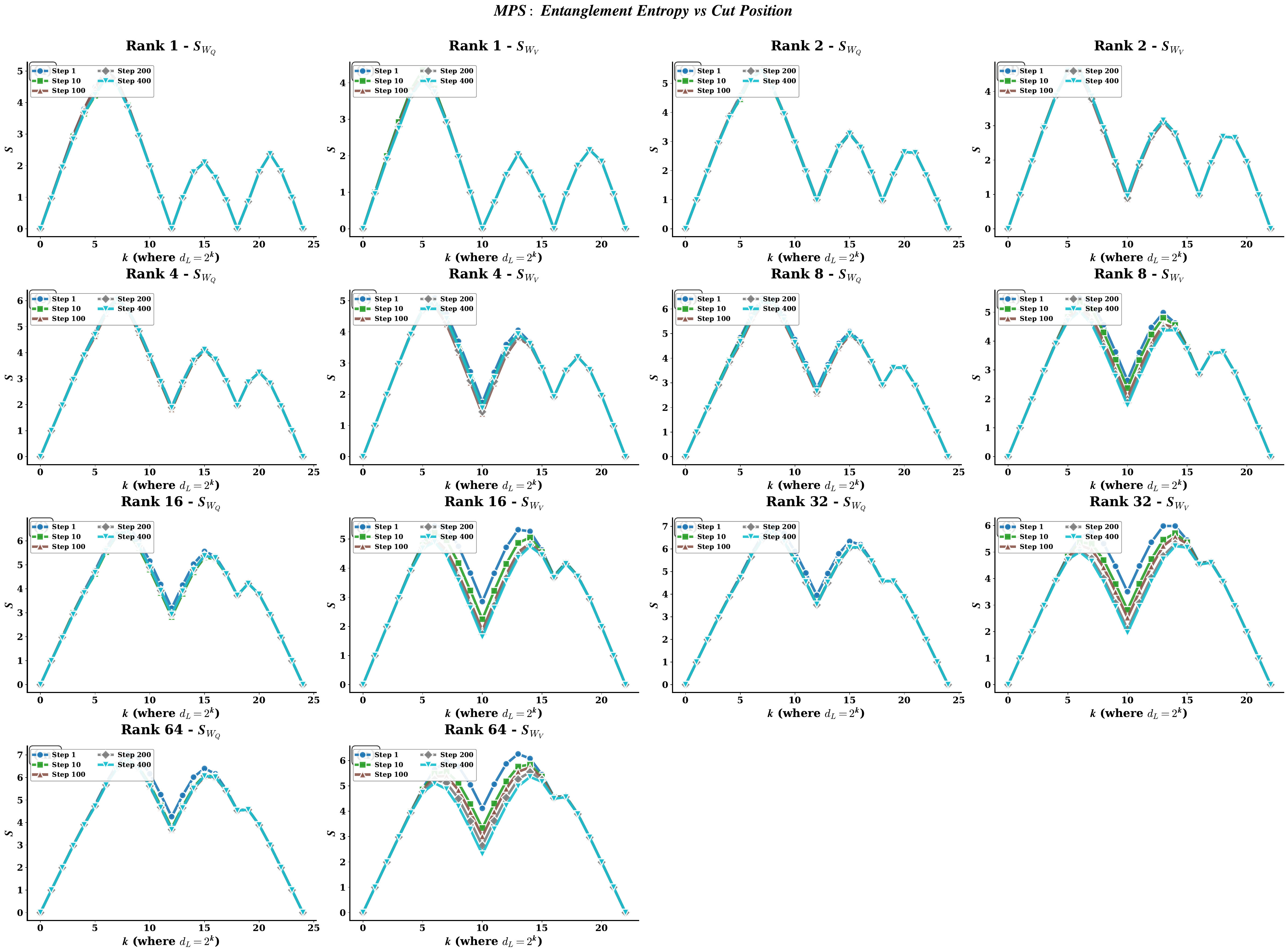}
    \caption{Artificial entanglement profiling of the MPS adaptation of $\Delta W$ for ranks $r \in \{1, 2, 4, 8, 16, 32\}$ with $\alpha = 16$ across training steps for LLaMA-3.1-8B fine-tuned on Tulu3 dataset.}
    \label{8B_tutu:entropy_vs_cut_position_by_rank_lr001_alpha16}
\end{figure}

\subsection{Artificial Entanglement Profiling: LLaMA-3.1-8B, OpenThoughts3 Dataset}

We further validate our findings by analyzing the LLaMA-3.1-8B model fine-tuned on the OpenThoughts3 dataset. The results confirm that the entanglement structures and no-hair phenomena are robust across both model scales and dataset distributions, reinforcing the universality.

\textbf{Artificial entanglement profiling in Attention matrix.}
FIG.~\ref{8B_OpenThoughts:attention_entropy_evolution_head31} and~\ref{8B_OpenThoughts:attention_entropy_normalized_by_alpha_all_heads} demonstrate that the attention matrices in the 8B model on OpenThoughts3 exhibit the same area-law scaling with logarithmic correction and no-hair property as observed in both the 1B model (Sec.~\ref{sec:Entanglement Structure in FFT and LoRA}) and the 8B model on Tulu3. The normalized entanglement entropy remains significantly below the theoretical maximum and robust to variations in $\alpha$, confirming that these phenomena are universal across different datasets and model scales.

\begin{figure}[H]
    \centering
    \includegraphics[width=0.6\linewidth]{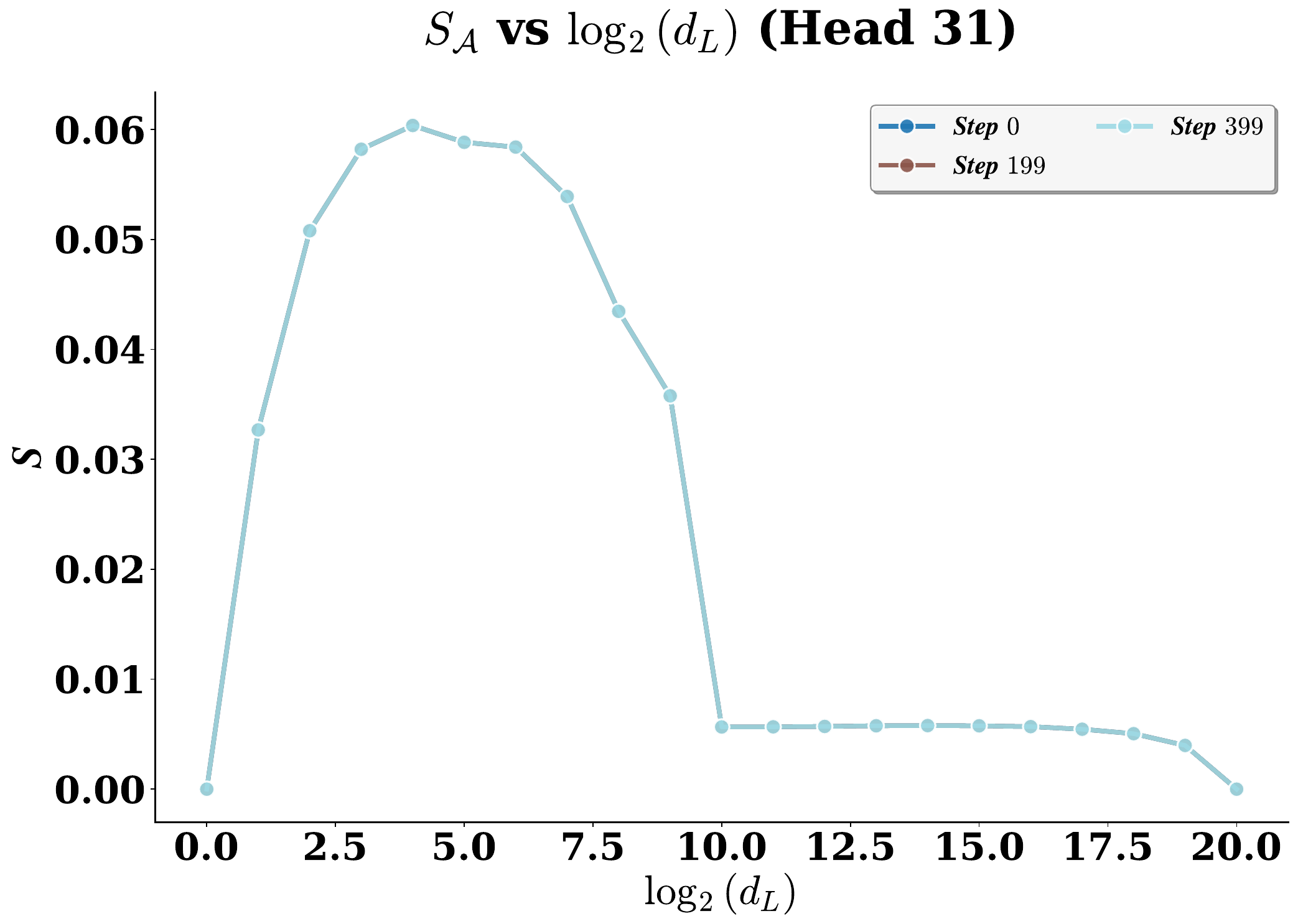}
    \caption{Entanglement entropy $S_A$ of the attention matrix (Head 31) with respect to the bi-partition position across training steps for LLaMA-3.1-8B fine-tuned on OpenThoughts3 dataset.}
    \label{8B_OpenThoughts:attention_entropy_evolution_head31}
\end{figure}

\begin{figure}[H]
    \centering
    \includegraphics[width=0.9\linewidth]{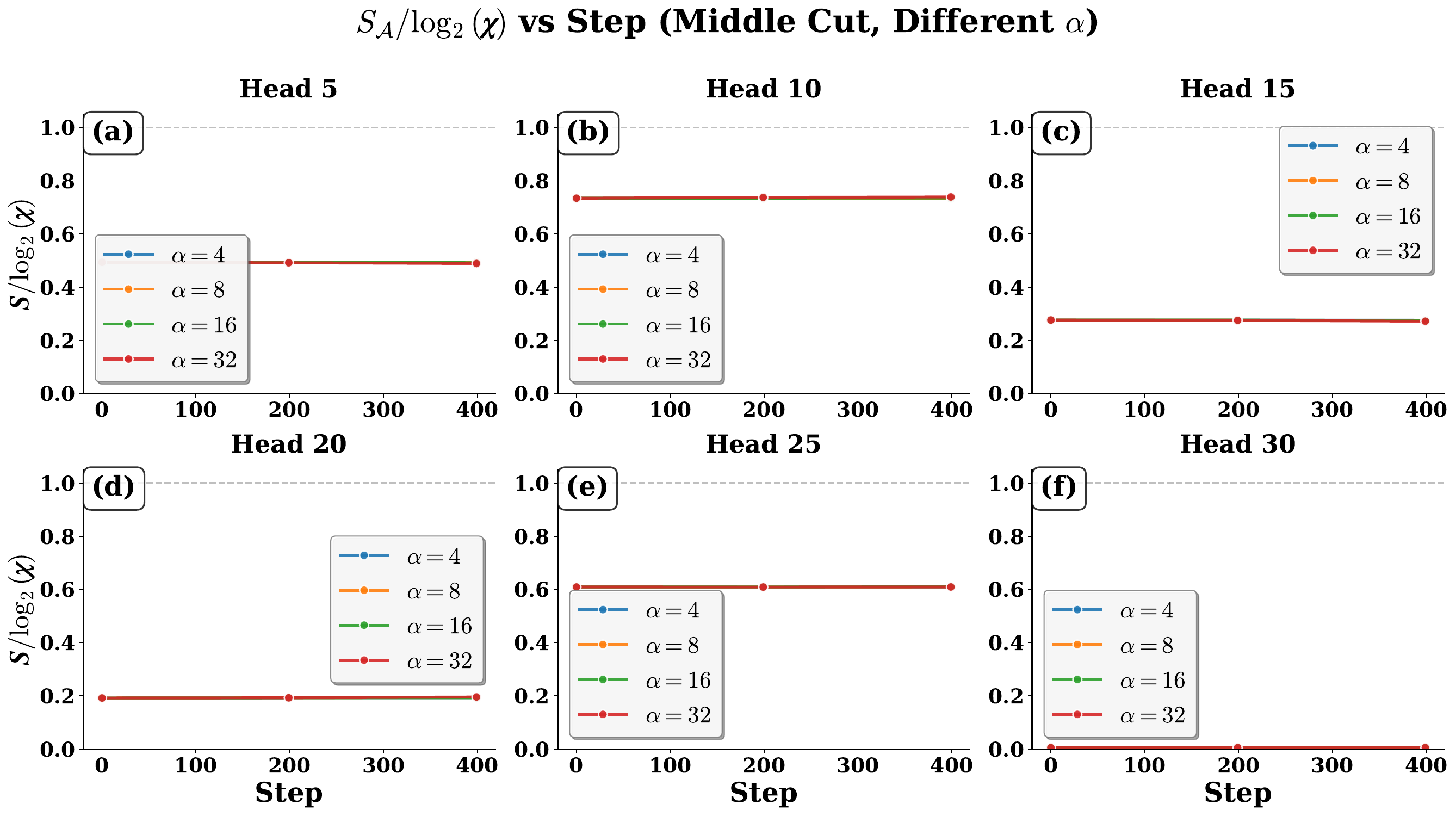}
    \caption{Normalized entanglement entropy $S_A/\log(\chi)$ across training steps for several attention heads under different scaling coefficients $\alpha$ for LLaMA-3.1-8B fine-tuned on OpenThoughts3 dataset.}
    \label{8B_OpenThoughts:attention_entropy_normalized_by_alpha_all_heads}
\end{figure}

\textbf{Token embedding space artificial entanglement profiling: volume-law and entanglement valley.}
FIG.~\ref{8B_OpenThoughts:delta_w_entropy_lr1e5_rank256_alpha16_4096_combined} reveals that the embedding space ($W_Q$ and $W_V$) in the 8B model on OpenThoughts3 exhibits the same volume-law behavior with a characteristic entanglement valley as observed in both the 1B model (Sec.~\ref{sec:Entanglement Structure in FFT and LoRA}) and the 8B model on Tulu3. The valley deepens during training in LoRA and shows sensitivity to the scaling coefficient $\alpha$, consistent with the findings in Sec.~\ref{sec:Entanglement Structure in FFT and LoRA}. This confirms that the volume-law entanglement profile and its hyperparameter sensitivity are fundamental properties independent of both model scale and dataset distribution.

\begin{figure}[H]
    \centering
    \includegraphics[width=0.8\linewidth]{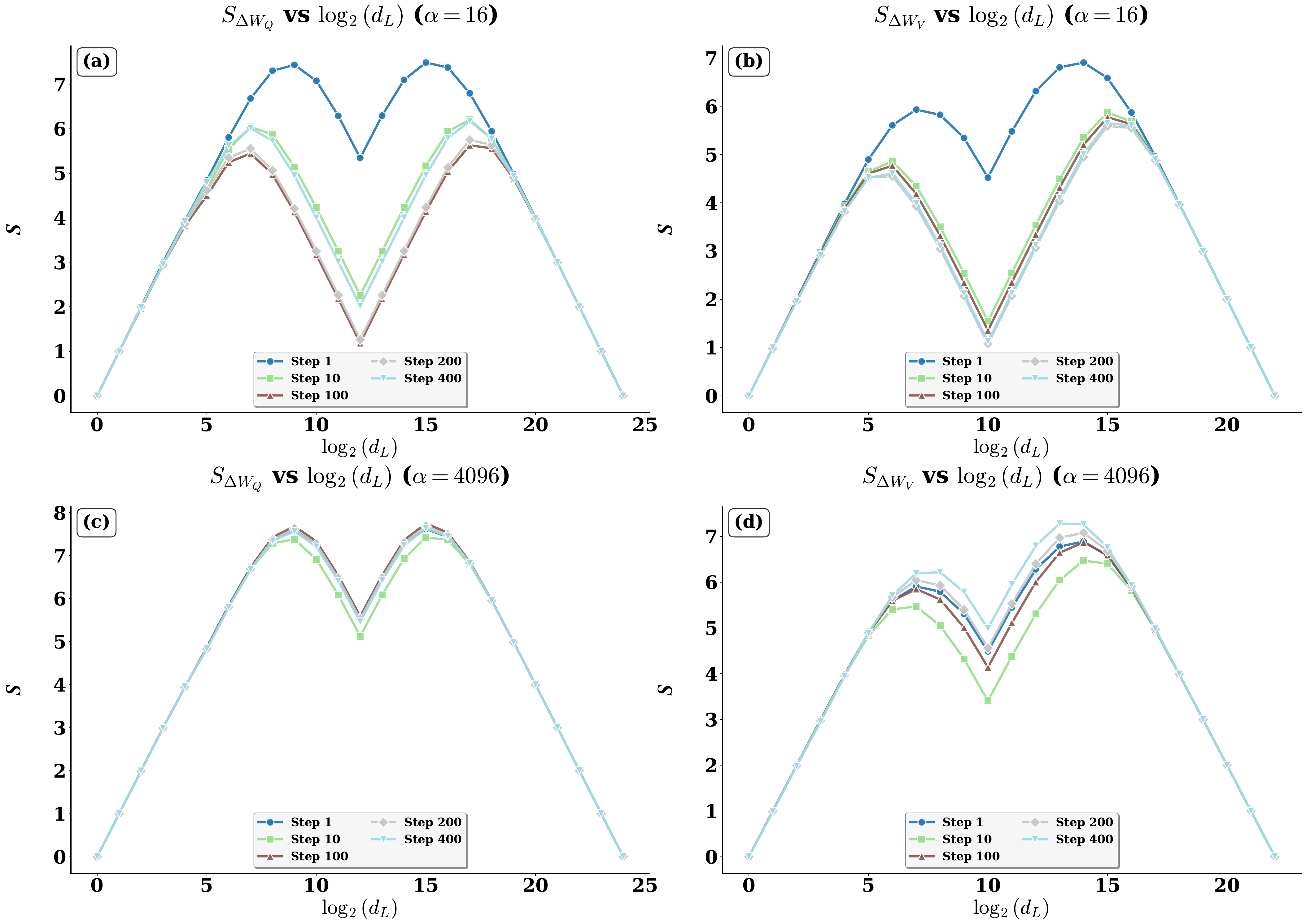}
    \caption{Artificial entanglement profiling of $\Delta W_Q$ and $\Delta W_V$ with respect to the cut position across training steps and scaling coefficients $\alpha \in \{16, 4096\}$ for LLaMA-3.1-8B fine-tuned on OpenThoughts3 dataset.}
\label{8B_OpenThoughts:delta_w_entropy_lr1e5_rank256_alpha16_4096_combined}
\end{figure}

\textbf{Optimization landscape comparison.}
FIG.~\ref{8B_OpenThoughts:lora_vs_mps_final_test_loss_vs_lr_llama3.1_8b_openthoughts} compares LoRA and MPS adaptation strategies for the 8B model on OpenThoughts3. The consistent performance patterns—where both methods achieve comparable results and optimal learning rates shift with $\alpha$—mirror the observations in other model-dataset combinations, providing strong evidence for the generalizability of tensor-network-based adaptation methods and the robustness of the no-hair phenomenon across diverse experimental settings.

\begin{figure}[H]
    \centering
    \includegraphics[width=0.5\linewidth]{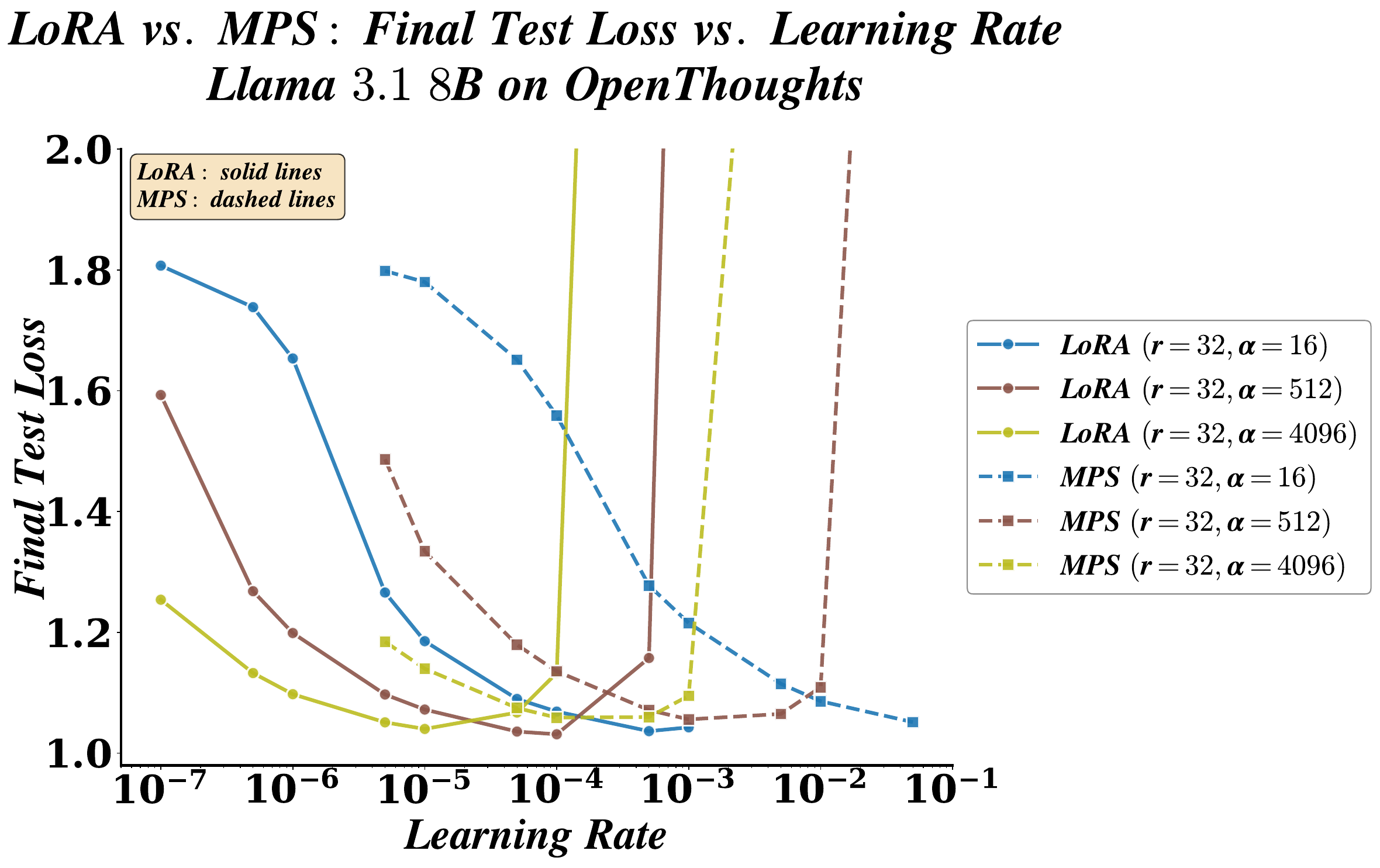}
    \caption{Final test loss of LoRA and MPS adaptation across multiple learning rates and scaling coefficients $\alpha$ for LLaMA-3.1-8B fine-tuned on OpenThoughts3 dataset.}
    \label{8B_OpenThoughts:lora_vs_mps_final_test_loss_vs_lr_llama3.1_8b_openthoughts}
\end{figure}

\end{document}